\documentclass{article}

\usepackage[nonatbib, final]{neurips_2022}

     \PassOptionsToPackage{numbers, compress}{natbib}

\usepackage{multirow}

\usepackage{minitoc}

\usepackage{graphicx} 
\usepackage{subcaption}
\usepackage{wrapfig}

\usepackage{natbib}
\usepackage{algorithm}
\usepackage{algorithmic}
\usepackage{filecontents}
\usepackage{bm}
\usepackage{times}
\usepackage{booktabs}
\usepackage{graphicx}
\usepackage{amsmath}
\usepackage{amssymb}
\usepackage{url}
\usepackage{multirow}
\usepackage{color}

\newcommand{\argmin}{\mathop{\rm arg~min}\limits}
\usepackage{mathtools}

\usepackage{varwidth}
\usepackage{subcaption}
\usepackage{array}
\usepackage[hidelinks]{hyperref} 
\usepackage{empheq}
\usepackage{wrapfig}
\usepackage{subcaption}
\usepackage{xcolor}    

\usepackage{amsmath}
\usepackage{amssymb}
\usepackage{mathtools}
\usepackage{amsthm}
\newtheorem*{definition*}{Definition}
\newtheorem*{proposition*}{Proposition}

\theoremstyle{plain}
\newtheorem{theorem}{Theorem}[section]
\newtheorem{proposition}[theorem]{Proposition}

\newtheorem{nono-prop}{Proposition}[]
\newtheorem{lemma}[theorem]{Lemma}

\theoremstyle{definition}

\theoremstyle{remark}

\definecolor{medGray}{RGB}{230,230,230}

\renewcommand*{\bibfont}{\small}

\makeatletter
\newcommand{\figcaption}[1]{\def\@captype{figure}\caption{#1}}
\newcommand{\tblcaption}[1]{\def\@captype{table}\caption{#1}}
\makeatother

\usepackage[T1]{fontenc}
\usepackage{textcomp}

\newcommand{\bfH}{{H}}
\newcommand{\bfM}{{M}}
\newcommand{\irrepM}{{V}}
\newcommand{\bfU}{{U}}

\newcommand{\Tc}{T_{\rm c}}
\newcommand{\Tp}{T_{\rm p}}

\newcommand{\tp}{t_{\rm p}}
\newcommand{\bs}{\mathbf{s}}
\newcommand{\bsc}{\mathbf{s}_{\rm c}}
\newcommand{\bsp}{\mathbf{s}_{\rm p}}
\newcommand{\bPsi}{{\Psi}}
\newcommand{\bPhi}{{\Phi}}
\newcommand{\overbar}[1]{\mkern 1.0mu\overline{\mkern-1.0mu#1\mkern-1.0mu}\mkern 1.0mu}

\usepackage{enumitem}
\setlist{leftmargin=5.5mm}

\newenvironment{tight_enumerate}{
\begin{enumerate}
  \setlength{\itemsep}{0pt}
  \setlength{\parskip}{0pt}
  \setlength{\topsep}{0pt}
  \setlength{\partopsep}{0pt}
}{\end{enumerate}}

\title{Unsupervised Learning of Equivariant Structure from Sequences}

\author{%
Takeru Miyato$^{*1,2}$ Masanori Koyama$^{*1}$ Kenji Fukumizu$^{3,1}$ \hspace{0.3mm} $^*${\rm equal contribution}\\ 
$^{1}$Preferred Networks, Inc. $^{2}$University of T\"ubingen $^{3}$The Institute of Statistical Mathematics
}

\begin{document}

\doparttoc
\faketableofcontents 

\maketitle

\begin{abstract}

In this study, we present \textit{meta-sequential prediction} (MSP), an unsupervised framework to learn the symmetry from the time sequence of length at least three. 
Our method leverages the stationary property~(e.g. constant velocity, constant acceleration) of the time sequence to learn the underlying equivariant structure of the dataset by simply training the encoder-decoder model to be able to predict the future observations. 
We will demonstrate that, with our framework, the hidden disentangled structure of the dataset naturally emerges as a by-product by applying \textit{simultaneous block-diagonalization} to the transition operators in the latent space, the procedure which is commonly used in representation theory to decompose the feature-space based on the type of response to group actions.
We will showcase our method from both empirical and theoretical perspectives.
Our result suggests that finding a simple structured relation and learning a model with extrapolation capability are two sides of the same coin. The code is available at {\color{orange}\url{https://github.com/takerum/meta_sequential_prediction}.

}.



\end{abstract}

\section{Introduction}


The recent evolution and successes of neural networks in machine learning fields have shown the importance of symmetry-aware neural network models~\cite{fukushima1982neocognitron, krizhevsky2012imagenet, kipf2016semi, vaswani2017attention}. 
In particular, symmetries in the form of geometric/algebraic constraints have been proven useful in various applications involving high-dimensional, highly-structured observations. 
For example, recent literature of robotics and reinforcement learning has succeeded in exploiting the knowledge of geometrical symmetries to improve the sample efficiency \cite{van2020mdp, wang2022mathrm} or to train a model that generalizes to unseen observations~\cite{simeonov2021neural}. 


However, building an inductive bias that matches the given dataset of interest is challenging, and recent studies have been exploring the ways to learn symmetries itself from observational sequences. Many of these approaches consider settings with relatively restrictive assumptions or weak supervision.
For example, ~\cite{son_action_known} allows the trainer to use the knowledge of the identities of the actions used in making the transition. 
Meanwhile, \cite{bouchacourt2020addressing, keller2021topographic, keller2021predictive} essentially assume that the transition velocity of all sequences in the datasets are the same.

These studies indicate that there is still much room left for the question of ``what is required for dataset/model to enable the unsupervised learning of the equivariance relation corresponding to the underlying symmetry''.
This paper advances this investigation by showing that if the sequential dataset consists of time series with a certain stationary property (constant velocity, constant acceleration), we can learn the underlying symmetries by simply training the model to be able to predict the future with linear transition in the latent space. 
Our theory in Section \ref{sec:theory_top} shows that this strategy can learn a model that is almost equivariant.
Moreover, we will experimentally show that, by training an encoder-decoder model in a framework of meta-learning which we call \textit{\bf meta-sequential prediction} (MSP), we can actually learn an equivariant model. 
In particular, we show that we can learn a hidden equivariance structure in the dataset by splitting  (1) the internal training step to compute the prediction loss of linear transition in the latent space from (2) the external training step to update the encoder and decoder. 
We will also empirically show that, in alignment with group representation theory~\cite{kondor2008group}, the learned linear latent transitions in our framework can be simultaneously block-diagonalized, and that each block corresponds to a disentangled factor of variation.

\section{Related works}

Recently, numerous studies have explored the ways to learn symmetry in a data-driven manner.
There is rich literature in unsupervised/weakly supervised approaches that use sequential datasets to exploit the structure that is shared across time, and they all differ by the types of inductive bias. 
For example, the object-centric approach introduces inductive bias in the form of architecture, and equips the model with pre-defined slots to be allocated to objects \cite{kipf2019contrastive, kabra2021simone,  kipf2021conditional}.
Meanwhile, \cite{greydanus2019hamiltonian, alet2021noether}
assumes that the symmetry to be found takes the form of the energy conservation law, and learns each variable in the law as a function of observations. 
While this approach assumes that some \textit{energy} is preserved in each observation, we assume that the transition parameters like velocity and acceleration are preserved within each observation.
Other more indirect forms of inductive bias include those relevant to distributional sparseness and symmetry defined through algebraic constraints. 
\cite{hyvarinen2016unsupervised} for instance assumed that every stationary component of a given time series is generated by a finite and independent latent time series. \cite{klindt2020towards} proposed to sparsely model the transition with a distribution of large kurtosis. 
Our work belongs to a family of unsupervised learning that seeks to find the underlying symmetry of the dataset based on an algebraic inductive bias that the transitions can be represented linearly in some latent space. 
In this sense, our inductive bias also has a connection to Koopman operator \cite{koopman1931hamiltonian, han2020deep, bevanda2021koopman}. 
We are different from these studies in that we are aiming to learn a common encoding function~(i.e. lifting function) under which the \textit{set} of sequences following different dynamics can be described linearly. 
Also \cite{zhang2021cross} applied Koopman operator on pedestrian walking sequences, and \cite{grosek2014dynamic} used Koopman operator to separate the foreground from the background. 
However, \cite{zhang2021cross} does not set out their model to solve the extrapolation, and neither \cite{zhang2021cross} nor \cite{grosek2014dynamic} discusses the natural algebraic decomposition of the latent space that results solely from the objective to predict the unseen future.

\paragraph{Unsupervised learning with algebraic/geometrical constraints} 
Many studies impose algebraic constraints that reflect some form of geometrical assumptions. 
\cite{falorsi2018explorations} uses a known coordinate map parametrization of a Lie group family to construct a posterior distribution on the manifold.  \cite{pfau2020disentangling} assumes that the observations are dense enough on the data-manifold to describe its tangent space, and 
exploits a property of random walks on the product manifold to decompose the data space.
In the analysis of sequential datasets, 
\cite{culpepper2009learning, sohl2010unsupervised, cohen2014learning, son_action_known, connor2020representing} make some Lie group type assumptions about the transition. \cite{son_action_known} also assumes that the identity of the actions in the sequences is known.
As for the approaches with less explicitly geometrical touch, \cite{keller2021topographic} uses capsule structure in their probabilistic framework to model a finitely cyclic structure while retaining the computability of posterior distribution. 
By design,  \cite{keller2021topographic} assumes that all sequences in the dataset transition with the same cyclic velocity.
\cite{yang2021towards} enforces the underlying transition action to be commutative.
Some of these studies learn the representation so that the linear transition in the latent space can be explicitly computed \cite{son_action_known, connor2020representing,bouchacourt2020addressing}. 
In particular, \cite{bouchacourt2020addressing} presents a theory that suggests that a representation without this feature would have topological defects, such as discontinuity.
Our approach shares a similar philosophy with these works except that, instead of imposing a strong assumption about the underlying symmetry, 
we only make a relatively weak stationarity assumption about the dataset;
although we assume each sequence to be transitioning with constant velocity/acceleration, we allow the velocity/acceleration to vary across different sequences.

\paragraph{Disentangled Representation Learning}
Disentangled structure \cite{higgins2018towards, higgins2022symmetry} is a form of symmetry that has also been actively studied.
It is known that, under the i.i.d assumption of examples, \textit{unsupervised} learning of disentanglement representation is not achievable without some inductive biases encoded in models and datasets~\cite{locatello2019challenging}. 
In response to this work, subsequent works have explored different frameworks such as weakly-/semi-supervised settings~\cite{locatello2019disentangling, locatello2020weakly} and learning on sequential examples~\cite{klindt2020towards} to learn disentangled representations.
For example, PhyDNet~\cite{guen2020disentangling} disentangles the known physical dynamics from the unknown factors by preparing an explicit module called PhyCell.  
ICA~\cite{hyvarinen2000independent} and recent works~\cite{zimmermann2021contrastive, von2021self} also discuss the identifiability property of learned representations. 
Classical methods like \cite{higgins2016beta, ccivae2017, kim2018disentangling} take an approach of incorporating the inductive bias in the form of a probabilistic model. 

We are different from many previous methods in that we do not equip our model with an explicit disentanglement framework. 
Our method achieves disentanglement as a by-product of training a model that can predict the future linearly in the latent space. 
The set of latent linear transformations estimated by our method for different time sequences can be simultaneously block-diagonalized, and the latent space of each block corresponds to a disentangled feature.
Our data assumption about constant velocity/acceleration might be similar in taste to the setting used by \cite{hyvarinen2000independent}, in which the observed time series can be split into the finite number of stationary components. 

\section{Learning of equivariant structure from stationary time sequences} \label{sec:theory_top}

Our goal is to learn the underlying symmetry structure of a dataset in an unsupervised way that helps us predict the future.
What do us humans 
require for the dataset when we are tasked to, for example, predict where a thrown ball would be 
in the next second? 
We hypothesize that we solve such a prediction task by analyzing a short, past time-frame with a certain stationary property (e.g., constant velocity/acceleration). 
Indeed, people with good dynamic visual acuity can chase a fast-moving object, because they can identify such a short stationary time-frame and use it to predict the near future \textit{linearly} in their latent space.
Based on this intuition, we propose to provide the trainer with a dataset consisting of constant velocity/acceleration sequences.
We formalize this idea below.
\paragraph{Dataset structure} 
Our dataset $\mathcal{S}$ consists of sequences in some ambient space $\mathcal{X}$, so that each member $\bs \in \mathcal{S}$ takes the form $\bs = [s_t \in  \mathcal{X};  t =1,...,T] \in \mathcal{S}$.
Because we want $\bs$ to be describing a sequence that transitions with constant \textrm{velocity}, we assume that all $s_t$ in a given instance of $\bs \in \mathcal{S}$ are related by a fixed transition  operation $g\in\mathcal{G}$ so that $s_{t+1} = g \circ s_t$ for all $t$, where $\mathcal{G}$ is the set of transition operators on $\mathcal{X}$ and each $g \in \mathcal{G}$ acts on $x \in \mathcal{X}$
by sending $x$ to $g \circ x$. 
We assume that $\mathcal{X}$ is closed under $\mathcal{G}$; that is, $g \circ x \in \mathcal{X}$ for all $x \in \mathcal{X}, g \in \mathcal{G}$. 
We allow $\mathcal{G}$ to be continuous as well, so that $g$ might not have a finite order (For instance, if $g$ is a rotation with speed $2\pi r$ with irrational $r$, any finite repetition of $g$ would not agree with identity mapping). 
This way, our setting is different from those used in \cite{keller2021topographic} that explores a cyclic structure using the capsules of same size. 
We emphasize that the transition action $g$ is generally assumed to differ across the different members of $\mathcal{S}$. 
For example, if $\mathcal{G}$ is a set of rotations and $\mathcal{X}$ a set of images, then the rotational speed, direction and the initial image may all be different for any two distinct sequences, $\bs$ and $\bs'$.
Because each instance of $\mathcal{S}$ is characterized by $s_1$ and $g$, 
we may write $\bs(s_1, g)$ to denote a sequence that begins with initial frame $s_1$ and transitions with $g$.
Summarizing, $\mathcal{S}$ is a subset of $\{[g^t \circ s_1 ;  t =0,...,T-1]  ;  s_1 \in \mathcal{X}, g\in \mathcal{G}\}$.
\paragraph{Prediction framework through equivariance} 
Our strategy is to exploit the stationary property of each $\bs \in \mathcal{S}$ to seek an invertible continuous function $\Phi: \mathcal{X} \to \mathbb{R}^{a \times m}$ such that there exists some $M :\mathcal{G} \to \mathbb{R}^{a \times a}$ satisfying 
\begin{align}
\begin{split}
 M_g \Phi(x)= \Phi(g \circ x)~\textrm{for all}~x \in \mathcal{X}~\textrm{and}~g\in \mathcal{G}.
 \end{split}
 \label{eq:equivariance}
\end{align}
This relation is known as equivariance~\cite{cohen2016group}, and this type of tensorial latent space has also been used in ~\cite{koyama2022} as well for the unsupervised learning of underlying structure of the dataset.
In other words, we seek a model in which $\mathcal{X}$ and $\mathbb{R}^{a \times m}$ are invertibly related by an equivariance relation with respect to $\mathcal{G}$, where $g \in \mathcal{G}$ acts on $\Phi(x) \in \mathbb{R}^{a \times m}$ via the map $\Phi(x) \mapsto M_g \Phi(x)$ with $M_g \in \mathbb{R}^{a \times a}$.  
We assume $m>1$ in our study\footnote{
We note that, when $m > 1$, the action of $M_g$ on $\mathbb{R}^{a \times m}$ can be realized by applying the same $M_g$ to $m$-copies of $\mathbb{R}^{a}$. In other words, our prediction framework assumes that there are $m$ number of subspaces that react to $g$ in the same way. 
This $m>1$ assumption is also considered in \cite{bouchacourt2020addressing}. 
The case in which there are no copies of subspaces that act in the same way is called multiplicity-free in the literature of representation theory~\cite{clausen, fulton_harris}, and is known to be a special case that happens only under a restrictive condition on the dimension of the observations space \cite{multiplicity_free}. 
A similar idea has been used in model architecture as well \cite{vector_neurons}.}.
In this framework, we predict the sequence $\bs(s_1, g)$ with the relation $\Phi^{-1} ( M_g  \Phi(s_{t-1})) = \tilde{s}_{t}$.
When $\mathcal{G}$ is a group, \eqref{eq:equivariance} would imply $M_{gh} \Phi(x) =\Phi(gh \circ x) = M_g \Phi(h \circ x) = M_g M_h \Phi(x)$ for all $x$, and the map 
$g \mapsto M_g$ is called a \textit{representation} of $\mathcal{G}$ ~\cite{ clausen, Weintraub, cohen2014learning}.

Because we are aiming to establish the framework in which the representation of each action $g$ can be explicitly computed, our philosophy has much in common with the proposition of \cite{bouchacourt2020addressing}. 
This approach is in contrast to \cite{keller2021topographic}, which encodes a predefined cyclic structure in the model. 

\begin{figure}[H]
    \centering
    \includegraphics[scale=0.14]{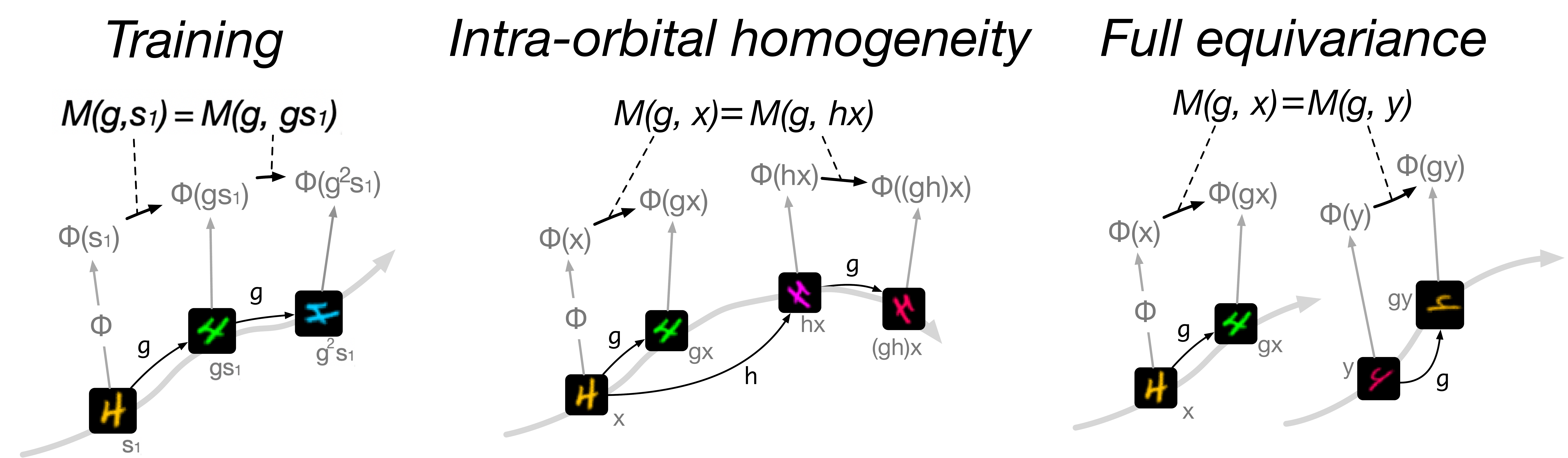}
    \caption{
    Visualization of intra-orbital homogeneity  vs full equivariance. 
    During the training, the model was trained to satisfy $M(g,s_1)=M(g, g\circ s_1)$ for all $g$ and $s_1$. 
    When intra-orbital homogeneity holds, $M(g, x) = M(g, h\circ x)$ for all $h,g \in G$ and $x$. 
    When the full equivariance holds, $M(g, x)$ is invariant across different orbits. 
    }
    \label{fig:train_to_equiv}
\end{figure}

\subsection{Learning equivariance relation from stationary sequential dataset}\label{sec:method}

However, training the model satisfying \eqref{eq:equivariance} with just the \textit{constant velocity assumption} is not a trivial task, because this model assumption only assures that, for each sequence $\bs(s_1, g)$, there is a sequence-specific operator $M(g, s_1)$ that is guaranteed only to be able to predict the sequence that transitions with $g$ and begins from $s_1$ in the way of $M(g, s_1)\Phi(s_t) = \Phi(s_{t+1}) = \Phi(g \circ s_t)$ (the left most panel in Figure~\ref{fig:train_to_equiv}).
In order to satisfy the \textit{full} equivariance (\eqref{eq:equivariance} or the right most panel Figure~\ref{fig:train_to_equiv}), $M(g, s_1)$ shall not depend on $s_1$ (i.e. homogeneous with respect to $s_1$). 
At  the same time, because the constant velocity assumption applies to each sequence over all time intervals, it at least assures that the latent transition $M$ is well defined within each sequence; that is, $M(g, x) = M(g, g \circ s_1) = M(g, g^2 \circ s_1) \cdots$ and so on.
It turns out that, with some regularity assumptions on the model and the choice of $\mathcal{G}$, we can extend this observation to say that  $M$ satisfies \textbf{intra-orbital homogeneity} (the middle panel in Figure~\ref{fig:train_to_equiv}) ; that is,  $M(g, x)$ is constant on the orbit $\mathcal{G} \circ x = \{ g \circ x ; g \in \mathcal{G}\}$ for each $x$. 
\begin{proposition}
Suppose that $\Phi(s_{t}) = M(g, s_1) \Phi(s_{t-1})$ for all $\bs$ and $t$.  If $m > a$
and if $\mathcal{G}$ is a compact commutative Lie group, then $M$ satisfies intra-orbital homogeneity.
\end{proposition}
Also, if $M$ satisfies  intra-orbital homogeneity, $M(g, x)$ and $M(g, x')$ for any pair $(x, x')$ in different orbits $\mathcal{G} \circ x \neq \mathcal{G} \circ x'$ can be shown to be at least similar.  
\begin{proposition}
Suppose that $M(g, x)$ satisfies intra-orbital homogeneity, and suppose that $\mathcal{G}$ is a compact connected group.
If $M(g,x)$ is continuous with respect $x$, then 
for all $(x, x')$, there exists some $P$ such that 
$P M(g, x)P^{-1} = M(g, x')$.
\end{proposition}
Thus, much of the equivariance property can be satisfied automatically by training the representation on the set of stationary sequences. 
Interestingly, as we experimentally demonstrate later, our training method in the next section successfully learns a fully equivariant $\bPhi$ without explicitly enforcing the change of basis $P$ to be $I$.


\subsection{Learning $\Phi$ via solving a \textit{meta-sequential prediction} task}\label{sec:enc_alg}
We propose a meta-learning way to learn a homeomorphic function $\bPhi:\mathcal{X} \to \mathbb{R}^{a \times m}$ with  equivariance property by seeking an injective ${\bPhi}$ such that $M(g, s_1)\bPhi(s_t) = \bPhi(s_{t+1})$ for all $g\in \mathcal{G}$, $s_1 \in \mathcal{X}$.
We learn such $\bPhi$ by casting this problem as a meta-learning problem in which $M(g, s_1)$ is to be internally estimated for each $\bs$.
In other words, we seek a pair of an encoder $\bPhi$ and a decoder $\bPsi$ such that $\mathcal{L} (\bPhi, \bPsi | \bs) =  \min_\bfM  \sum_{s_t, s_{t+1} \in \bs} \|\bPsi (\bfM \bPhi(s_t)) - s_{t+1}\|_2^2$ is optimized for each $\bs$. 

We conduct this optimization by splitting each $\bs=\{s_1, ..., s_T\} $ into conditional time sequence $\bsc = \{s_1, ..., s_{T_c}\}$ and validation time sequence $\bsp = \{s_{{T_c}+1}, ..., s_{T}\}$, while using the former for the internal optimization of $M$ to force the linear algebraic relation in the latent space and using the latter for the prediction loss. 
More precisely, we solve the following optimization problem about $\Phi$ and $\Psi$:
\begin{align}
\begin{split}
\mathcal{L}^p(\Phi, \Psi) &:= {\textstyle\sum}_\bs {\textstyle\sum}_{t={\Tc+1}}^{T} \left\| \Psi ( \bfM^*(\bsc|\Phi)^{t-\Tc} \Phi(s_{\Tc}) )-s_t \right\|_2^2 \\
\textrm{where~} & \bfM^*(\bsc|\Phi)= \argmin{}_\bfM {\textstyle\sum}_{t=1}^{\Tc-1}  \left\| \bfM \Phi(s_{t}) - \Phi(s_{t+1})\right\|_F^2.
\end{split} \label{eq:pred_loss}
\end{align}
Since $M^*$ is obtained from the latent sequence in the internal optimization, we call this learning framework the \textbf{meta-sequential prediction} (MSP).
It might appear as if we can also set $\bs=\bsc$ and optimize the following reconstruction version of Eq.\eqref{eq:pred_loss}:
\begin{align}
\mathcal{L}^r(\Phi, \Psi) := {\textstyle\sum}_\bs {\textstyle\sum}_{t=2}^{\Tc} \|\Psi(\bfM^*(\bs|\Phi)^{t-1} \Phi(s_1)) - s_{t}\|_2^2.
\label{eq:rec_loss}
\end{align}
However, as we will see in the experiment section, the use of the validation sequence $\bsp$ makes a substantial difference in the learned representation.
This is most likely because the minimization of the validation error of $M^*$ on $\bsp$ would encourage $\Phi$ to exclude the $\bsc$-specific information from the transition $M^*$.
We will illustrate this effect in the experiment section.
We shall note that we can also parameterize $\bfM$ as $\bfM:= \exp(A)$, where $A \in \mathbb{R}^{a \times a}$ is a Lie algebra element to be internally optimized. 
This type of approach was used in \cite{dehmamy2021automatic} for building Lie group convolutional network and in~\cite{sohl2010unsupervised, connor2020representing} for predicting a sequence that is not necessarily stationary. 
In order to train their model, \cite{sohl2010unsupervised} used additional parameters to diagonalize each algebra element as well as hyperparameters to stabilize the training. 
 We also experimented with a Lie-algebra style representation of $\bfM^*$ and used SGD to internally optimize the exponent parameters, but we needed to carefully tune the hyper-parameter for the norm regularization term to stabilize the training and never really succeeded to train the model without collapsing.
 Our internal optimization procedure is free of such a parameter tuning. 
\begin{wrapfigure}{R}{0.5\columnwidth}
    \centering
    \includegraphics[scale=0.13]{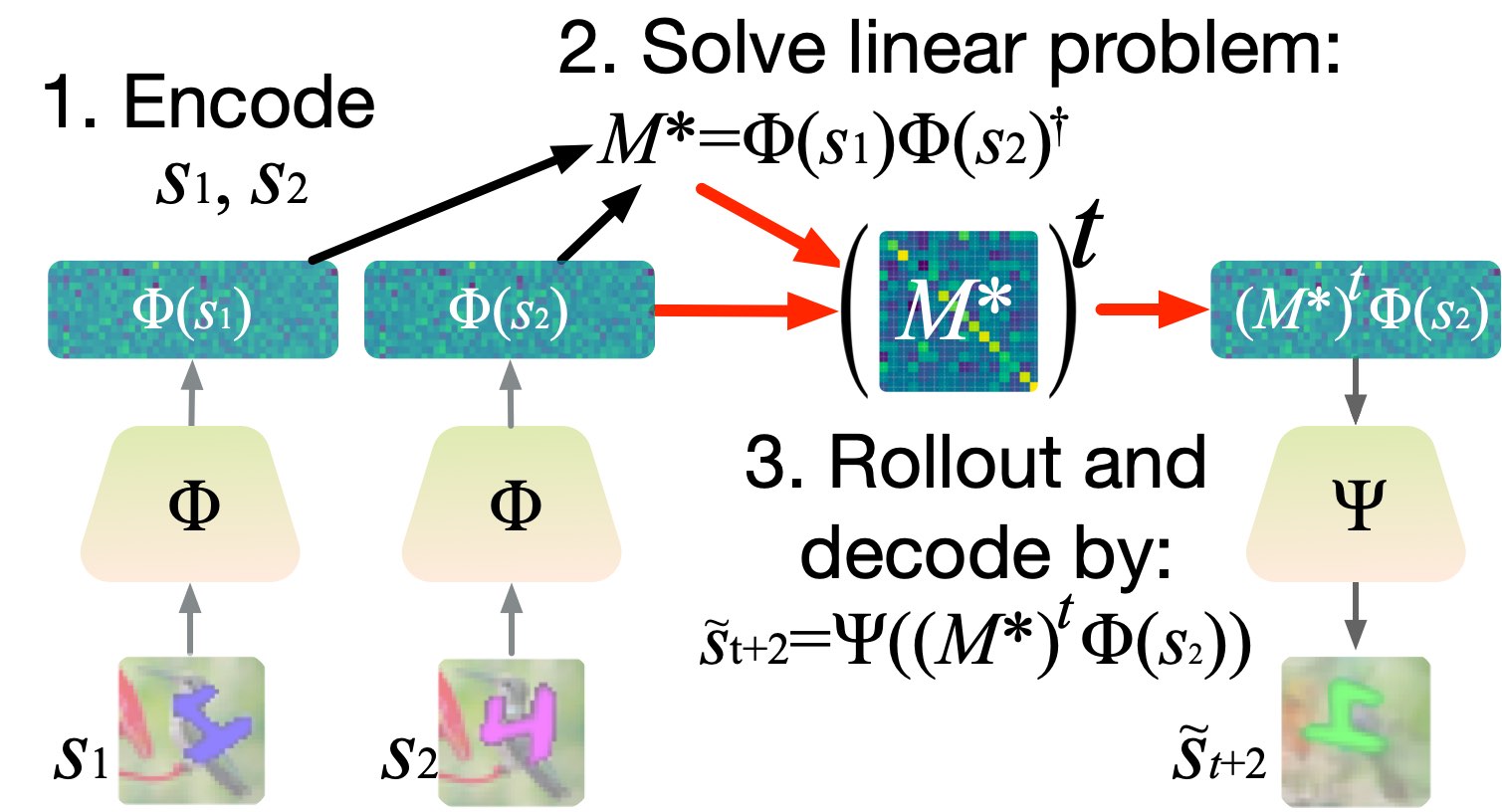}
    \caption{The overview of \textit{meta-sequential prediction}~(MSP) when $\Tc =2$ and $\Tp=t$. After the encoder encodes the observations into tensor representations $\Phi(s_1), \Phi(s_2)$, the method solves the least square problem: $\bfM^*=\argmin{}_\bfM\|\bfM \Phi(s_1)-\Phi(s_2)\|_F^2$. The model then predicts the future observations by $\tilde{s}_{t+2}=\Psi((\bfM^*)^t \Phi(s_2))$. These processes (including the linear problem) are all differentiable.}
    \label{fig:model}
\end{wrapfigure}

\paragraph{Internal Optimization of $M^*$}
Because the internal optimization in~Eq.\eqref{eq:pred_loss} is a linear problem, it can be solved analytically as 
\begin{align}
    M^*(\bsc | \Phi) = \bfH_{+1}\bfH_{+0}^{\dagger}, ~\label{eq:lstsq}
\end{align}
where $\bfH_{+0}=[\Phi(s_1) ;...;\Phi(s_{\Tc-1})] \in \mathbb{R}^{a \times (\Tc-1) m} $ and $\bfH_{+1}=[\Phi(s_2) ;...;\Phi(s_{\Tc}) \in \mathbb{R}^{a \times(\Tc-1) m}$ are the horizontal concatenations of the encoded frames and $\bfH_{+0}^{\dagger}$ is the Moore-Penrose pseudo inverse of $\bfH_{+0}$. 
Because $M^*(\bsc | \Phi)$ is a closed form with respect to $\Phi$, the loss \eqref{eq:pred_loss} can be directly optimized by differentiating it with respect to the parameters of both $\Phi$ and $\Psi$. 
Thus, the training is done in an end-to-end manner. Figure~\ref{fig:model} summarizes the overall procedure to make prediction on a given sequence when $T_c=2, T_p=t$.
We note that, although we have assumed the dataset to consist of constant-velocity sequences,
we can readily extend our method to the dataset consisting of the time series with higher-order stationarity, such as constant acceleration. 
See Section~\ref{sec:accl} for the detailed explanation of the model extension and the experimental results.

\subsection{Irreducible decomposition of $M^*$s}
\label{sec:sbd}
Representation theory guarantees that, if $\mathcal{G}$ is a compact connected group, any representation $D: \mathcal{G} \to \mathbb{R}^{a \times a}$ can be simultaneously block-diagonalized; that is, there is a common change of basis $U$ such that $V := U D_g U^{\rm T} =  \bigoplus_j V_g^{(j)}$, where $V_g^{(j)}$ is called irreducible representation that cannot be block-diagonalized any further~\cite{kondor2008group,cohen2014learning, Weintraub}.
The equivariance of our $\Phi$ which we show in the experimental section suggests that $\bfM^*(\bs | \Phi)$ may be simultaneously block-diagonalizable as well. 
This block-diagonalization sometimes reveals disentanglement structure because any irreducible representation of $\mathcal{G}_1 \times \mathcal{G}_2$ is of form $V^{(1)} \otimes V^{(2)}$, where $V^{(k)}$ is an irreducible representation of $\mathcal{G}_k$. In particular, if $M^*$'s irreducible representations have the form $V^{(1)}\otimes 1$ or $1 \otimes V^{(2)}$, then each block would either corresponds to the action of $\mathcal{G}_1$ or of $\mathcal{G}_2$.\begin{footnote}{$V: \mathcal{G} \to \{1\}$ is a valid irreducible representation for any $\mathcal{G}$.  }\end{footnote}

To find $U$ that simultaneously block-diagonalizes all $\bfM^*(\bs|\Phi)$, we optimized $U$ based on the following objective function that measures the block-ness of $\irrepM^{*}(\bs):= \bfU \bfM^{*}(\bs | \Phi) \bfU^{-1}$ based on the normalized graph Laplacian operator $\Delta$: 
\begin{align}
    \mathcal{L}_{\rm bd}(\irrepM^*(\bs) ) := \|\Delta ( A(\irrepM^*(\bs)))\|_{\rm trace} = {\textstyle\sum}_{d=1}^{a} \sigma_d ( A(\irrepM^*(\bs))) \label{eq:bd_loss}
\end{align}
where $A( \irrepM^*(\bs) ) = abs(\irrepM^*(\bs)) abs(\irrepM^*(\bs))^{\rm T}$ with 
$abs(\irrepM^*(\bs))$ representing the matrix such that~$abs(\irrepM^*(\bs))_{ij} = |\irrepM^*_{ij}|$.
Our objective function is based on the fact that, if we are given an adjacency matrix $A$ of a graph, then the number of connected components in the graph can be identified by looking at the rank of the graph Laplacian. For the derivation, please see Appendix~\ref{sec:full_blockdiag}.
Through this decomposition, we are able to uncover the hidden block structure of $M^*$s.
See Figure~\ref{fig:D_UVU} for the actual block-decompositions of $M^*$s through our simultaneous block diagonalization.
We show in Section~\ref{sec:sbd_results} that each block component of $\irrepM^*$ with optimized $U$ corresponds to the disentangled factor of variations in dataset.

\section{Experiments}\label{sec:experiments}
We conducted several experiments to investigate the efficacy of our framework. 
In this section, we briefly explain the experimental settings. 
For more details, please see Appendix~\ref{sec:exp_settings}.
We tested our framework on \textit{Sequential} MNIST,  3DShapes~\cite{3dshapes18}, and SmallNORB~\cite{lecun2004learning}. Sequential MNIST is created from MNIST dataset~\cite{lecun1998gradient}.
For all experiments, we used a ResNet~\cite{he2016deep}-based encoder-decoder architecture and we set $a=16$ and $m=256$ so that the latent space lives in $\mathbb{R}^{16 \times 256}$. 

For Sequential MNIST, we chose our $\mathcal{G}$ to be the set of all combinations of three types of transformations: shape rotation, hue rotation, and translation, and 
randomly sampled a single instance of $g\in \mathcal{G}$ for each sequence(See Appendix~\ref{sec:exp_settings} for the examples of sequences). 
To create each sequence,  we first resized the MNIST image to 24$\times$24, applied repetitions of a randomly sampled, fixed member of $g\in \mathcal{G}$ and embedded the results to $32 \times 32$ images.
For shape and hue rotations, we randomly sampled the velocity of angles from uniform distribution on the interval $[-\pi/2, \pi/2)$ for each sequence. For translation, we randomly sampled the start point and end point in the range of [-10, 10], and then moved the digit images on a straight line between the sampled points. 
We also experimented on sequential MNIST with background (Sequential MNIST-bg). 
For Sequantial MNIST-bg, we used the same generation rule as Sequential MNIST but we added background images  behind the moving digits. 
For the background, we used a randomly sampled images from  ImageNet~\cite{ILSVRC15}, which were all resized to 32$\times$32.
Also, we only used the images of digit 4 for most of the experiments on Sequential MNIST/MNIST-bg. Unless otherwise noted, all evaluations in this paper for the Sequential MNIST are based on training with only \textit{digit 4}.

\begin{figure}

    \includegraphics[scale=0.155]{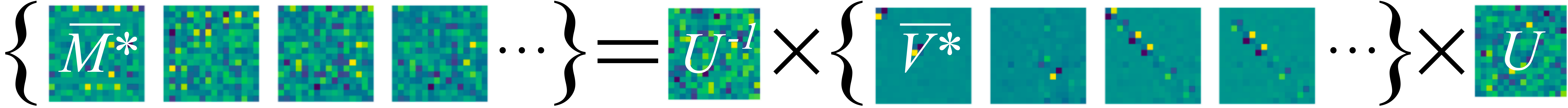}
    \caption{Simutaneous block diagonalization (SBD) applied to the set of $M^*$s obtained from 3DShapes sequences. SBD finds the common change of basis under which 
    all matrices $V^*= U^{-1} M^* U$ simultaneously take the form of block diagonal matrices with the same block positions.  For clarity, we provide in this figure the visualizations of 
     $\overbar{M^*}:=M^*-I$ and $\overbar{V^*}:=V^*-I$ instead of $M^*$ and $V^*$. }
    \label{fig:D_UVU}
\end{figure}


3DShapes and SmallNORB are datasets with multiple factors of variation. 
We created a set of~\textit{constant-velocity} sequences from these datasets by varying a fixed combination of factors for each sequence.
That is, on these datasets, we chose our $\mathcal{G}$ to be the set of variations of factors, and sampled each $g$ as $\prod_i g_{i}^{\ell_i}$,   where $g_{i}$ represents the increase of $i$-th factor by one unit and $\ell_i \in \mathbb{Z}$. 
Thus, the value of $\ell_i$ represents the velocity in the direction of the $i$-th factor on the grid. 
For 3DShapes, we chose \textit{wall hue}, \textit{floor hue}, \textit{object hue}, \textit{scale}, and \textit{orientation} as the factors to vary. 
We varied \textit{elevation} and \textit{azimuth} for SmallNORB. For the split of each sequence into $(\bsc, \bsp)$ , we set $\Tc=2$ and $\Tp=1$ on all of the constant velocity experiments.

We also conducted experiments for the sequences with constant acceleration on Sequential MNIST. 
To create a sequence with constant acceleration, we chose a pair $g_a, g_v \in \mathcal{G}$ for each sequence, and generated $\bs$ by setting $s_{t+1} = g_a^t g_v s_t$. 
We elaborate on the detail of this extension in \ref{sec:accl}. 

As ablations, we tested several variants of our method: 
\textbf{fixed 2x2 blocks}~(abbreviated as fixed blocks), \textbf{Neural$M^*$},   \textbf{Reconstruction model} (abbreviated as Rec. Model), and \textbf{Neural transition}. 
For the method of \textit{fixed 2x2 blocks}, we separated the latent tensor $\Phi(s)\in \mathbb{R}^{16 \times 256}$ into 8 subtensors $\{\Phi^{(k)} (s)\in \mathbb{R}^{2\times 256}\}_{k=1}^8$ and calculated pseudo inverse for each $k$ to compute the transition in each $\mathbb{R}^{2 \times 256}$ dimensional space.
This variant yields $M^*$ as a direct sum of 
eight $2 \times 2$ matrices. 
We tested this variant to see the effect of introducing a predetermined representation theoretic structure as in \cite{cohen2014learning}.
For Rec.~model, we trained $\Phi$ and $\Psi$ based on $\mathcal{L}^r$ in Eq.~\ref{eq:rec_loss} with $\Tc=3$. 
We tested this variant to see the effect of our use of $T_p$. 
For Neural$M^*$, we trained an additional network $M_\theta$ that maps $\bsc$ to a transition matrix, replaced $\bfM^*$ with $M_\theta$ in \eqref{eq:lstsq}, and optimized $\theta$ and $(\Phi, \Psi)$ simultaneously.
We may see Neural$M^*$ as a variant in which the \textit{meta} part of the internal and external training is removed from our method.
For Neural transition, we trained 1x1 1D-convolutional networks to be applied to 
latent sequences \textit{in the past}
to produce the latent tensor in the next time step; for instance, $\tilde{s}_{t+1} = \Psi (\rm{1DCNN}(\Phi(s_t), \Phi(s_{t-1})))$ when $\Tc=2$ {}\footnote{This model can be seen as a  simplified version of \cite{oord2016wavenet}}. The 1DCNN was applied along the multiplicity dimension $m$.  
In this variant, the relation between $\Phi(s_t)$ and $\Phi(s_{t+1})$ is not necessarily linear.
Section \ref{sec:detail_ablation} in Appendix describes each of the comparison methods more in detail.
In testing all of these variants, we used the same pair of encoder and decoder architecture as the proposed method. 
\begin{figure}[t]
\subcaptionbox{Generated examples with our proposed method on Sequential MNIST and MNIST-bg \label{fig:gen_mnist}}[.99\linewidth]{\includegraphics[scale=0.30]{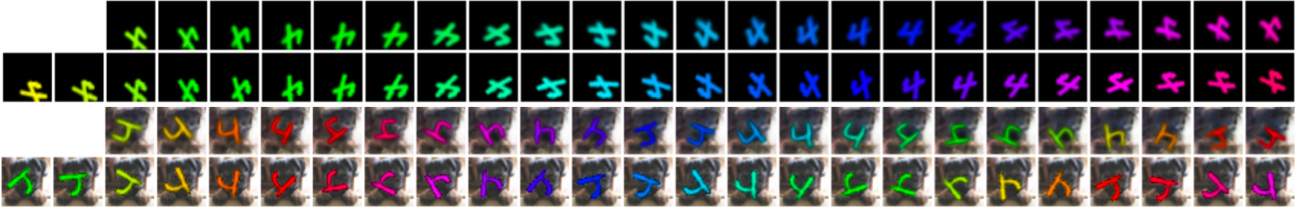}}
\vspace{0.5mm}
\subcaptionbox{Comparison of variant methods \label{fig:gen_comparison}}[.99
\linewidth]{\includegraphics[scale=0.35]{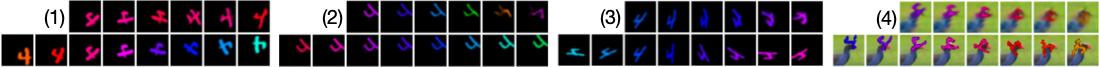}}%
\caption{(a): Predictions made by \textit{meta-sequential prediction} on Sequential MNIST and MNIST-bg. The ground truth sequence is placed below the predictions, with the first two images representing $\bsc$.
(b): Typical failure examples generated by the comparative methods. (1)(2)(3) and (4) are Neural$M^*$, Neural transition, Rec. model and Ours w/ fixed block, respectively.
See Appendix~\ref{sec:gen} for more examples.
\label{fig:gen_images_mnist}}

\end{figure}
\subsection{Qualitative and quantitative results on the prediction}
\begin{figure}[t]
\centering
    \subcaptionbox{Sequential MNIST\label{fig:prederrors_mnist}}[.32\textwidth]{
        \includegraphics[scale=0.27]{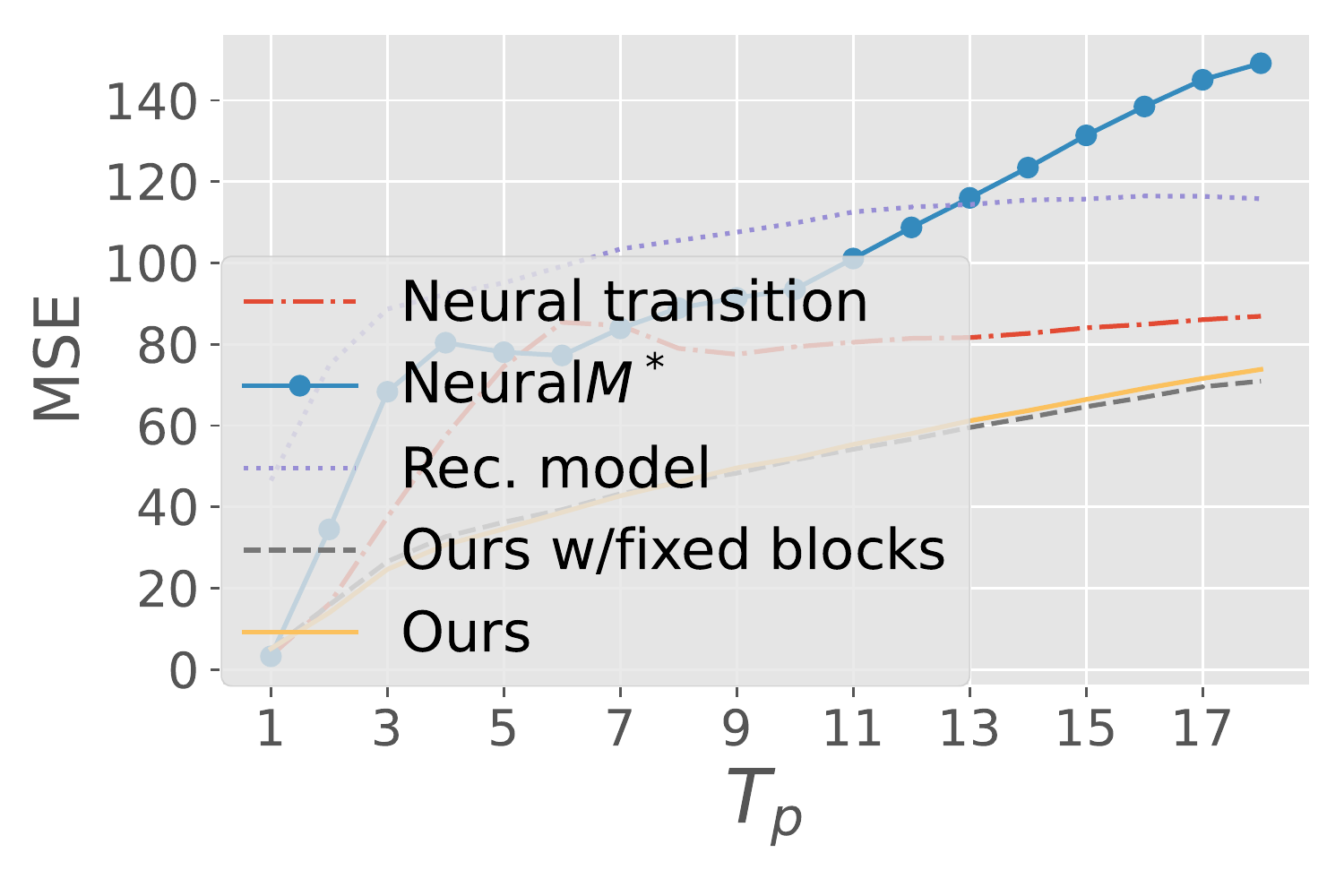}
    }
    \subcaptionbox{Sequential MNIST-bg\label{fig:prederrors_mnist_bg}}[.32\linewidth]{
        \includegraphics[scale=0.27]{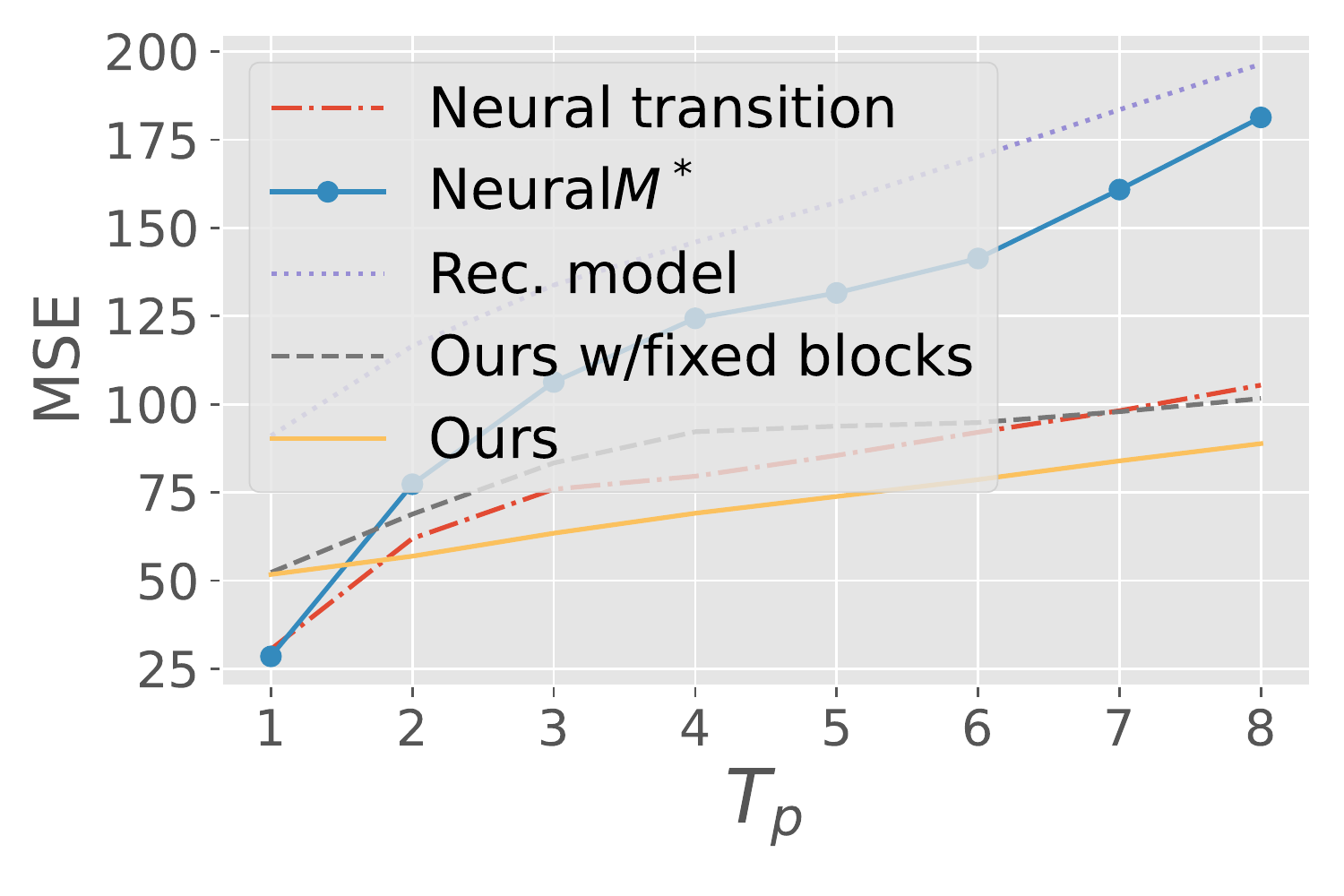}
    }
    \subcaptionbox{3DShapes\label{fig:prederrors_3dshapes}}[0.32\linewidth]{
        \includegraphics[scale=0.27]{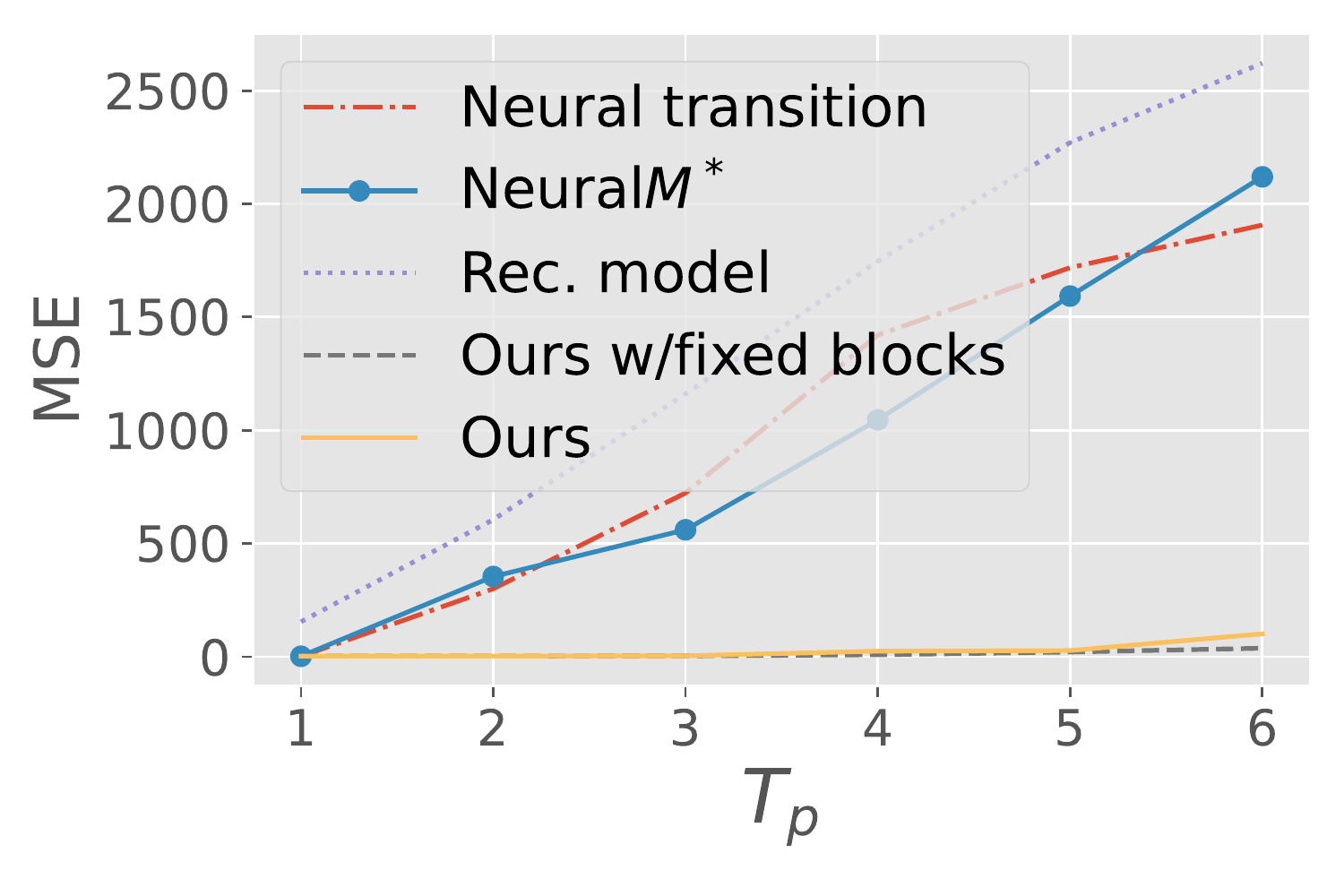}
    }
    \caption{Prediction errors $\mathcal{L}^p$ with $\Tc=2$ and $\Tp=1,\dots,18$. During the training phase, models are trained to predict the observations only at $\Tp=1$. The prediction errors at $\Tp>1$ indicate the extrapolation performance. The results on SmallNORB can be found in Appnedix ~\ref{sec:supp}.\label{fig:prederror}}
\end{figure}

Figure~\ref{fig:gen_images_mnist} shows the example sequences generated by the proposed model and comparative models. 
Figure \ref{fig:prederror} presents the prediction performance at $\Tc + \Tp$ when $\Tc=2$. 
To produce this result, we back-propagated the prediction error at $\Tp=1$ to the encoder during the training, and the prediction at $\Tp > 1$ was used to evaluate the extrapolation performance.   
Our method successfully predicts the images for $\Tp\geq 1$.
Neural transition and Neural$M^*$ had almost the same prediction performance at $\Tp=1$, but they both failed in extrapolation.
Our \textit{fixed 2$\times$2 blocks} variant failed in extrapolation as well. 
This might be because the over-regularized structure of 2x2 block hindered with the training of the SGD optimization~\cite{frankle2018lottery}.  

To evaluate how our learned representation relates to the structural features in the dataset, we also regressed the factors of transition from $\bfM^*$ and regressed the class of the digits from $\Phi(s_1)$ 
(Figures~\ref{fig:actpred} and \ref{fig:cls} in Appendix \ref{sec:supp}).
SimCLR~\cite{chen2020simple} and contrastive predictive coding~(CPC)~\cite{van2018representation, henaff2020data} are tested as baselines.
Please see Appendix~\ref{sec:exp_settings} for the detailed experimental settings for SimCLR and CPC.
Our method yields the representation with better prediction performance than the comparative methods on the test datasets.

\subsection{Equivariance performance}\label{sec:equiv_pf}
\begin{figure}[t]
    \centering
    \subcaptionbox{
    Dependency of $\bfM^*$ on sequence. The 2x3 tiled images in the leftmost panel represents two sequences $\bs^{(1)}, \bs^{(2)}$ with the same transition action $g$. We consider the effects of $\bfM^{*(1)}$ and $\bfM^{*(2)}$ inferred respectively from $\bs^{(1)}$ and $\bs^{(2)}$.
    Neural$M^*$ fails in prediction when $\bfM^{*(2)}$ is used to predict $\bs^{(1)}$. Our method does not fail by this swap, indicating $\bfM^{*(2)} \cong \bfM^{*(1)}$. 
    \label{fig:gen_swap}}[1.0
    \linewidth]{\includegraphics[scale=0.18]{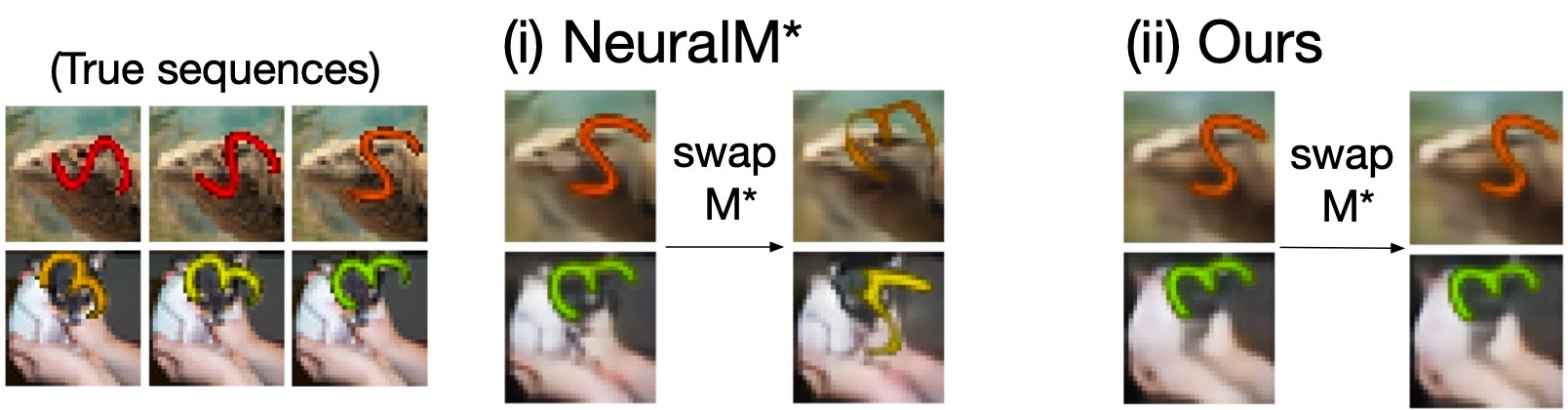}}%
    \vspace{2.5mm}
    
    \centering
    \subcaptionbox{Equivariance performance based on $\mathcal{L}^p$(Eq.\eqref{eq:pred_loss}) and $\mathcal{L}^p_{\rm equiv}$(Eq.\eqref{eq:equiv_loss}) with $\Tc=2$ and $\Tp=1$. To evaluate equivariance errors on more difficult settings, we used all of digits in Sequential MNIST-bg for both training and test sets.\label{tab:equiv_error}}[1
    \linewidth]{
    \begin{tabular}{lrrrrrrrr}
        \toprule
             &\multicolumn{2}{c}{MNIST} &
             \multicolumn{2}{c}{MNIST-bg} & \multicolumn{2}{c}{3DShapes} & \multicolumn{2}{c}{SmallNORB} \\  
        \cmidrule{2-3} \cmidrule{4-5} \cmidrule{6-7} \cmidrule{8-9}
        Method&  \multicolumn{1}{c}{$\mathcal{L}^p$}  & $\mathcal{L}_{\rm equiv}^p$ 
              &  \multicolumn{1}{c}{$\mathcal{L}^p$}  & $\mathcal{L}_{\rm equiv}^p$   
              &  \multicolumn{1}{c}{$\mathcal{L}^p$}  & $\mathcal{L}_{\rm equiv}^p$   
              &  \multicolumn{1}{c}{$\mathcal{L}^p$}  & $\mathcal{L}_{\rm equiv}^p$   \\
        \midrule 
        Rec. Model      &48.91&64.22& 87.05& 95.66 &153.39& 258.20&57.01&78.13 \\
        Neural$M^*$     &4.99& 64.25& 20.60 & 83.18 & 2.09& 217.73&28.98&53.24\\
        MSP (Ours)      &6.42& \textbf{15.91}& 27.38 & \textbf{36.41} &2.74& \textbf{2.87}& 31.14& \textbf{44.77} \\  
        \bottomrule
    \end{tabular}}%
    \caption{ Quantitative and qualitative evaluation of learned equivariance. 
    }
    \label{fig:equiv}
\end{figure}

As we have described through Section \ref{sec:method} and \ref{sec:enc_alg}, the equivariance is achieved when 
$M^*(\bs(s_1, g) | \Phi)$ in \eqref{eq:pred_loss} does not depend on $s_1$, where we recall that $\bs(s_1, g)$ represents the sequence that begins with $s_1$ and transitions with $g$.
To see how much the trained model is equivariant to the transformations in the sequential dataset, we therefore calculated the \textit{equivariance error}, which is the prediction error from applying $\bfM^*(\bs|\Phi)$ to $\Phi(s'_{T_c})$ for a pair $\bs \neq \bs'$ that transitions with the same $g$.  In other words, when $T_c=2$, we compute the following; 
\begin{align}
\begin{split}
 \mathcal{L}_{\rm equiv}^p := \mathbb{E}_g \mathbb{E}_{\bs, \bs' \in \mathcal{S}(g)} [ \| \Psi(\bfM^*(\bs|\Phi) ~\Phi(s'_2))  - s'_3  \|_2^2]
 \end{split} \label{eq:equiv_loss}
\end{align}
where $\mathcal{S}(g)$ represents the set of all sequences that transition with $g$. For each pair of $\bs \neq \bs'$ we set $\Tc=2$ and $\Tp=1$ as done in the experiments in the previous sections.
Table~\ref{tab:equiv_error} compares the Neural$M^*$ method against our method in terms of the equivariance error. 
Figure~\ref{fig:gen_swap} shows the result of applying $\bfM^*(\bs| \Phi)$ on $\Phi(s'_2)$ and applying $\bfM^*(\bs'| \Phi)$ on $\Phi(s_2)$. 
We see that when we swap $\bfM^*$ this way, Neural$M^*$ also swaps the digits; this implies that $\bfM^*$ learned by  Neural$M^*$ encodes the \textit{sequence} specific information together with the transition. 
On the other hand, swapping of $\bfM^*$ does not affect the prediction for our method, suggesting that our method is succeeding to learn an equivariant model.
This is somewhat surprising, because our model does not have the explicit mechanism to enforce the full equivariance ($P=I$ in Section~\ref{sec:method}).  


\subsection{Structures found by simultaneous block-diagonalization of $\bfM^*$s}\label{sec:sbd_results}
We have seen in the previous section that the trained $\Phi$ is fully equivariant to transformations $\mathcal{G}$, which implies each $M^*$ is a representation of the corresponding transformation of $g \in \mathcal{G}$.
As we describe in Section~\ref{sec:sbd}, we apply simultaneous block-diaognalization to uncover \textit{the symmetry structure} captured by $M^*$s.
Figure~\ref{fig:sbd} shows the structure revealed by simultaneous block-diagonalization through the change of basis $U$ trained by minimizing the average of $\mathcal{L}_{\rm bd}$ in eq.\eqref{eq:bd_loss} over all $\bs$.
 Figure~\ref{fig:steer_3dshapes} shows the results of applying transformation of only one block. 
 We can see that each block only alters one factor of variation. Our results suggest that the learned $\bfM^{*}$ captures the hidden disentangled structure of the group actions behind the datasets.

\begin{figure}[t]
    \centering
    \begin{minipage}{0.49\textwidth}
    \subcaptionbox{Simeltaneously block-diagonalized matrices\label{fig:sbd}}[\linewidth]{
        \includegraphics[scale=0.075]{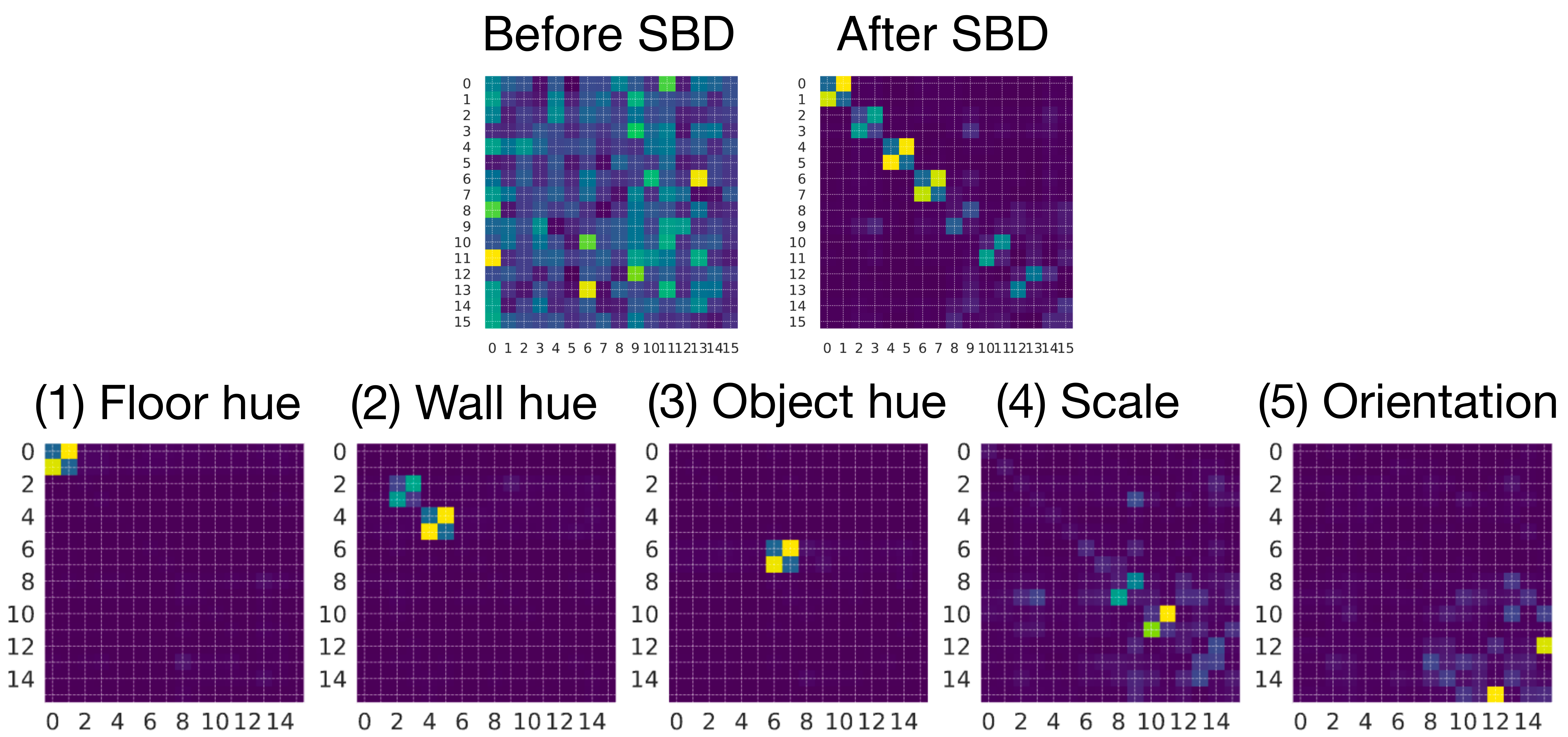}
    }
    \end{minipage}
    \begin{minipage}{0.5\textwidth}
    \subcaptionbox{Disentangled transition representation on 3DShapes\label{fig:steer_3dshapes}}[\linewidth]{
        \includegraphics[scale=0.06]{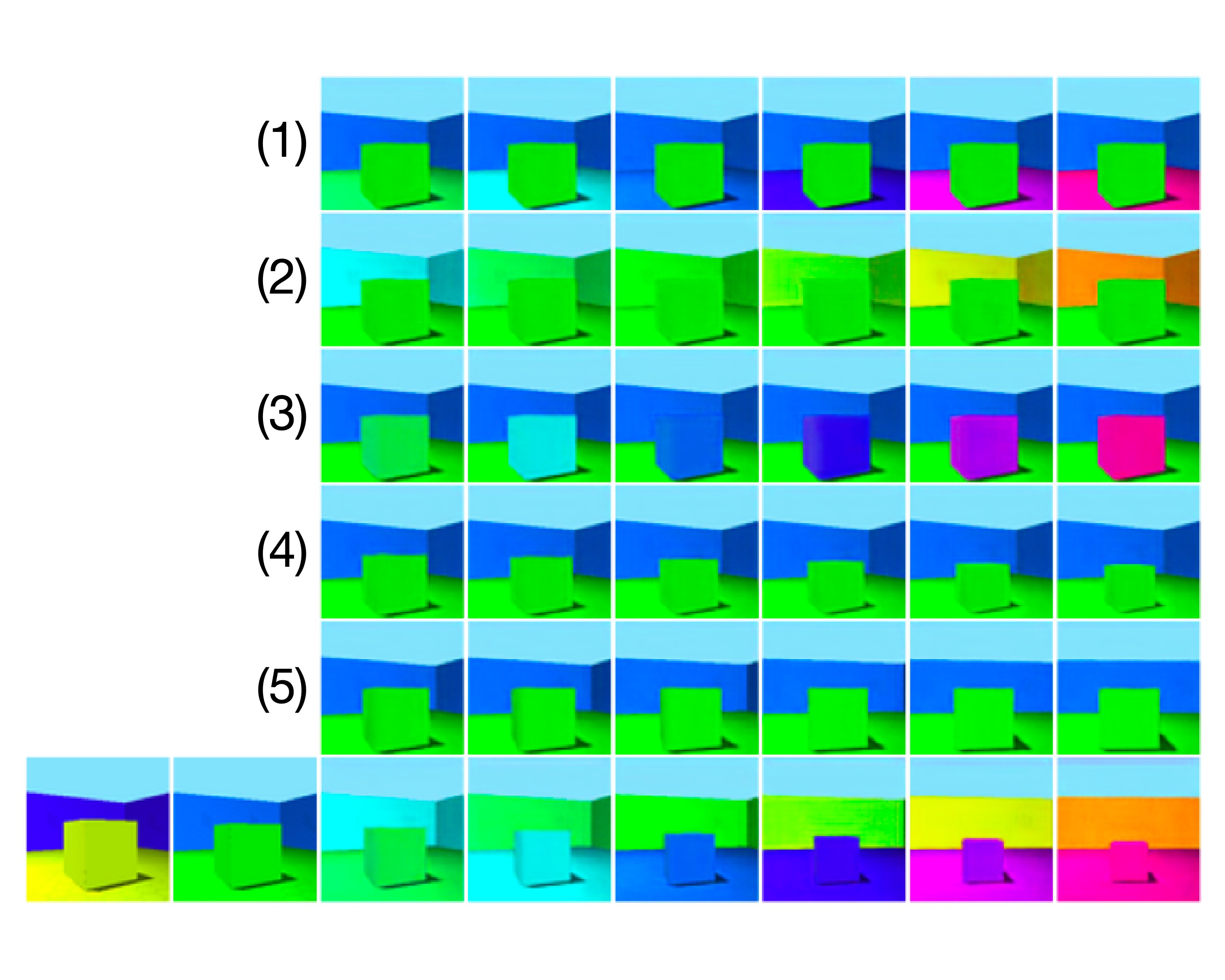}
    }
    \end{minipage}
    \caption{
    (a) Simultaneous block-diagonalization~(SBD) of $\bfM^{*}$. The top right matrix is the visualization of $abs(\irrepM^*-I)$ averaged over all of the training sequences. Each of the five matrices below is the visualization of $abs(\irrepM^*-I)$ averaged over the set of sequences on which only a single factor was varied. 
    Coordinates are permuted for better visibility.
    (b) Sequences generated by applying the transformation of just one block. 
    To produce the disentangled sequences in each row from the leftmost two images in the bottom row, we performed the internal optimization of $M^*$ while setting \textit{all but the block positions corresponding to each factor of variation} to be identity.  We elaborate this result and
    the results for Sequential MNIST, MNIST-bg and SmallNORB in Appendix~\ref{sec:supp}.
    } 
    \label{fig:steer}
\end{figure}





\subsection{Extension to the sequences with constant acceleration}\label{sec:accl}
We have seen that \textit{meta-sequential prediction} successfully learns an equivariant structure from the set of constant-velocity sequences. In this section, we show that we can extend our concept to the set of sequences sharing the stationarity of higher order (constant acceleration). 
By definition, the pair of $\Phi(s_t)$ and $\Phi(s_{t-1})$ encodes the information about the velocity at $t$. When the multiplicity $m$ is sufficiently large\footnote{If $m$ is less than $a$, we cannot obtain the pseudo inverse because of the rank deficient in $\Phi(s_{t-1}) \Phi(s_{t-1})^{\rm T}$. Thus $m$ should be at least larger than $a$.}, the velocity can be estimated by: ${}^{1}\bfM_t = \Phi(s_t)\Phi(s_{t-1})^\dagger$.
Because this would yield a sequence of velocities, we can simply apply our method again to estimate the constant acceleration by
${}^{2}\bfM^{*} = {}^{1}\bfM_{+1} {}^{1}\bfM_{+0}^{\dagger}$
 where ${}^{1}\bfM_{+0}=[{}^{1}\bfM_2;...;{}^{1}\bfM_{\Tc-1}] \in \mathbb{R}^{a \times (\Tc-2) a}$
${}^{1}\bfM_{+1}=[{}^{1}\bfM_3;...;{}^{1}\bfM_{\Tc}] \in \mathbb{R}^{a \times (\Tc-2) a}$.
We can then predict the future representation $\tilde{s}_t$ for $t=\Tc+1,...,\Tp$ by 
\begin{align}
    \tilde{s}_t = \Psi\left(\left({\textstyle\prod}_{t'=\Tc+1}^{t}{}^{1}\tilde{M}_{t'}\right)  \Phi(s_{\Tc}) \right)~{\rm where}~ {}^{1}\tilde{M}_{t'}={}^{2}\bfM^{*(t'-\Tc)}~{}^{1}\bfM_{\Tc}.
\end{align}
where $\prod$ represents multiplications from left. We train $\Phi$ and $\Psi$ by minimizing the mean squared error between $\tilde{s}_t$ and $s_t$ for $t=\Tc+1,...,T$ as in Eq.\eqref{eq:pred_loss}.
To create a sequence of constant acceleration from MNIST dataset, we only used shape and color rotations. We chose the initial velocity for these  rotations randomly on the interval $[-\pi/5, \pi/5)$ for each sequence, and chose the acceleration on the interval $[-\pi/40, \pi/40)$. 
The results are shown in~Figure~\ref{fig:mnist_accl}. The Neural transition and the constant-velocity version of our method failed to predict the accelerated sequence, while the 2nd order model succeeded in predicting the sequence even after $\Tp>5$. Also, Figure~\ref{fig:equiv_error_mnist_accl} and Table~\ref{tab:equiv_error_small_norb} in Appendix shows that the accelerated version of our proposed model again achieves learning the equivariance relation.

\section{Discussion \& Limitations}
\paragraph{How is full equivariance achieved in our method?}
The theoretical results we provided in section \ref{sec:method} only assure that $\bfM(g, x)$ and $\bfM(g, x')$ are similar when the underlying group is commutative, compact and connected.  
However, as we have shown experimentally, our method seems to be learning $\Phi$ for which the estimators of $\bfM$ satisfy $\bfM^* (\bs(g, x)|\Phi) \cong \bfM^* (\bs(g, x')|\Phi)$. 
This can be happening because our framework and the training method based on the internal optimization in the latent space is somehow encouraging $\bfM^*$ to be orthogonal (See the loss curve of orthogonality of $M^*$ in Appendix~\ref{sec:supp}).
Maybe this is forcing the change of matrix $P$ such that $P\bfM(g, x)P^{-1}= \bfM(g, x')$ to be also rotations as well, which commutes with $\bfM(g, x)$ itself. 
Also, Figure~\ref{fig:gen_comparison} and~\ref{fig:prederror} show that, as reported in ~\cite{keller2021predictive}, the models trained with reconstruction loss like~\eqref{eq:rec_loss} does not well capture the group transformation behind the sequences: the encoder representation was found to be significantly worse than that of the model trained with \eqref{eq:pred_loss}. 
We hypothesize that \eqref{eq:rec_loss} fails to remove the sequence-specific information from $\bfM^*$, while \eqref{eq:rec_loss} succeeds to do so by training the model to be able to predict the unseen images.



\paragraph{Towards learning symmetries from more realistic observations}

As we are making a connection between our prediction framework and group equivariance, we are essentially assuming that the transitions are always invertible, because group is closed under inversion.
However, this might not be always the case in real world applications; for instance, if the image sequences are the sequential renderings of a rotating 3D object, the transitions are generally not invertible because only a part of the object is visible at each time step.
We experimented Sequential ShapeNet, which is created from ShapeNet~\cite{shapenet2015} dataset. A series of rendered images is generated by sequentially applying 3D rotations of different speeds for each axis.
Generated results on Sequential ShapeNet (See ~\ref{fig:gen_shapenet} in Appendix~\ref{sec:gen} show that actually our current method was not able to generate the images on 3D rotated datasets.
If the transitions are not invertible, some measures must be taken in order to resolve the indeterminacy, such as probabilistic modeling or additional structural inductive bias. 

\begin{figure}[t]
    \centering
    \subcaptionbox{Generated images on \textit{accelerated} Sequential MNIST. }[0.64
    \textwidth]{
    \centering
    \includegraphics[scale=0.125]{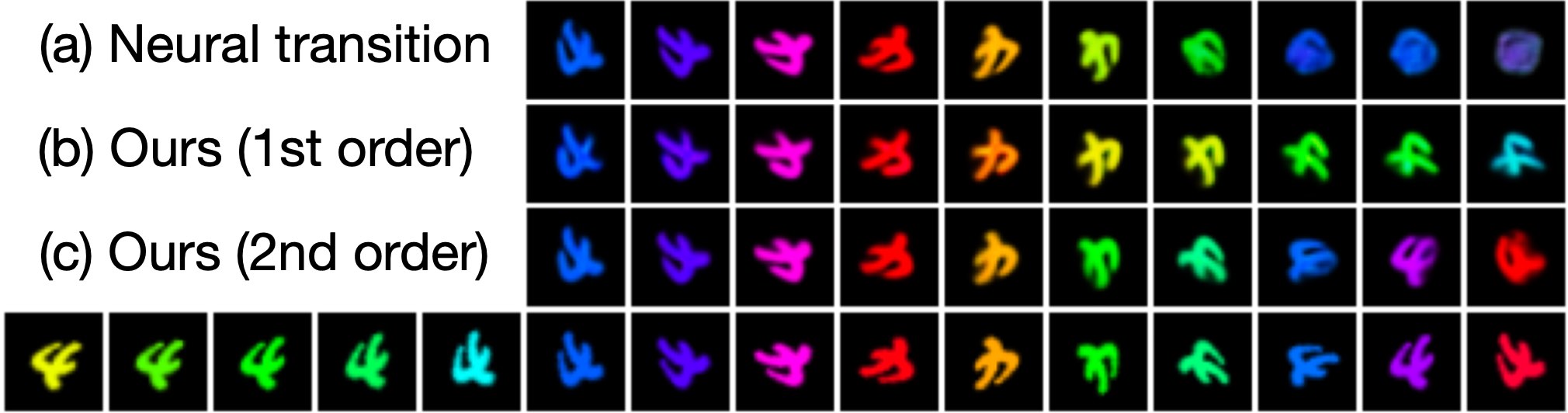}
    }
    \subcaptionbox{Prediction error.}[0.35\textwidth]{
    \centering
    \includegraphics[scale=0.28]{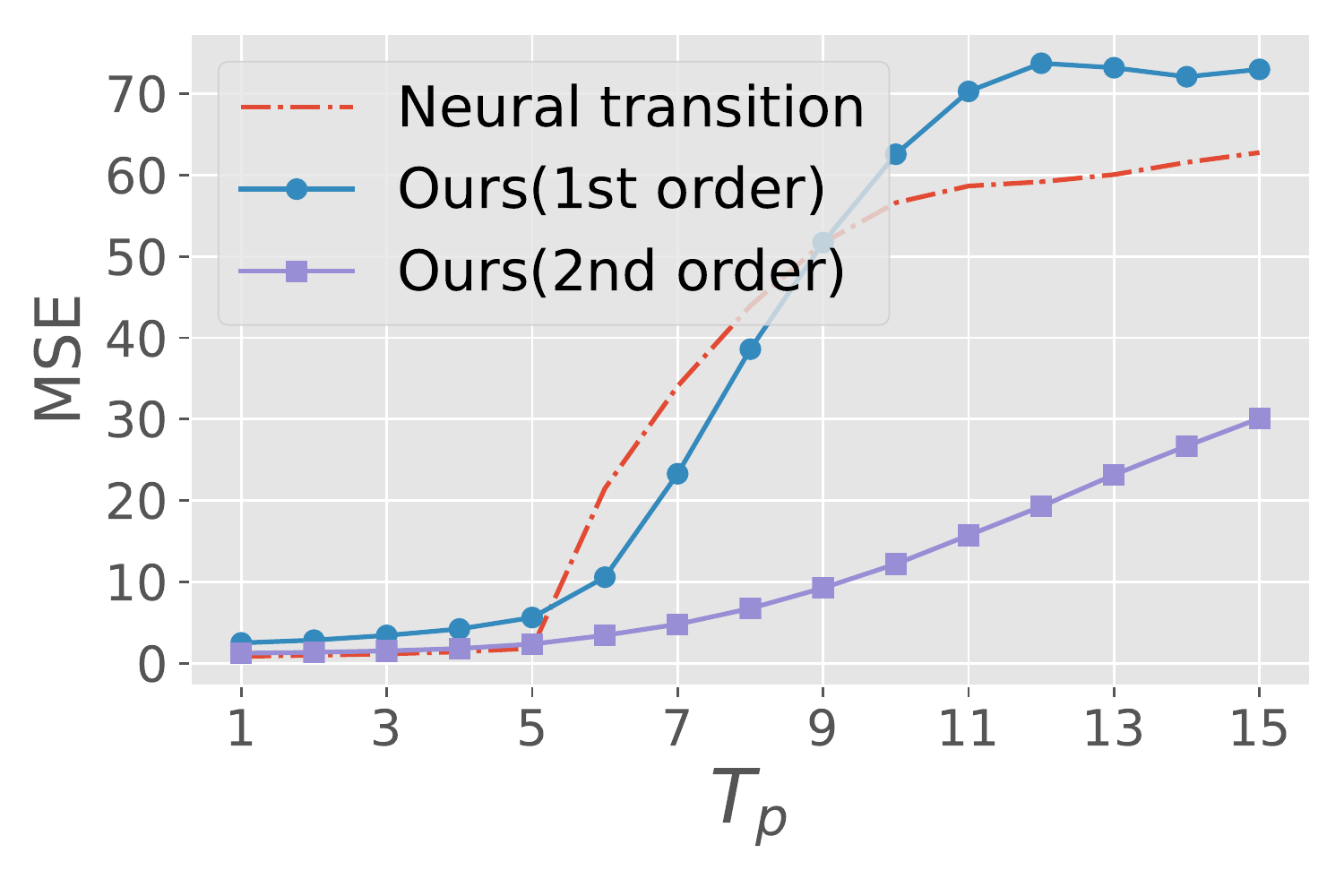}
    }
    \caption{
    Results on accelerated Sequential MNIST.
    Every model was trained with $\Tc=5$, $\Tp=5$.
    Neural transition overfitted and collapsed from $\Tp=5$~(beyond the training horizon). 
    }
    \label{fig:mnist_accl}
\end{figure}

\paragraph{Broader impact}\label{sec:broader_impact}
Because our study generally contributes to predictions and extrapolation, it has as much potential to negatively affect the society as most other prediction methods. 
In particular, applications of our method to image sequence can be potentially integrated into weapon systems, for example. 
At the same time, our unsupervised learning of the symmetrical structure from sequential datasets may also contribute to new discoveries in the systems of finance, medical science, physics and other fields of ML such as reinforcement learning.


\bibliography{main}

\section*{Checklist}

\begin{enumerate}

\item For all authors...
\begin{enumerate}
  \item Do the main claims made in the abstract and introduction accurately reflect the paper's contributions and scope?
    \answerYes{The abstract and introduction reflects the contributions of this paper.}
  \item Did you describe the limitations of your work?
    \answerYes{In the Discussion \& Limitations section, we discuss the type of dataset to which our work cannot be directly applied, as well as the theoretical problem that has not been solved.}
  \item Did you discuss any potential negative societal impacts of your work?
    \answerYes{We included the broader impact discussion in Section~\ref{sec:broader_impact}}
  \item Have you read the ethics review guidelines and ensured that your paper conforms to them?
    \answerYes{We read the guidelines and we confirmed that the contents in our paper conformed to them.}
\end{enumerate}

\item If you are including theoretical results...
\begin{enumerate}
  \item Did you state the full set of assumptions of all theoretical results?
    \answerYes{We included the full set of assumptions. We provide the details of the assumptions in the supplementary material.  }
        \item Did you include complete proofs of all theoretical results?
    \answerYes{We provide complete proofs of theoretical results in Supplementary material. }
\end{enumerate}

\item If you ran experiments...
\begin{enumerate}
  \item Did you include the code, data, and instructions needed to reproduce the main experimental results (either in the supplemental material or as a URL)?
    \answerYes{We provided code for reproducing the experimental results as a supplementary material.}
  \item Did you specify all the training details (e.g., data splits, hyperparameters, how they were chosen)?
    \answerYes{We explained how we chose the hyperparameters and how we split the dataset into training and test sets in section~\ref{sec:experiments} and appendix section \ref{sec:exp_settings}.}
        \item Did you report error bars (e.g., with respect to the random seed after running experiments multiple times)?
    \answerYes{We include the standard deviation values of equivariance performance in Appendix~\ref{sec:supp} (We omitted them in the main submission because of the space limitation) and the linear classification performance of the transition parameters in Appendix~\ref{sec:supp}.}
        \item Did you include the total amount of compute and the type of resources used (e.g., type of GPUs, internal cluster, or cloud provider)?
    \answerYes{We put the information of GPUs which we used to train the models. Also we included approximated costs for reproducing the full results of experiments in~\ref{sec:exp_settings}.}
\end{enumerate}

\item If you are using existing assets (e.g., code, data, models) or curating/releasing new assets...
\begin{enumerate}
  \item If your work uses existing assets, did you cite the creators?
    \answerYes{We cited \cite{lecun1998gradient} for MNIST, \cite{shapenet2015} for ShapeNet, \cite{kim2018disentangling} for 3Dshapes and \cite{lecun2004learning} for SmallNORB.}
  \item Did you mention the license of the assets?
    \answerYes{For each dataaset, we included the URL to the license information in the reference section.}
  \item Did you include any new assets either in the supplemental material or as a URL?
    \answerNA{}
  \item Did you discuss whether and how consent was obtained from people whose data you're using/curating?
    \answerYes{For ShapeNet dataset, we are approved to download the assets.}
  \item Did you discuss whether the data you are using/curating contains personally identifiable information or offensive content?
  \answerYes{We don't find any personally identifiable information in the images used in our experiments.}
    
\end{enumerate}

\item If you used crowdsourcing or conducted research with human subjects...
\begin{enumerate}
  \item Did you include the full text of instructions given to participants and screenshots, if applicable?
   \answerNA{}
  \item Did you describe any potential participant risks, with links to Institutional Review Board (IRB) approvals, if applicable?
   \answerNA{}
  \item Did you include the estimated hourly wage paid to participants and the total amount spent on participant compensation?
   \answerNA{}
\end{enumerate}

\end{enumerate}

\newpage


\appendix
\part{Appendix} 

\parttoc

\section{Supplemental results} \label{sec:supp}

\subsection{Qualitative and quantitative results on the prediction}
\begin{figure}[H]
    \centering
    \subcaptionbox{SmallNORB\label{fig:prederrors_smallNORB}}[.55\linewidth]{
        \includegraphics[scale=0.5]{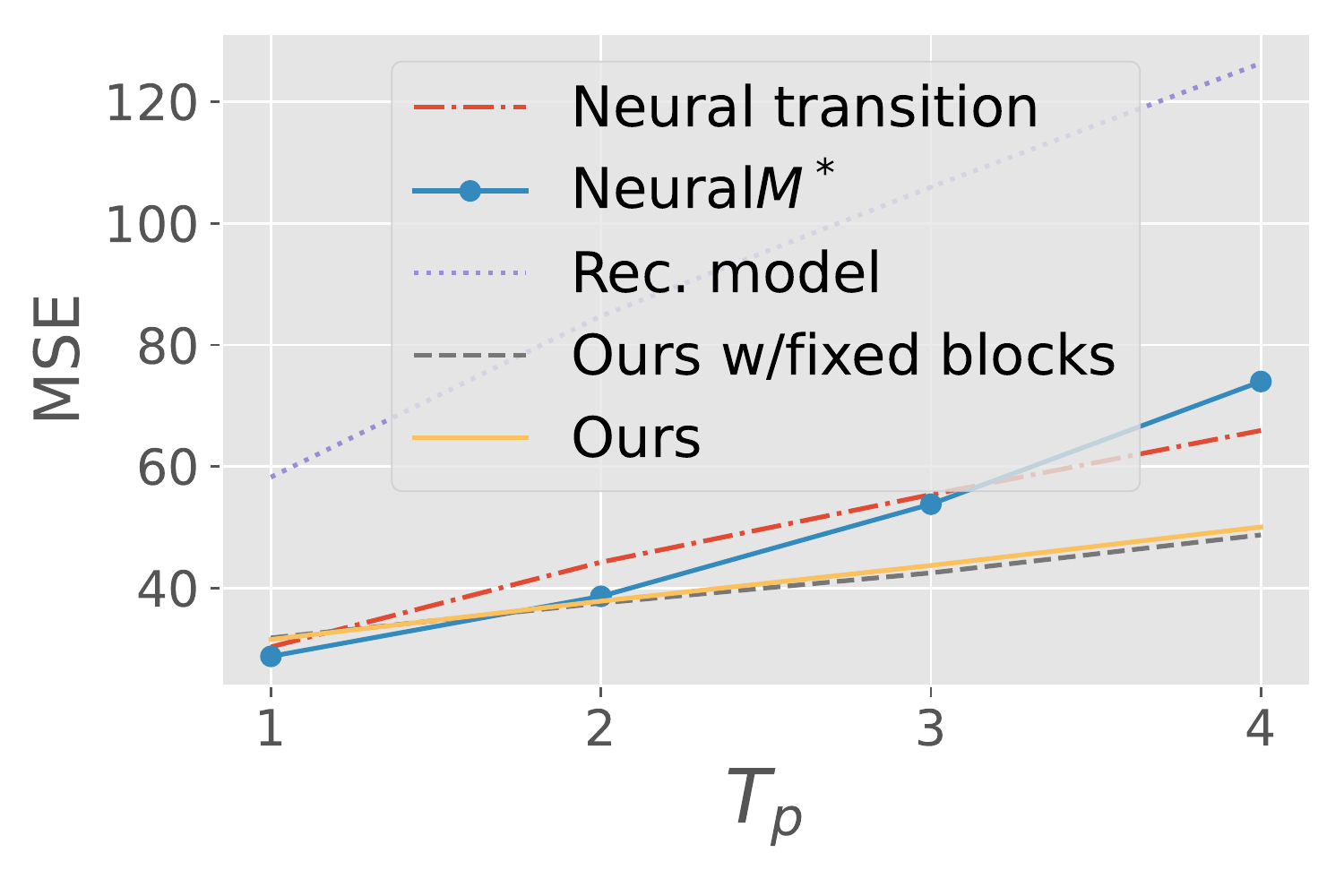}
    }
    \caption{Prediction errors on smallNORB. During the training phase, the models were trained to predict the observations only at $\tp=1$. The prediction errors at $\tp>1$ indicate the extrapolation performance.\label{fig:prederror_smallNORB}}
\end{figure}

\begin{figure}[H]
    \centering
    \subcaptionbox{Sequential MNIST\label{fig:actpred_mnist}}[.49\linewidth]{
        \includegraphics[scale=0.23]{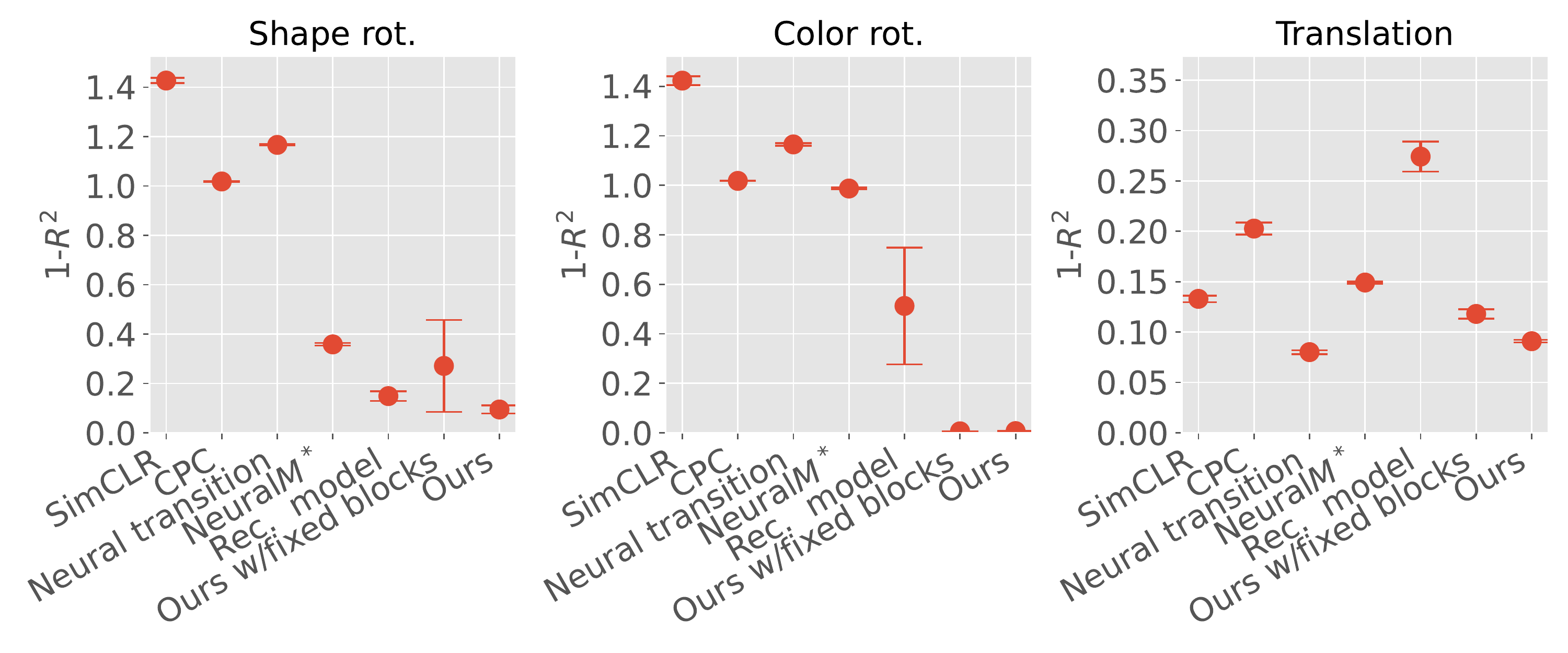}
    }
    \centering
    \subcaptionbox{Sequential MNIST-bg\label{fig:actpred_mnist-bg}}[.49\linewidth]{
        \includegraphics[scale=0.23]{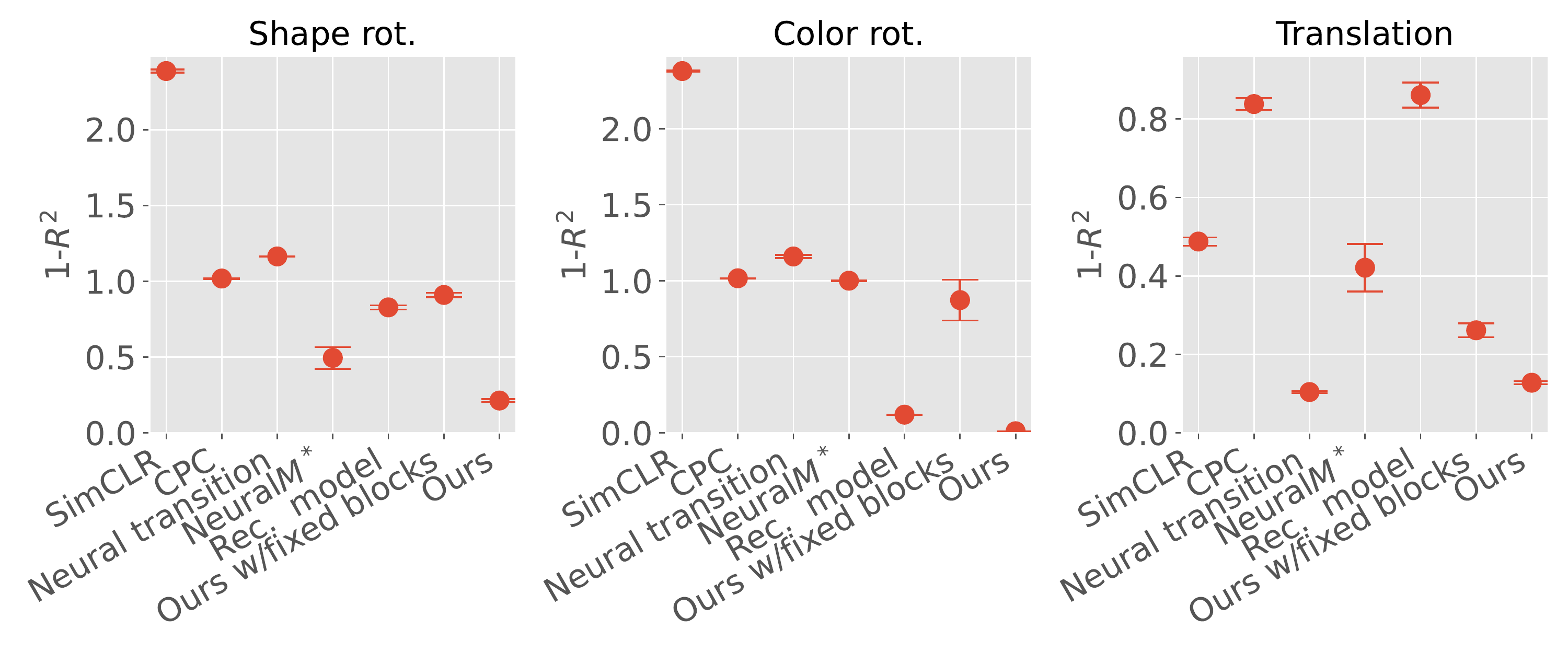}
    }
    \caption{The results of linearly regressing the true transition parameters from $\bfM^*$. For the performance evaluation, we used $1-R^2$ scores~(The value of 0 indicates the perfect prediction and 1 indicates the performance is chance level. $1-R^2 > 1$ can happen when the model significantly overfits to the training set). For the color rotation and the shape rotation, $(\cos(v), \sin(v))$ was used as the target value where $v$ is the angle velocity.
    For this experiment, we trained the models on a set of sequences generated from \textit{digit 4 class} onl, and trained/evaluated the linear regression performance on the trained models' features on a set of sequences created from all digit classes in MNIST. 
    Because SimCLR, CPC and Neural transition do not directly compute $\bfM^*$, the linear regression was computed from the concatenation of the two consecutive latent representations that were used in our method for the computation of $M^*$.  }
    \label{fig:actpred}
\end{figure}

\begin{figure}[H]
    \centering
    \includegraphics[scale=0.35]{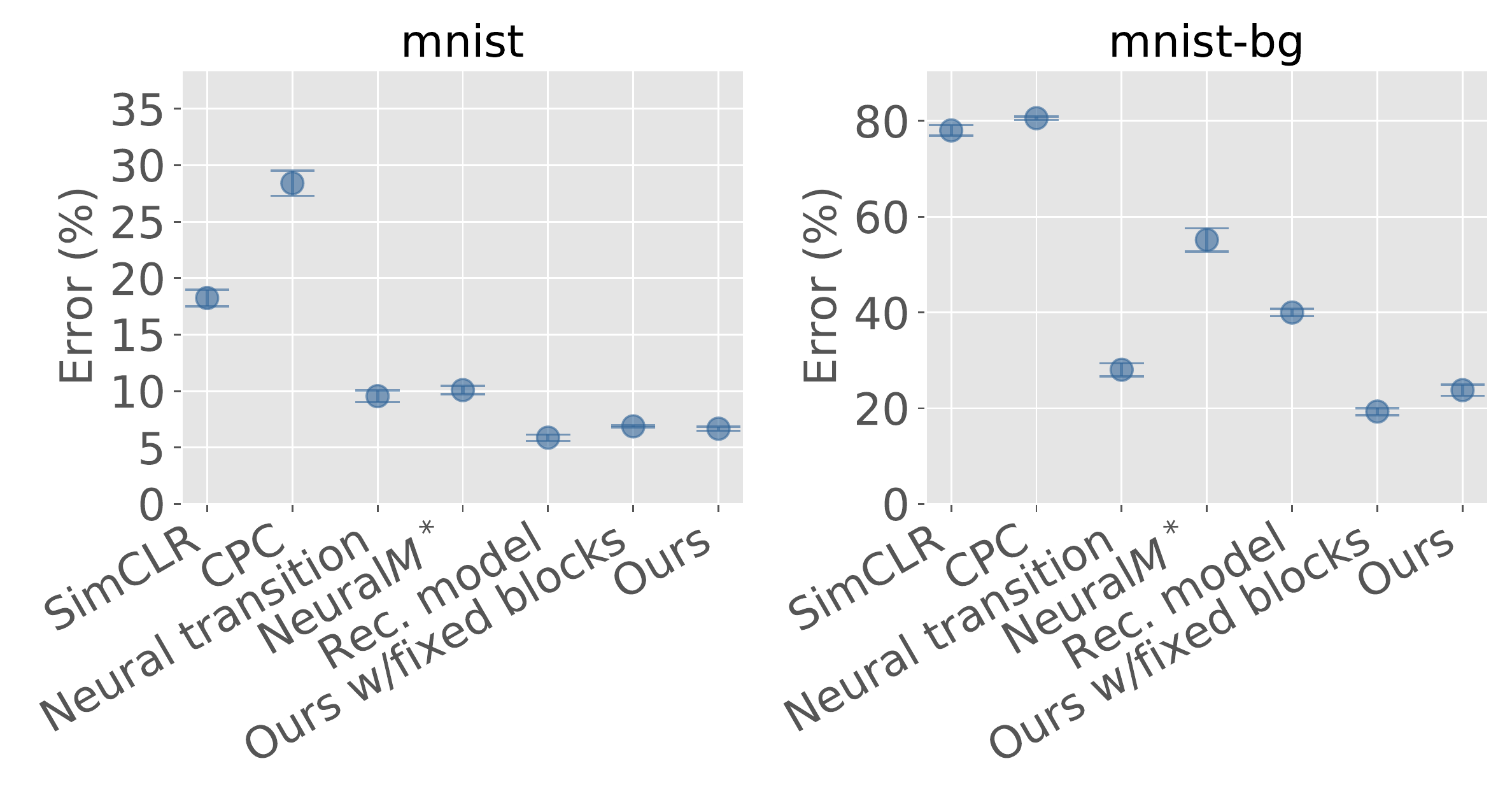}
    \caption{The results of digit classification evaluation on the sequential MNIST and MNIST-bg datasets. 
    For this experiment, we trained the models on a set of sequences generated from only \textit{digit 4 class}. We trained and evaluated the softmax classifier on the feature $\Phi(s_1)$ where $s_1$s are generated from all digit classes in MNIST.
    \label{fig:cls}}
\end{figure}

\newpage

\subsection{Equivariance performance}
Figure~\ref{fig:diffM} shows $(M^*(\bs)-M^*(\bs'))^2$ for the pairs of sequences 
that transition with same $g$ (e.g. $\bs = \bs(s_1, g),  \bs' = \bs'(s_1', g)$). 
We see that $M^*$s computed from the representation learned by our method do not differ across $\bs$ and $\bs'$.    
This can also be confirmed visually in the generated sequences as well (Figure~\ref{fig:transfer}). 
\begin{figure}[H]
    \centering
    \includegraphics[scale=0.14]{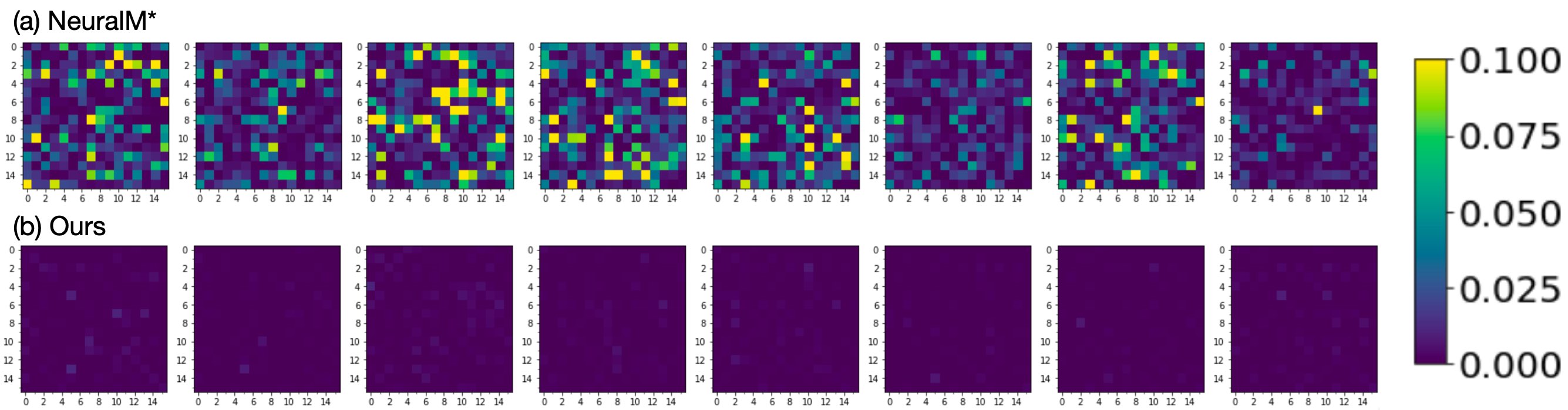}
    \caption{Visualization of $(M^*(\bs)-M^*(\bs'))^2$ where $\bs,\bs'$  that transition with the same $g$. }
    \label{fig:diffM}
\end{figure}

\begin{figure}[H]
    \centering
    \subcaptionbox{Sequential MNIST\label{fig:transfer_mnist}}[.9\linewidth]{
        \includegraphics[scale=0.22]{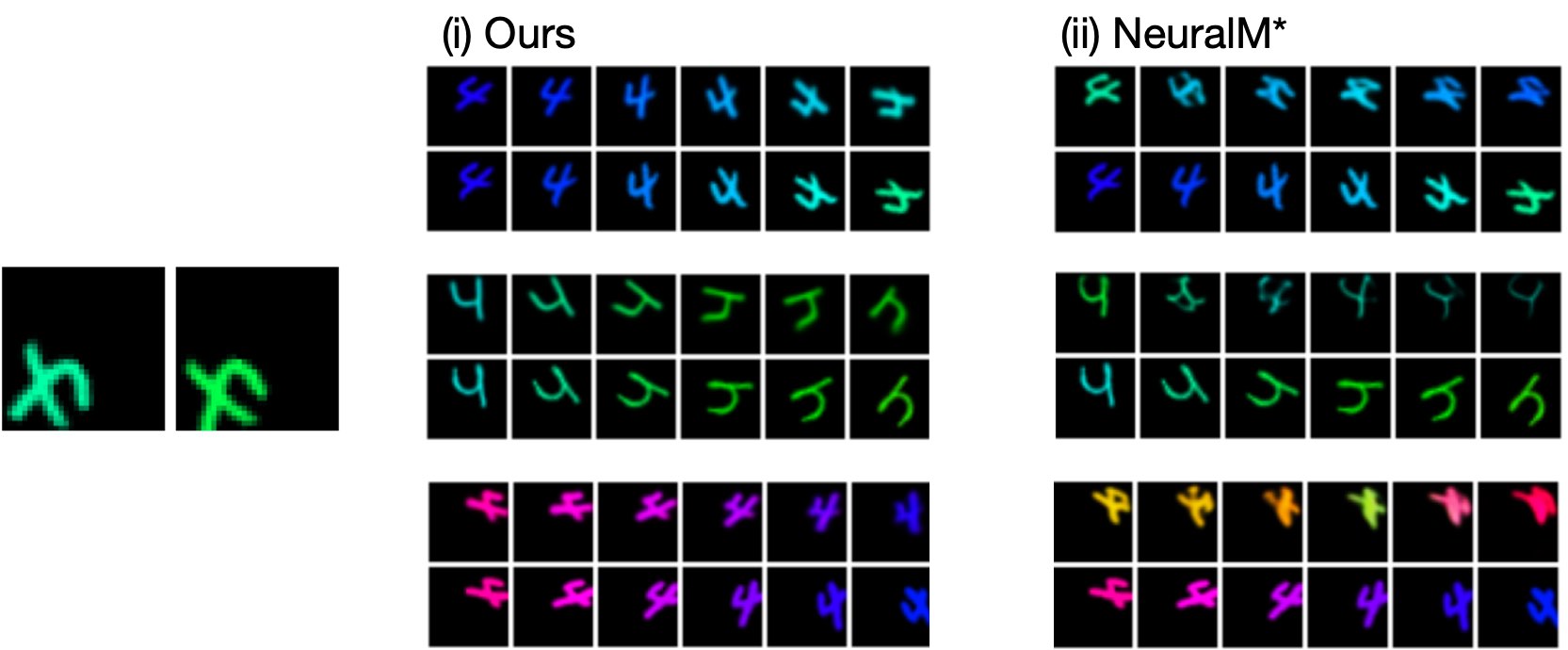}
    }
    \centering
    \subcaptionbox{Sequential MNIST-bg\label{fig:transfer_mnist-bg}}[.9\linewidth]{
        \includegraphics[scale=0.22]{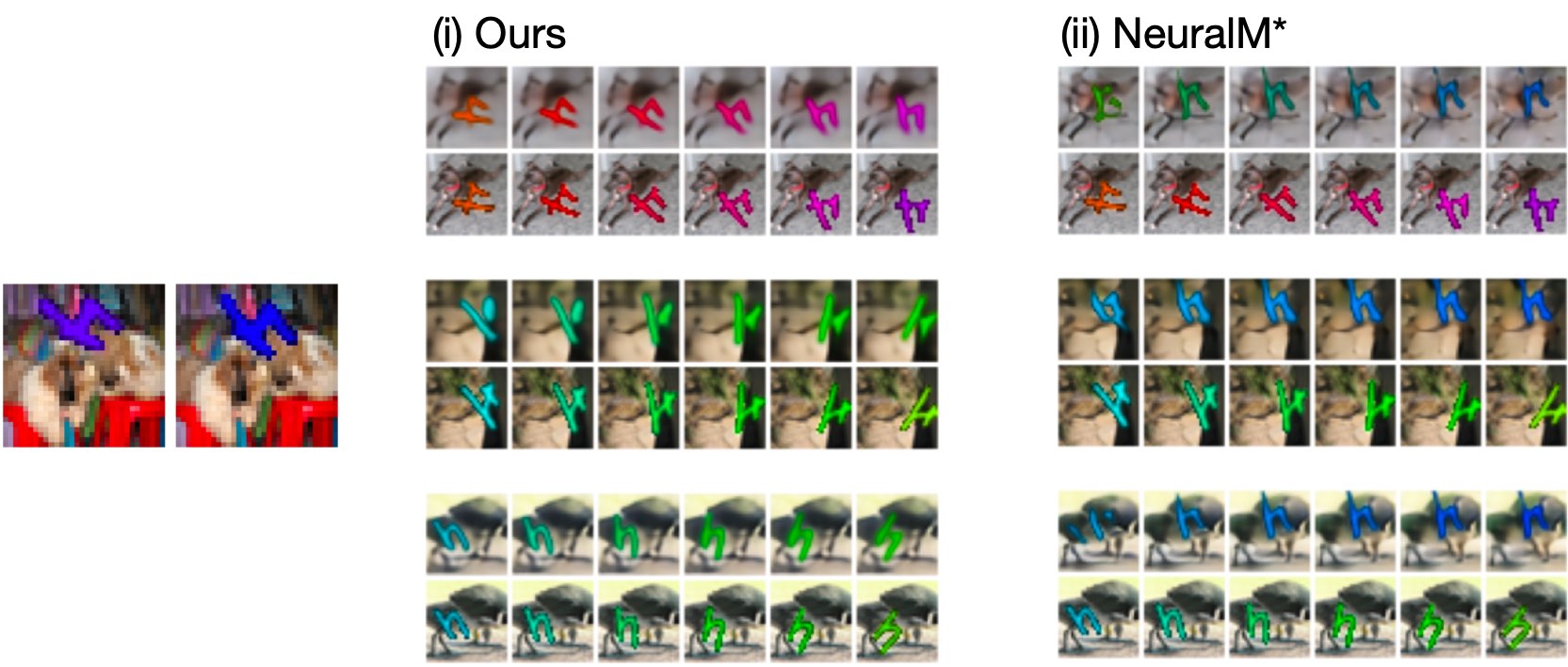}
    }
    \caption{The result of transferring $M^*$ computed from one sequence to other sequences. For both sequential MNIST and sequential MNIST-bg, $M^*$ was computed from the two consecutive images placed on the left edge of the figure. In each pair of rows shown on the right, 
    the top row corresponds to the generated sequence and the bottom row corresponds to the ground truth sequence that transitions with the same $g$ that was used to create the two consecutive images on the left.
    We see that each $M^*$ computed from our representation acts on different sequences in the same way.
    \label{fig:transfer}}
\end{figure}

\begin{figure}[H]
    \centering
    \subcaptionbox{3DShapes\label{fig:transfer_3dshapes}}[.9\linewidth]{
        \includegraphics[scale=0.22]{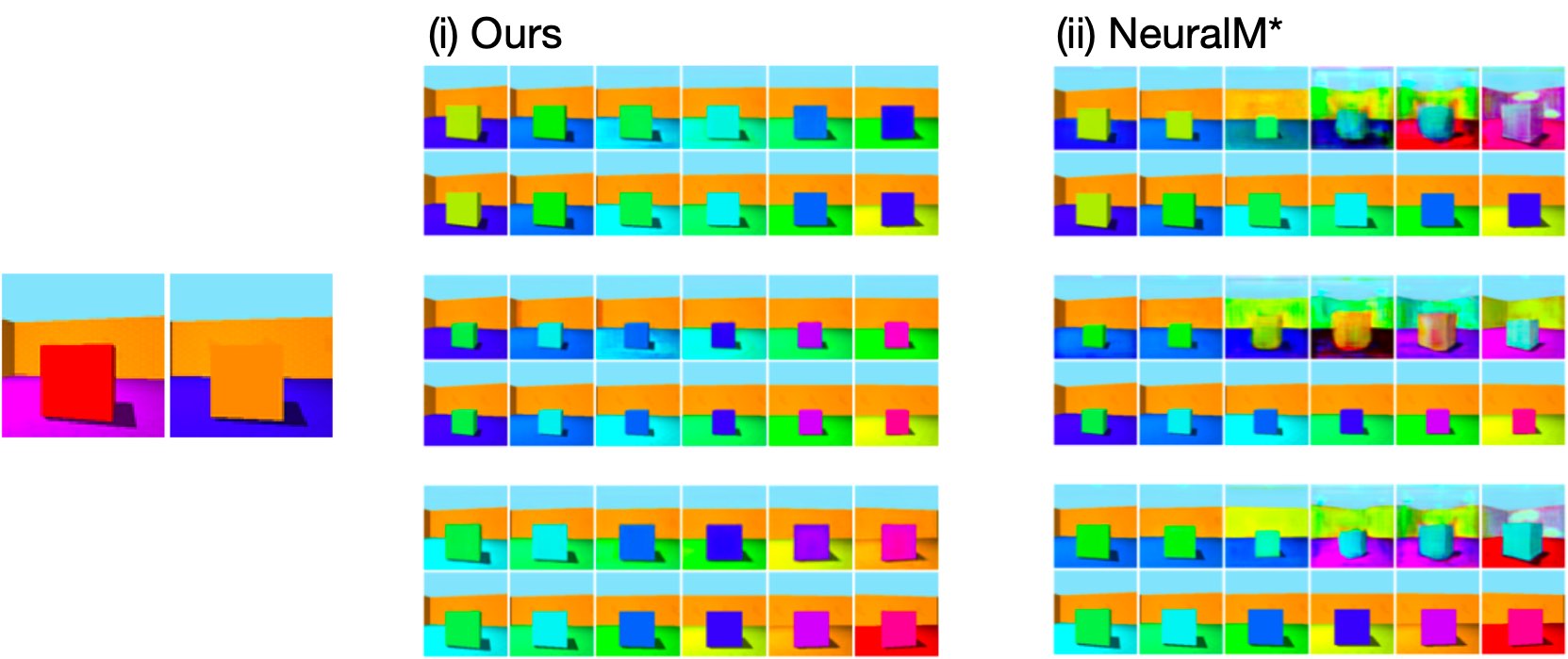}
    }
    
    \subcaptionbox{SmallNORB\label{fig:transfer_smallNORB}}[.9\linewidth]{
        \includegraphics[scale=0.18]{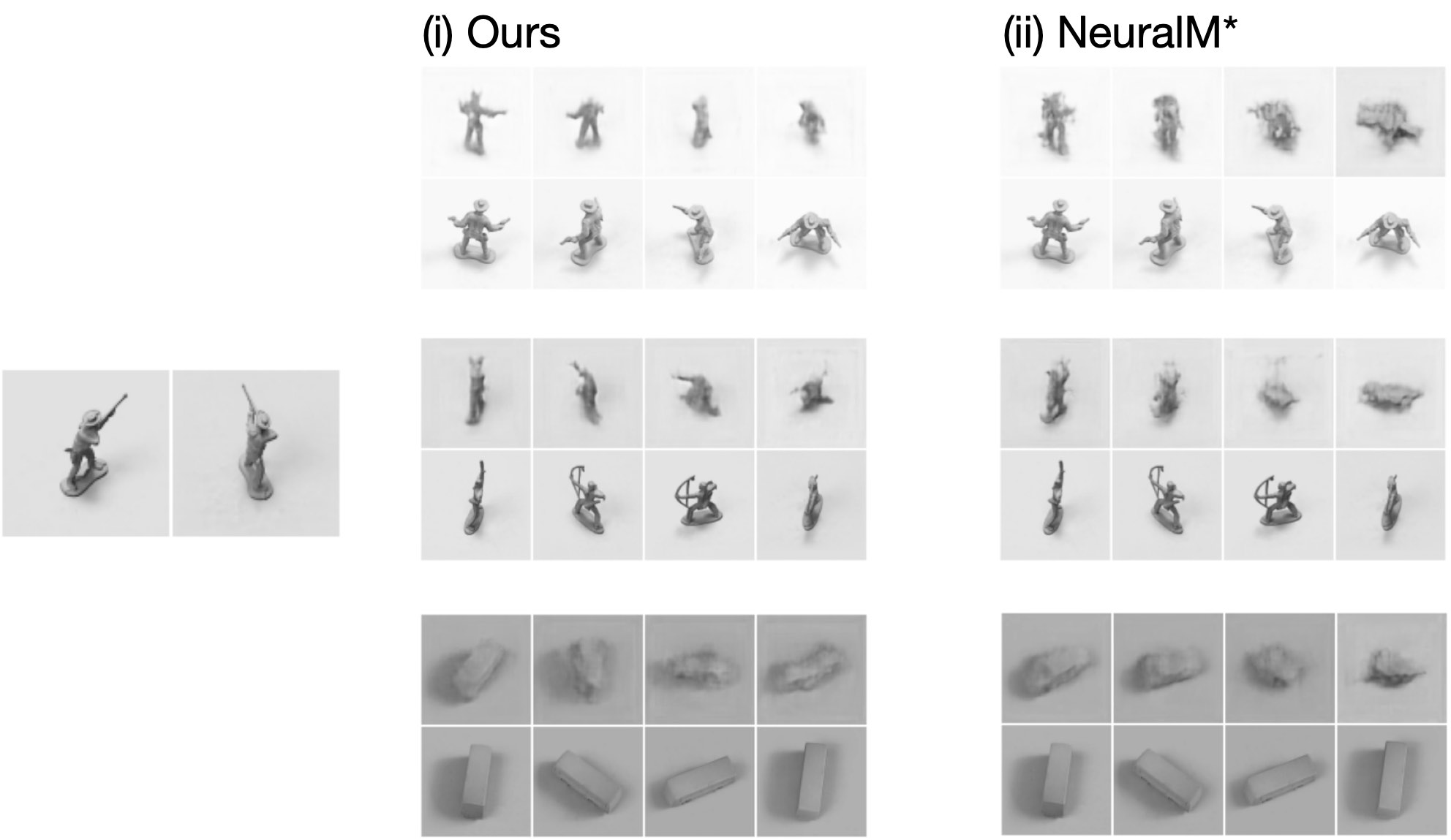}
    }
    \caption{The result of transferring $M^*$ on 3DShapes and SmallNORB. The visualization follows the same protocol as in Figure~\ref{fig:transfer}.
    \label{fig:transfer_2}}
\end{figure}

\begin{figure}[H]
    \centering
    \subcaptionbox{Sequential MNIST\label{fig:equiv_error_mnist}}[.24\linewidth]{
        \includegraphics[scale=0.22]{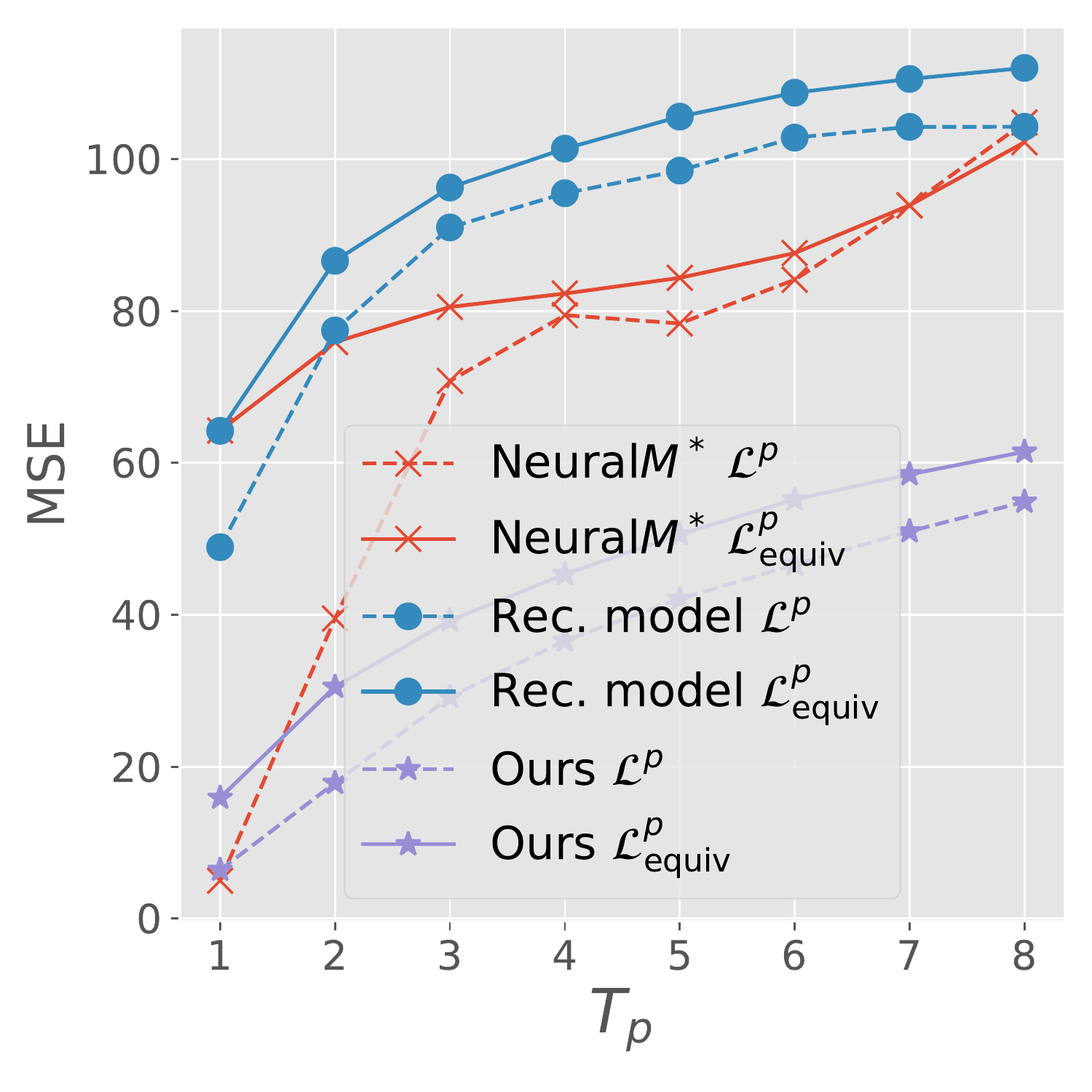}
    }
    \subcaptionbox{Sequential MNIST w/ digit 4\label{fig:equiv_error_mnist_bg}}[.24\linewidth]{
        \includegraphics[scale=0.22]{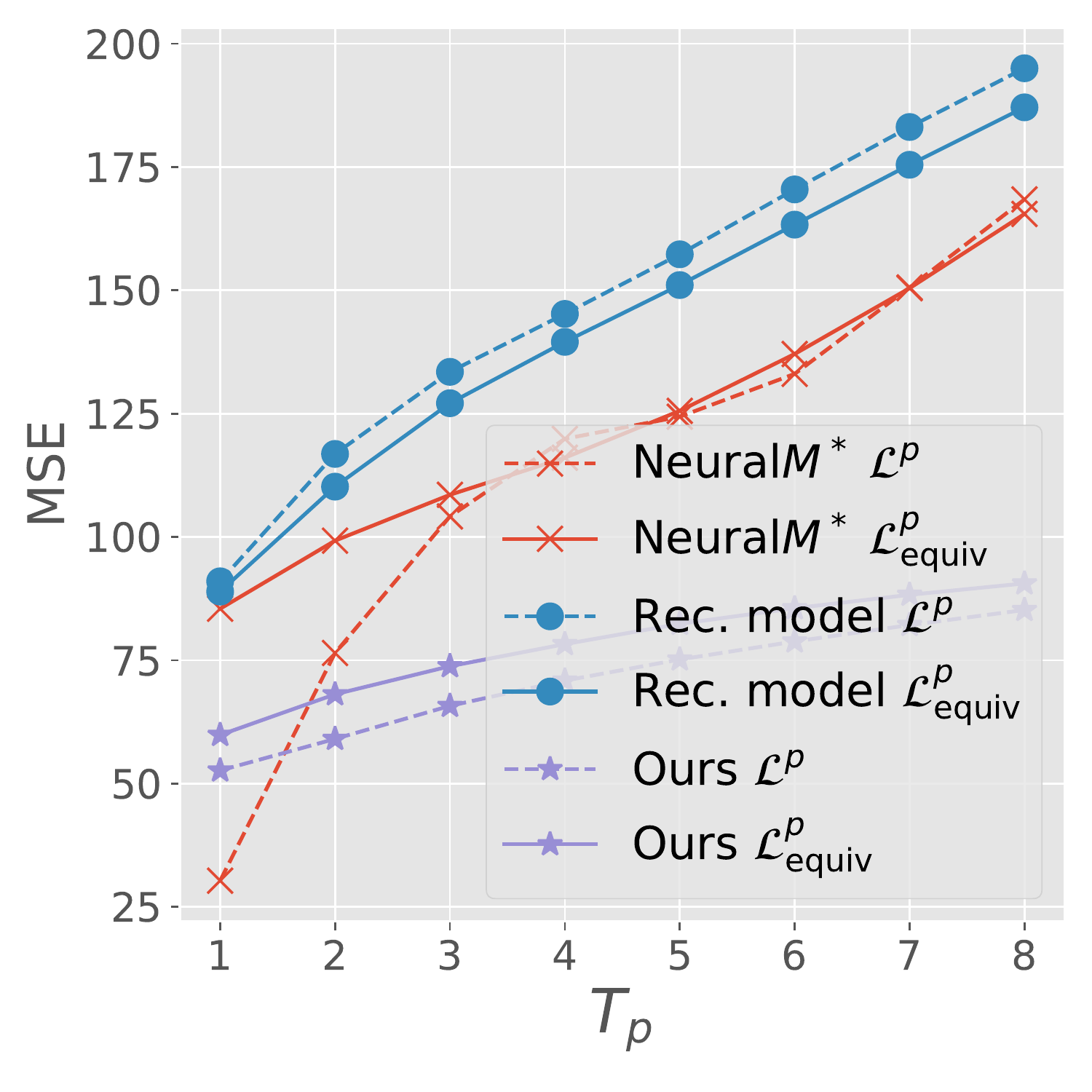}
    }
    \subcaptionbox{Sequential MNIST w/ all digits   \label{fig:equiv_error_mnist_bg_full}}[.24\linewidth]{
        \includegraphics[scale=0.22]{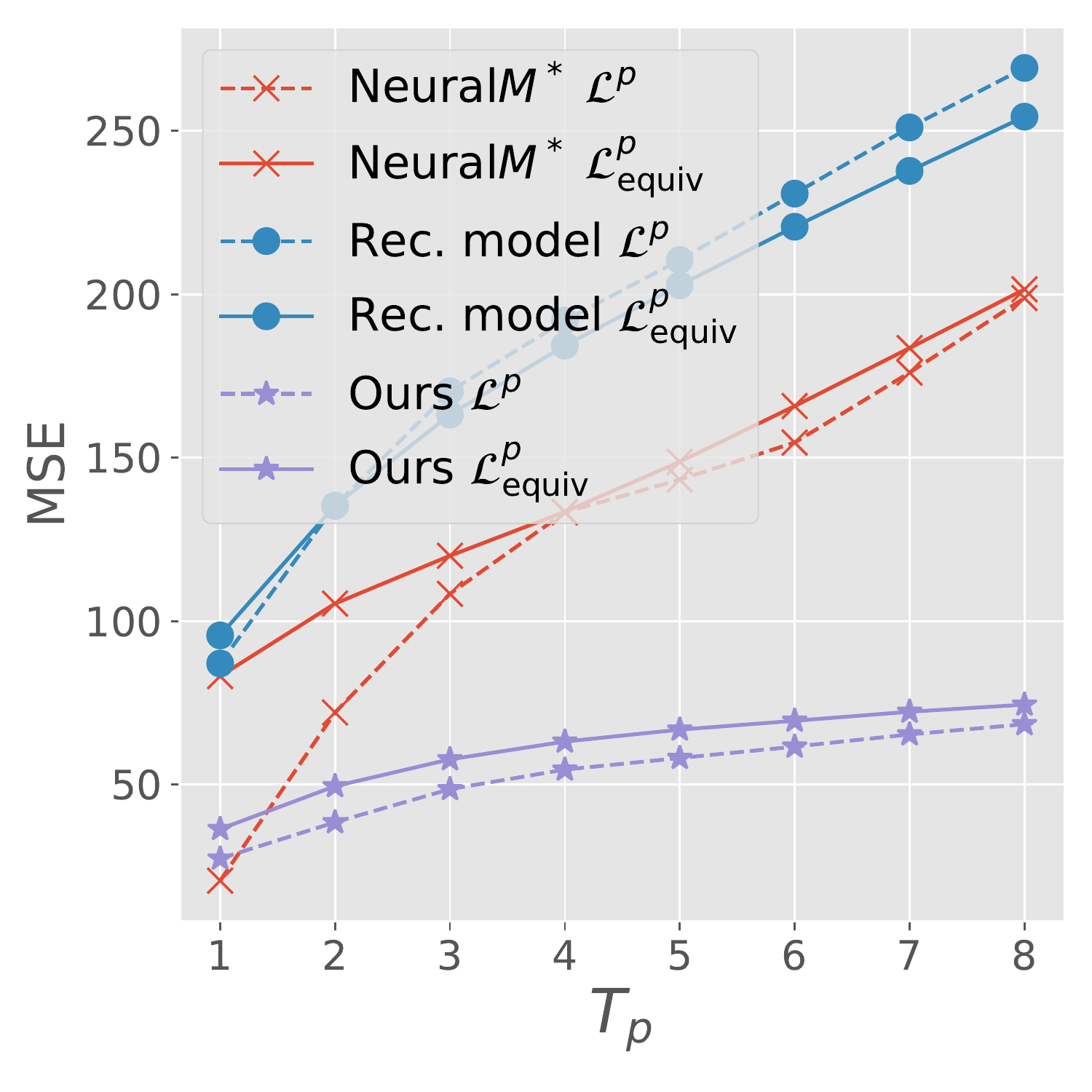}
    }
    
    \subcaptionbox{3DShapes\label{fig:equiv_error_3dshapes}}[.24\linewidth]{
        \includegraphics[scale=0.22]{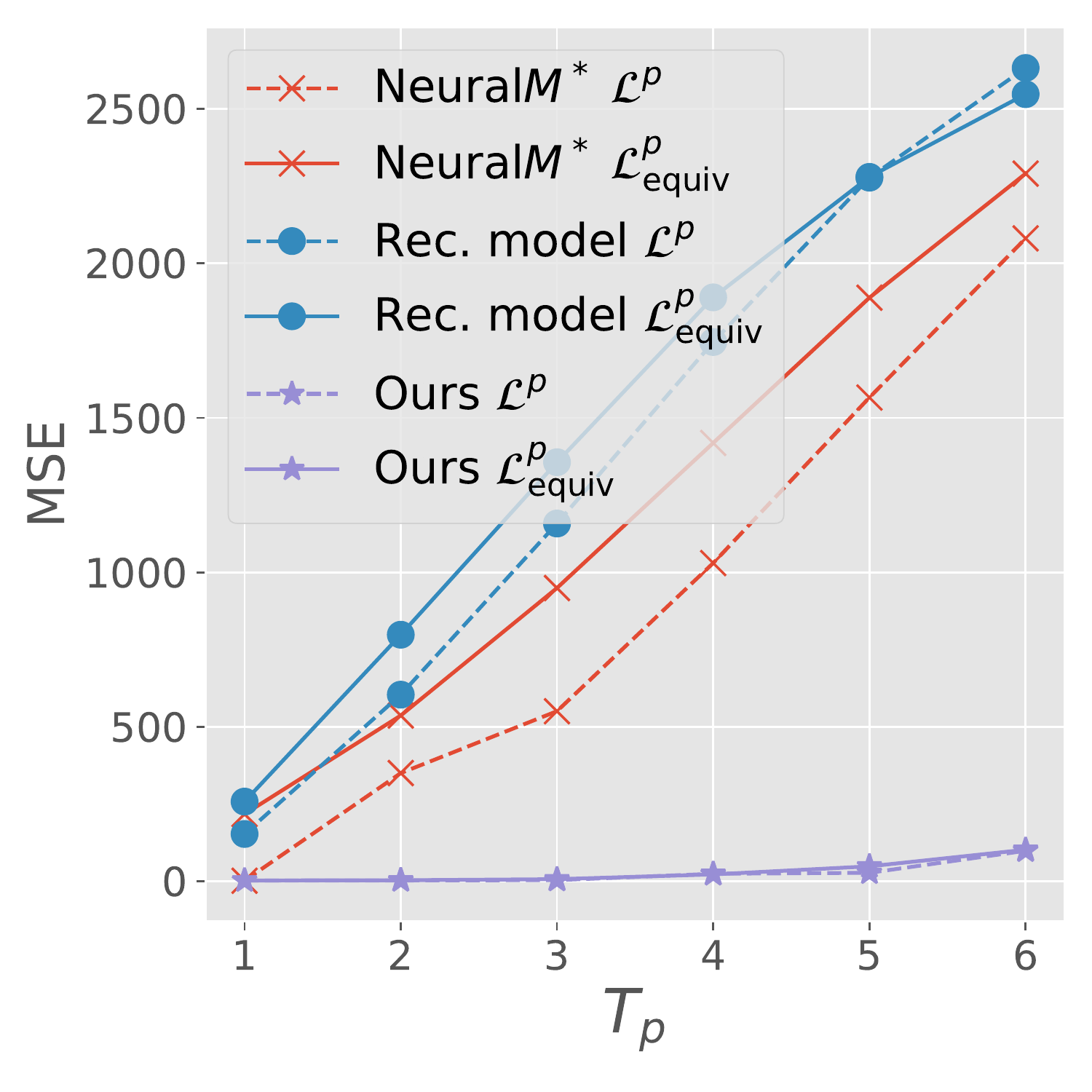}
    }
    \subcaptionbox{SmallNORB\label{fig:equiv_error_smallnorb}}[.24\linewidth]{
        \includegraphics[scale=0.22]{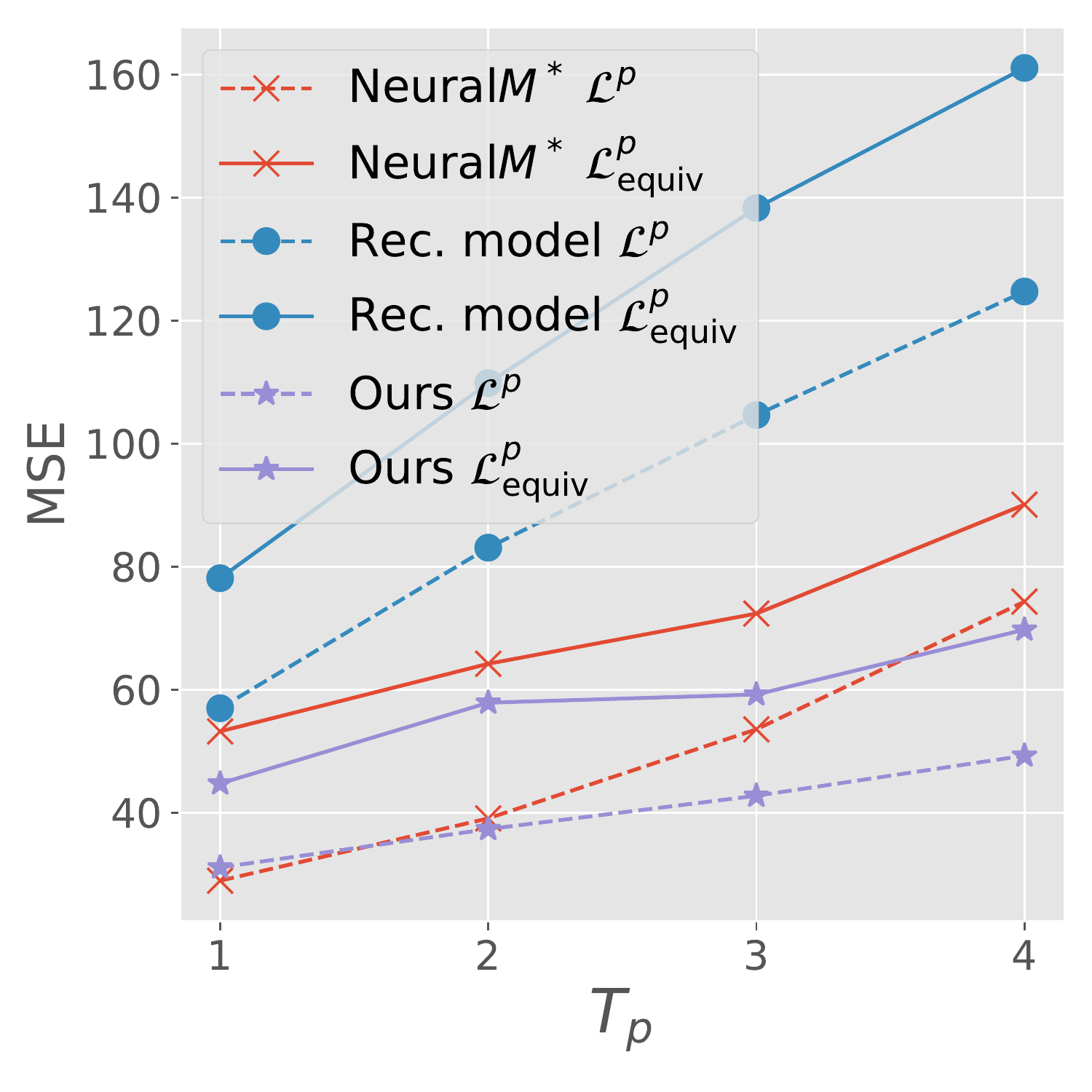}
    }
    \subcaptionbox{Accelerated Sequential MNIST   \label{fig:equiv_error_mnist_accl}}[.24\linewidth]{
        \includegraphics[scale=0.22]{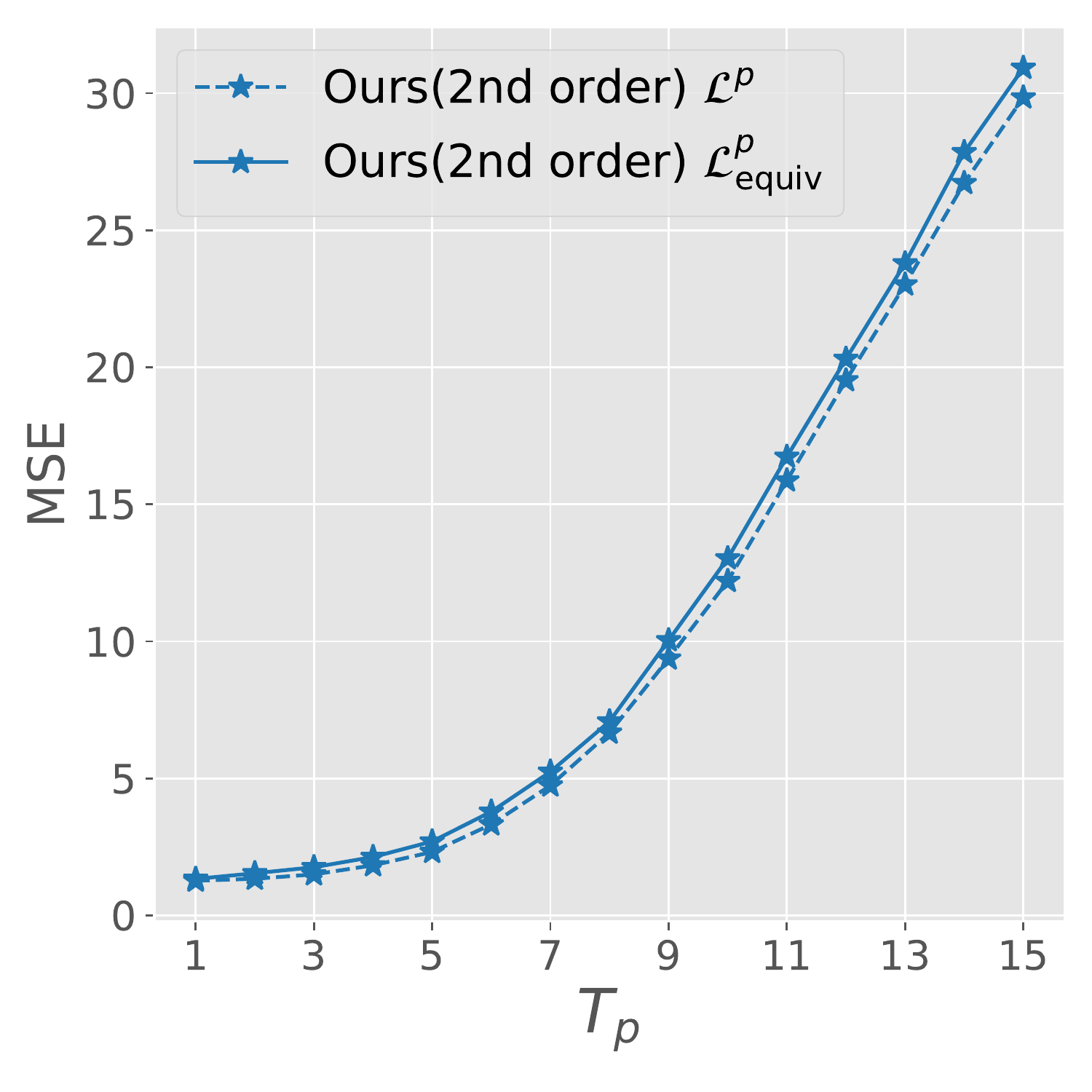}
    }
    \caption{Comparison of the prediction errrors and equivariance errors at $\Tp \geq 1$.
    \label{fig:equiv_error}}
\end{figure}

\begin{figure}[H]
    \centering
    \centering
    \subcaptionbox{Equivariance performance on sequential MNIST and MNIST-bg w/ digit 4 \label{tab:equiv_error_mnist}}[1
    \linewidth]{
    \begin{tabular}{lrrrr}
    \toprule
         &\multicolumn{2}{c}{Seq. MNIST} & \multicolumn{2}{c}{Seq. MNIST-bg (w/ digit 4)}  \\  
    \cmidrule{2-3} \cmidrule{4-5} 
    Method&  \multicolumn{1}{c}{$\mathcal{L}^p$}  & \multicolumn{1}{c}{$\mathcal{L}_{\rm equiv}^p$} &  \multicolumn{1}{c}{$\mathcal{L}^p$}  & \multicolumn{1}{c}{$\mathcal{L}_{\rm equiv}^p$}  \\
    \midrule 
    Rec. Model      &48.91$\pm$4.47 & 64.22$\pm$5.69 & 91.02$\pm$2.22 & 88.93$\pm$2.89   \\
    Neural$M^*$     &4.99$\pm$0.87& 64.25$\pm$2.59  & 30.32$\pm$0.36 & 85.46$\pm$2.66  \\
    MSP (Ours)      &6.42$\pm$0.21& \textbf{15.91}$\pm$0.49 & 52.67$\pm$0.86 & \textbf{59.87}$\pm$1.37  \\  
    \bottomrule
    \end{tabular}}%
    
    \subcaptionbox{Equivariance performance on sequential MNIST-bg  w/ all digits and 3DShapes\label{tab:equiv_error_mnist_and_3dshapes}}[1
    \linewidth]{
    \begin{tabular}{lrrrr}
    \toprule
         & \multicolumn{2}{c}{Seq. MNIST-bg (w/ all digits)} & \multicolumn{2}{c}{3DShapes}  \\  
    \cmidrule{2-3} \cmidrule{4-5}
    Method&  \multicolumn{1}{c}{$\mathcal{L}^p$}  & \multicolumn{1}{c}{$\mathcal{L}_{\rm equiv}^p$} &  \multicolumn{1}{c}{$\mathcal{L}^p$}  & \multicolumn{1}{c}{$\mathcal{L}_{\rm equiv}^p$}  \\
    \midrule 
    Rec. Model     & 87.05$\pm$3.32 & 95.66$\pm$7.71 &153.39$\pm$24.1& 258.20$\pm$25.8 \\
    Neural$M^*$   & 20.60$\pm$0.25 & 83.18$\pm$2.50  &2.09$\pm$0.12& 217.73$\pm$46.7\\
    MSP (Ours)     & 27.38$\pm$0.14& \textbf{36.42}$\pm$0.08 &2.75$\pm$0.25& \textbf{2.87}$\pm$0.30\\  
    \bottomrule
    \end{tabular}}%
    
    \centering
    \subcaptionbox{Equivariance performance on SmallNORB and accelerated sequential MNIST\label{tab:equiv_error_small_norb}}[1
    \linewidth]{
    \begin{tabular}{lrrrr}
    \toprule
         & \multicolumn{2}{c}{SmallNORB} &\multicolumn{2}{c}{Accelerated Seq. MNIST}\\  
    \cmidrule{2-5}       
    Method&  \multicolumn{1}{c}{$\mathcal{L}^p$}  & \multicolumn{1}{c}{$\mathcal{L}_{\rm equiv}^p$} &  \multicolumn{1}{c}{$\mathcal{L}^p$}  & \multicolumn{1}{c}{$\mathcal{L}_{\rm equiv}^p$}  \\
    \midrule 
    Rec. Model  &57.01$\pm$2.69 &78.14$\pm$4.42 \\
    Neural$M^*$  &28.98$\pm$1.25&53.24$\pm$0.64\\
    MSP (Ours)   & 31.14$\pm$0.52 & \textbf{44.77}$\pm$0.38 & 1.27$\pm$ 0.02 & 1.34 $\pm$ 0.03\\  
    \bottomrule
\end{tabular}}%
    \caption{More detailed version of Fig~\ref{fig:equiv_error} with standard deviation values.  
    The statistics in this figure were calculated over three models initialized with different random seeds. 
    For the definition of $\mathcal{L}^p$ and $\mathcal{L}^p_{\rm equiv}$, see \eqref{eq:pred_loss} and \eqref{eq:equiv_loss}.}
    \label{fig:equiv_m_and_std}
\end{figure}
\newpage

\subsection{More results on simultaneous block-diagonalization}

\begin{figure}[H]
    \centering
    \subcaptionbox{Sequential MNIST\label{fig:sbd_mnist}}[.49\linewidth]{
        \includegraphics[scale=0.05]{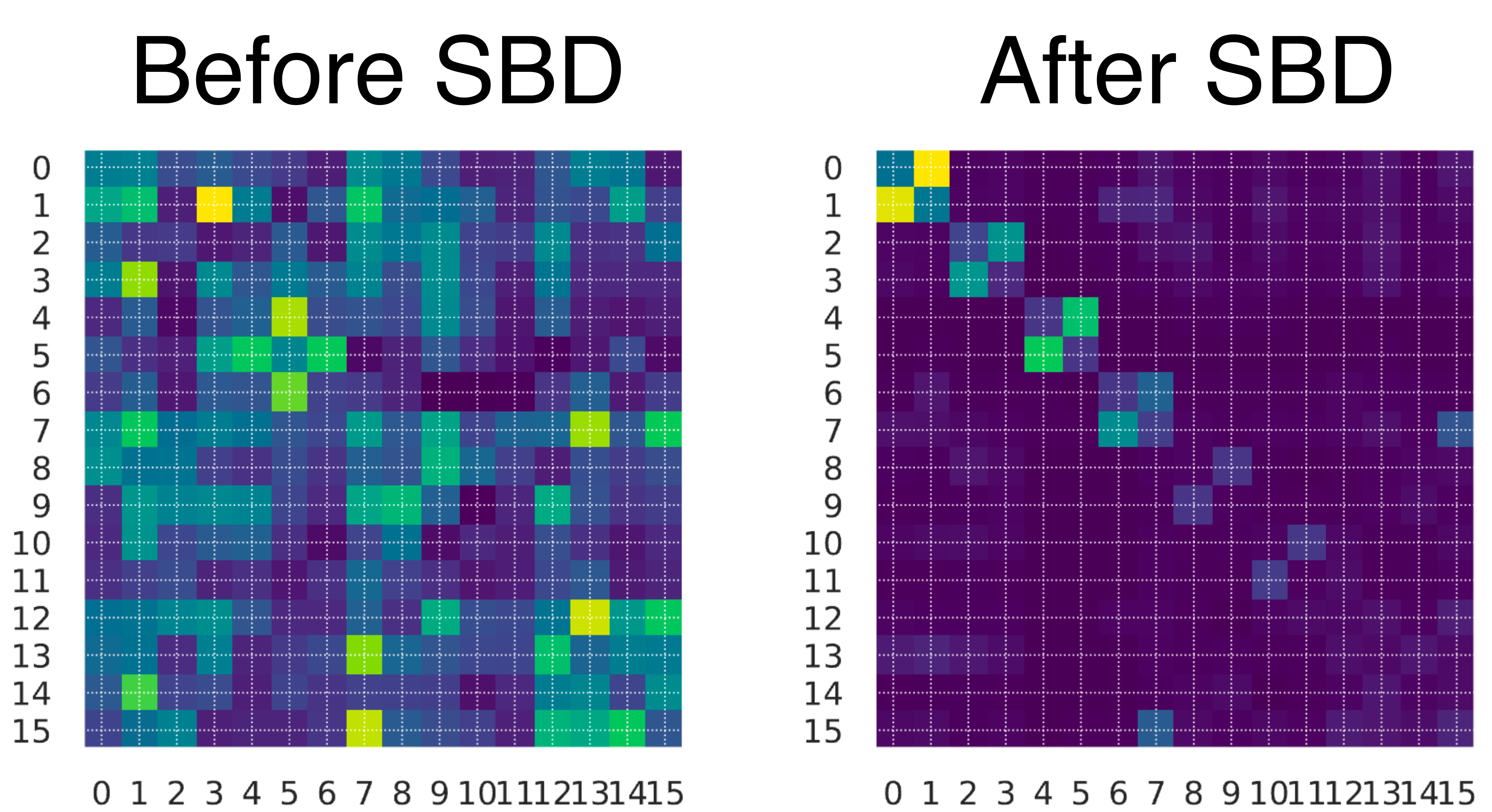}
    }
    \subcaptionbox{Sequential MNIST-bg w/ digit 4\label{fig:sbd_mnist-bg}}[.49\linewidth]{
        \includegraphics[scale=0.05]{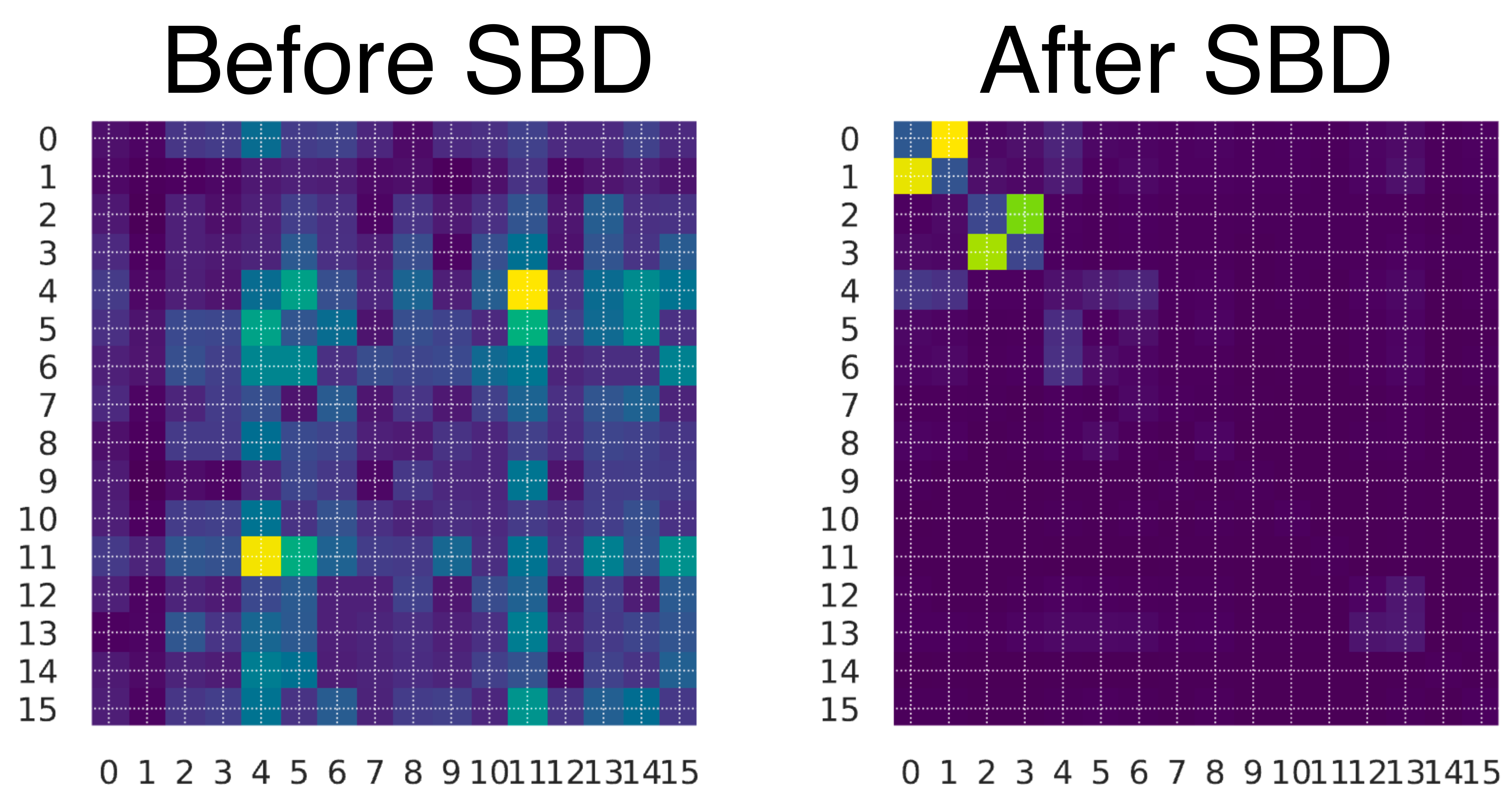}
    }
    \subcaptionbox{Sequential MNIST-bg w/ all digits \label{fig:sbd_mnist-bg-full}}[.49\linewidth]{
        \includegraphics[scale=0.05]{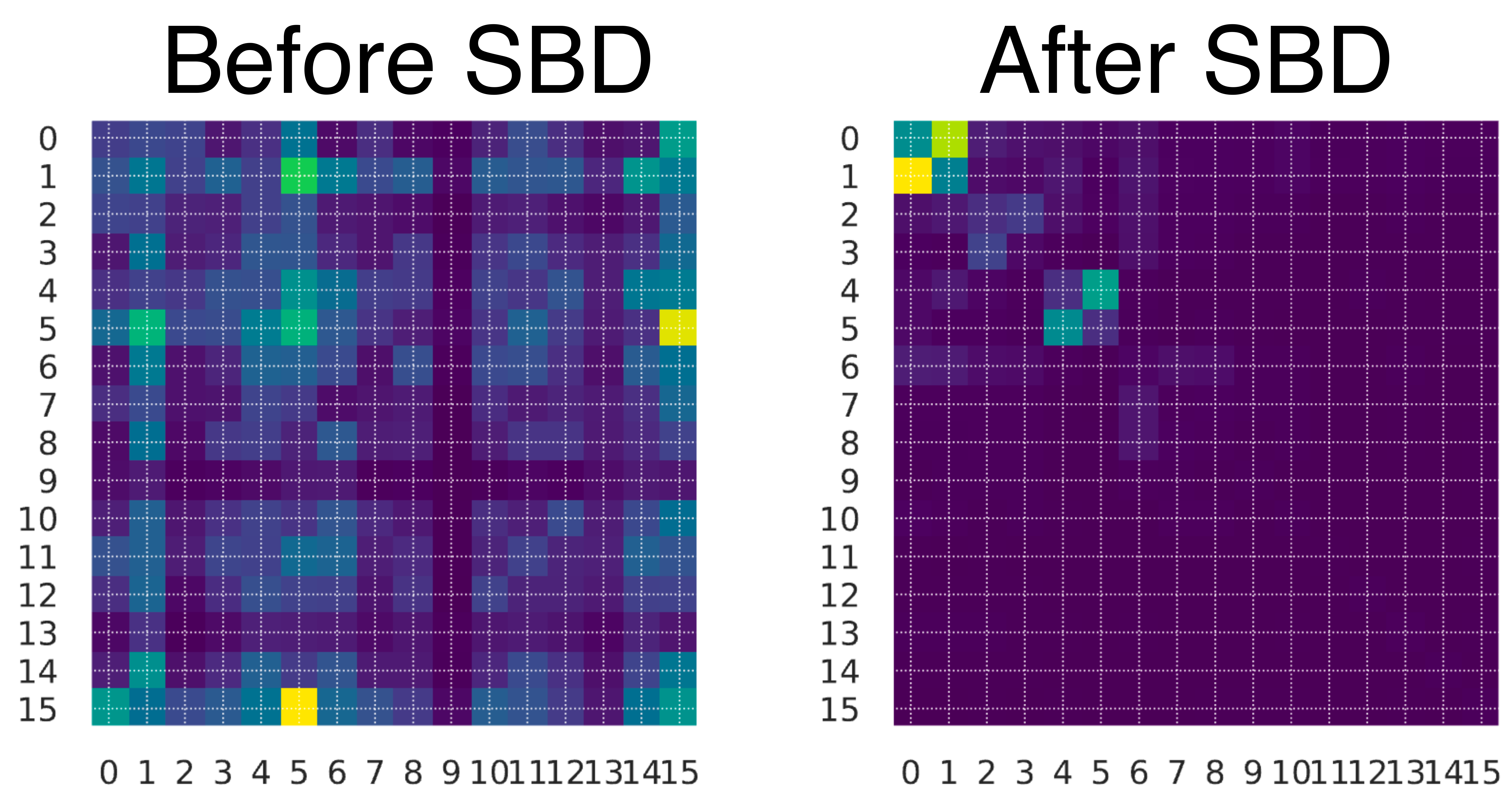}
    }
    \subcaptionbox{SmallNORB\label{fig:sbd_smallNORB}}[.49\linewidth]{
        \includegraphics[scale=0.05]{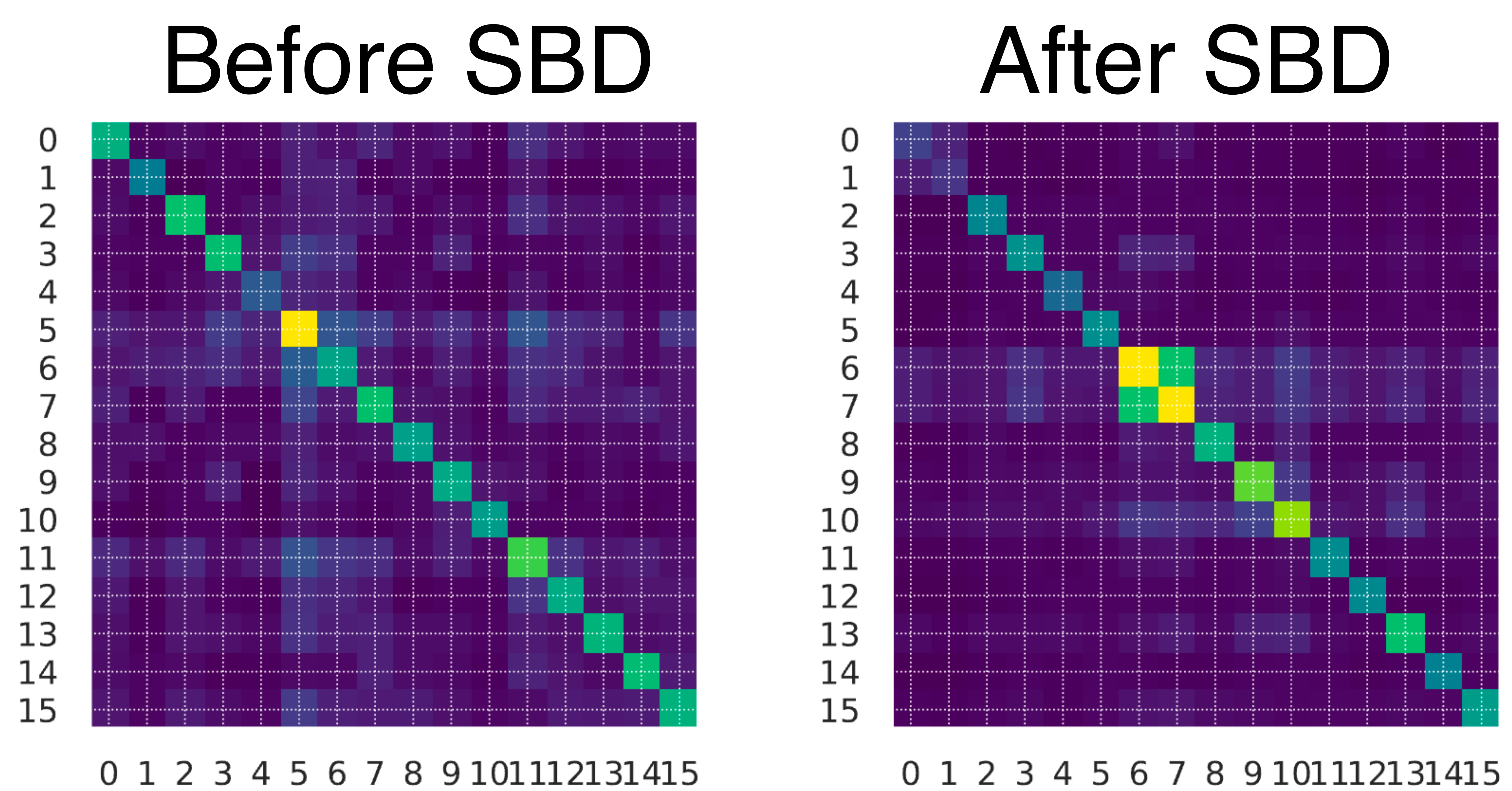}
    }
    \caption{The visualization of simultaneously block-diagonalized (SBD) matrices for Sequential MNIST/MNIST-bg and SmallNORB datasets. 
    As in Figure \ref{fig:sbd}, our visualizations correspond to  $abs(\bfM^*-I)$ and $abs(\irrepM^*-I)$ instead of the raw matrices ($\irrepM^*$ is the block-diagonalized version of  $\bfM^*$. See Section~\ref{sec:sbd} and  Section ~\ref{sec:full_blockdiag}).   
    \label{fig:sbd_full}
    }
\end{figure}

Figure~\ref{fig:sbd_full} is the visualization of the block structures revealed by the simultaneous block-diagonalization on Sequential MNIST/MNIST-bg and SmallNORB. The detail of the block-diagonalization method is provided in Section ~\ref{sec:sbd} and Section ~\ref{sec:full_blockdiag}.

To investigate what type of transformations these blocks correspond to, 
we studied the effect of using just one particular set of blocks in the block diagonalized transition matrix
(Figure~\ref{fig:steer_full}).
To create the transformation of \textit{one particular set of blocks}, we modified the block-diagonalized $M^*$ by setting all block positions other than the target blocks to identity. 
We can visually confirm that disentanglement is achieved by the partition of block positions.
See the figure captions for more details.
\newpage

\begin{figure}[H]
    \centering
    \subcaptionbox{Sequential MNIST\label{fig:steer_mnist}}[.49\linewidth]{
        \includegraphics[scale=0.045]{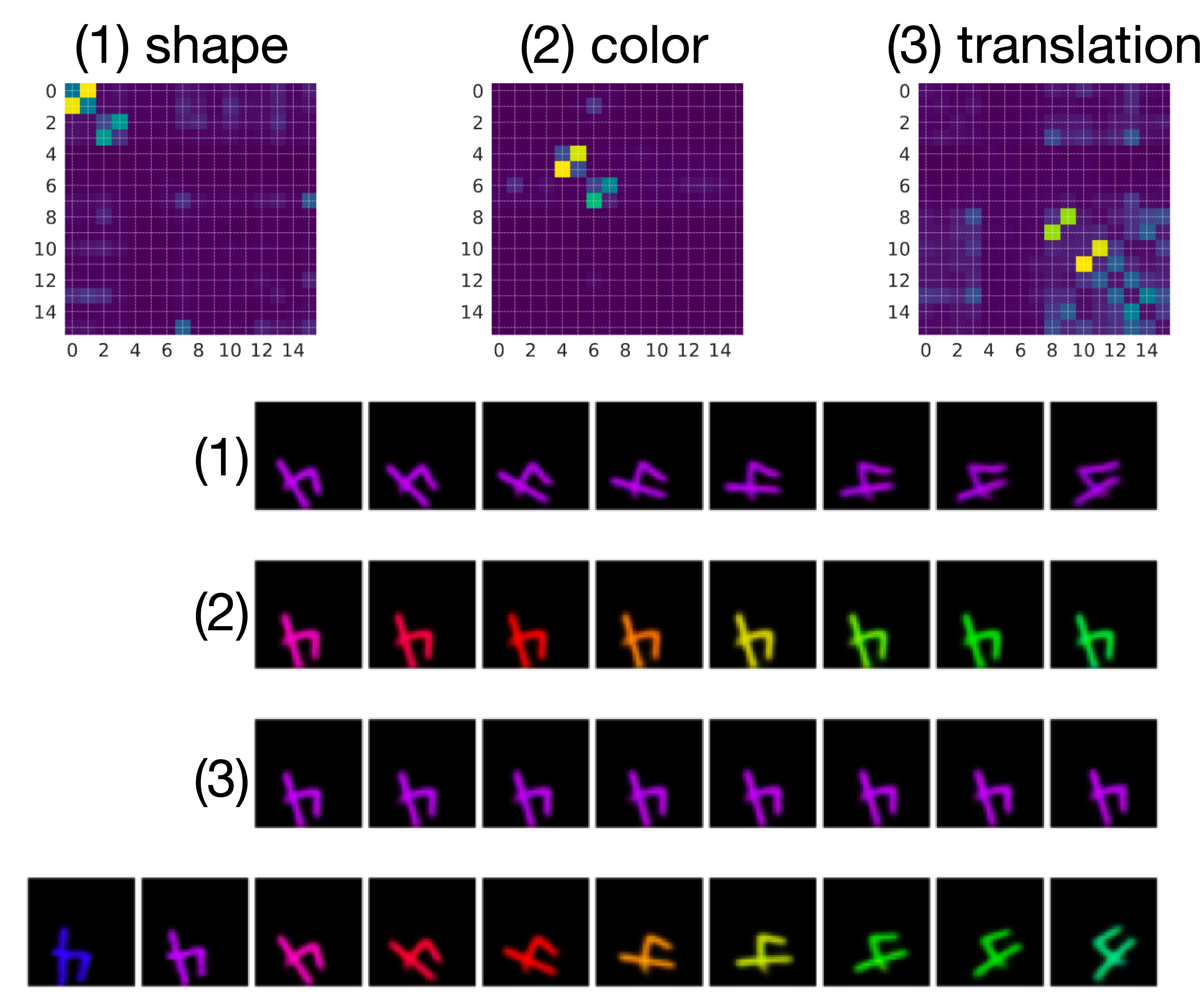}
    }
    \subcaptionbox{Sequential MNIST-bg w/ digit 4\label{fig:steer_mnist-bg}}[.49\linewidth]{
        \includegraphics[scale=0.05]{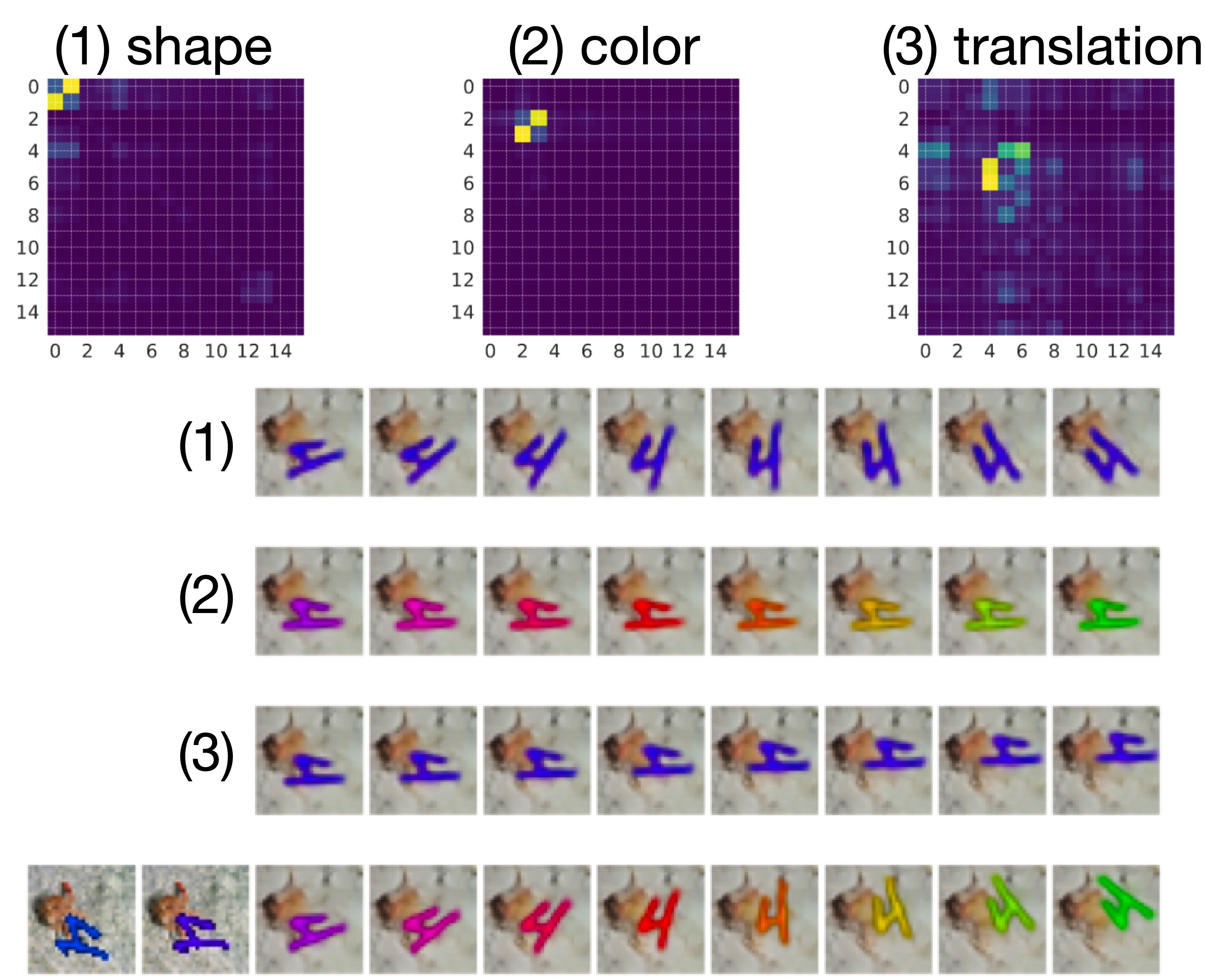}
    }
    \subcaptionbox{Sequential MNIST-bg w/ all digits \label{fig:steer_mnist-bg-full}}[.49\linewidth]{
        \includegraphics[scale=0.045]{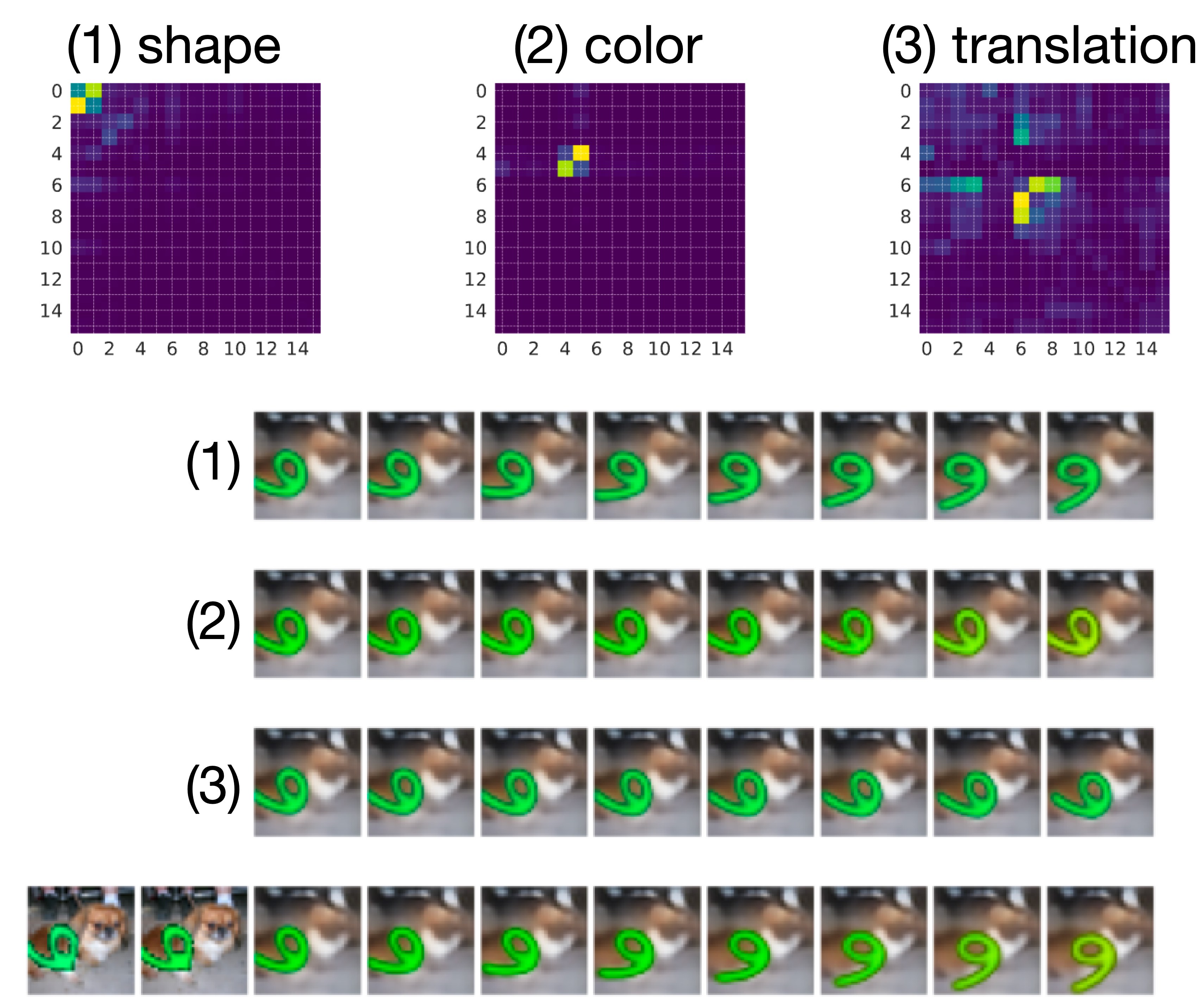}
    }
    \subcaptionbox{SmallNORB\label{fig:steer_smallNORB}}[.49\linewidth]{
        \includegraphics[scale=0.05]{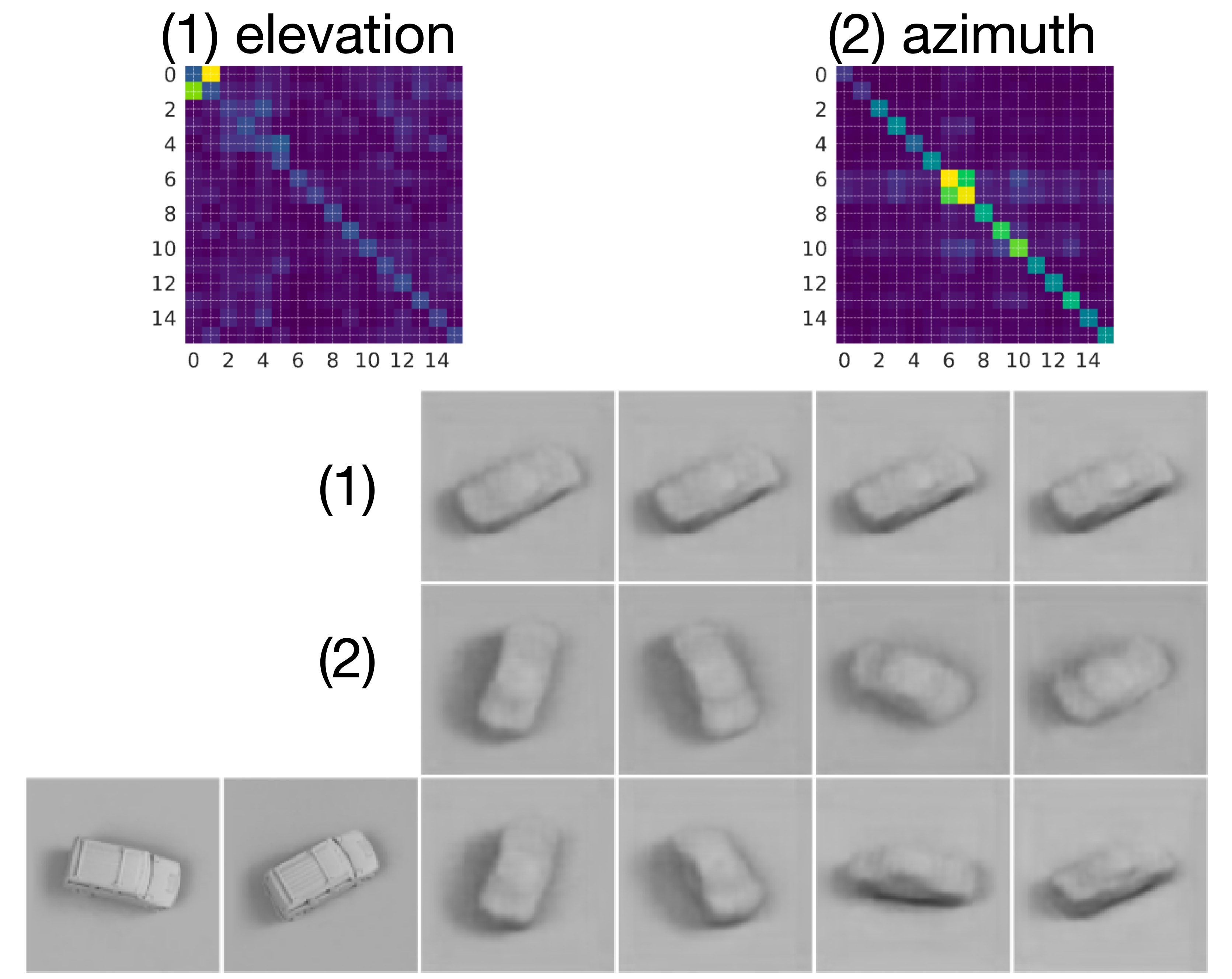}
    }
    \caption{
    Generation of disentangled  sequences.
    The bottom sequence in each frame of this figure is the ground truth. We generated each one of (1), (2) (and (3)) by applying the transformation corresponding to only one particular set of the blocks. 
To create each sequence, we first computed $M^*$ from the first two time-steps($t=1, t=2$) in the ground truth, and block-diagonalized $M^*$ to obtain $\hat M^*$. 
We then created the transformation corresponding to only \textit{one particular set of blocks} by setting all the block positions of $\hat M^*$ other than the target blocks to identity. 
We then applied the powers of the one-block-set transformation to the image at $t=2$ to generate the disentangled future sequence.
The assignment of block positions to disentangled factors was found manually by looking at the activated blocks when we altered one factor in the ground truth sequences.
See Table \ref{tab:block_pos} for the correspondence between block positions and disentangled factors.
We can visually confirm that disentanglement is achieved through block partitions.
    \label{fig:steer_full}}
\end{figure}

\begin{table}[H]
    \centering
    \begin{tabular}{cc}
    \toprule
    dataset & The \textit{factor}-\textit{block position}     correspondence \\
    \midrule

         Sequential MNIST& (1) $\{0,1,2,3\}$, (2) $\{4,5,6,7\}$, (2) $\{8,9,10,11\}$  \\
         Sequential MNIST-bg w/ digit 4& (1) $\{0,1\}$, (2) $\{2,3\}$, (3) $\{4,5,6,7,8\}$\\
        Sequential MNIST-bg w/ all digits & (1) $\{0,1,2,3\}$, (2) $\{4,5\}$, (3) $\{6,7,8\}$,\\
         \multirow{2}{*}{3DShapes} & (1) $\{0,1\}$, (2) $\{2,3,4,5\}$, (3) $\{6,7\}$, \\ 
                         & (4) $\{8,9,10,11\}$, (5) $\{12, 13,14,15\}$\\
         SmallNORB & (1) $\{0,1,2,3,4,5\}$, (2) $\{6,7,8,9,10,11,12,13\}$\\
    \bottomrule
    \end{tabular}
    \caption{
    The correspondence between block positions and disentangled factors in 
    simultaneous block-diagonalization. 
    For each $i$,  " $(i) \{a_1,  a_2, ... a_m\}$"  means that the $i$-th disentangled factor has coordinates $\{a_1,  a_2, ... a_m\}$. 
    For example, the block that is positioned at coordinates $\{4,5\}$ changes the second disentangled factor~(shape rotation) in Sequential MNIST.
    }
    \label{tab:block_pos}
\end{table}

\newpage

\subsection{Orthgonality of $M^*$ during training}
\begin{figure}[H]
    \centering
    \subcaptionbox{Sequential MNIST\label{fig:orth_mnist}}[.49\linewidth]{
        \includegraphics[scale=0.44]{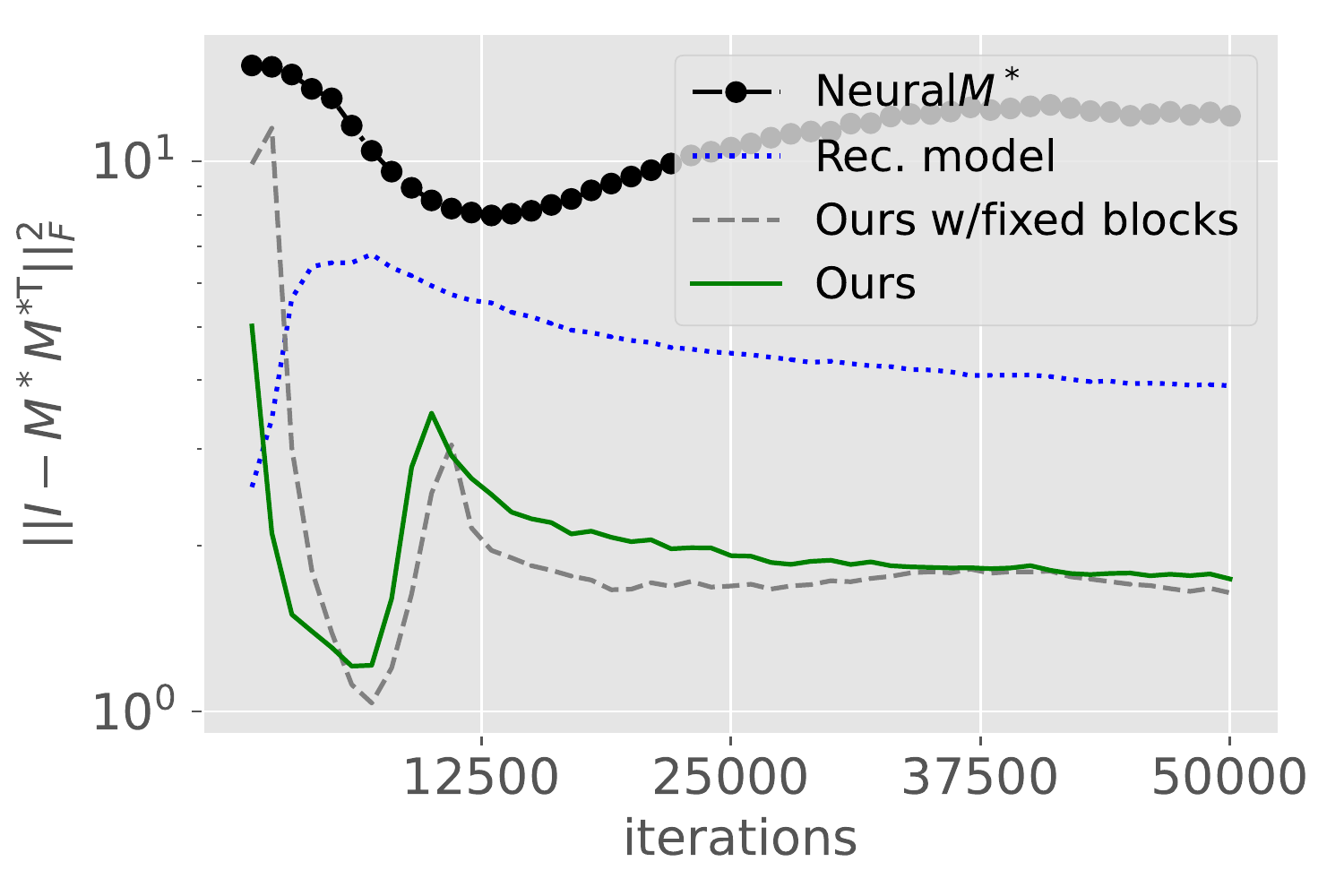}
    }
    \subcaptionbox{Sequential MNIST-bg\label{fig:orth_mnist_bg}}[.49\linewidth]{
        \includegraphics[scale=0.44]{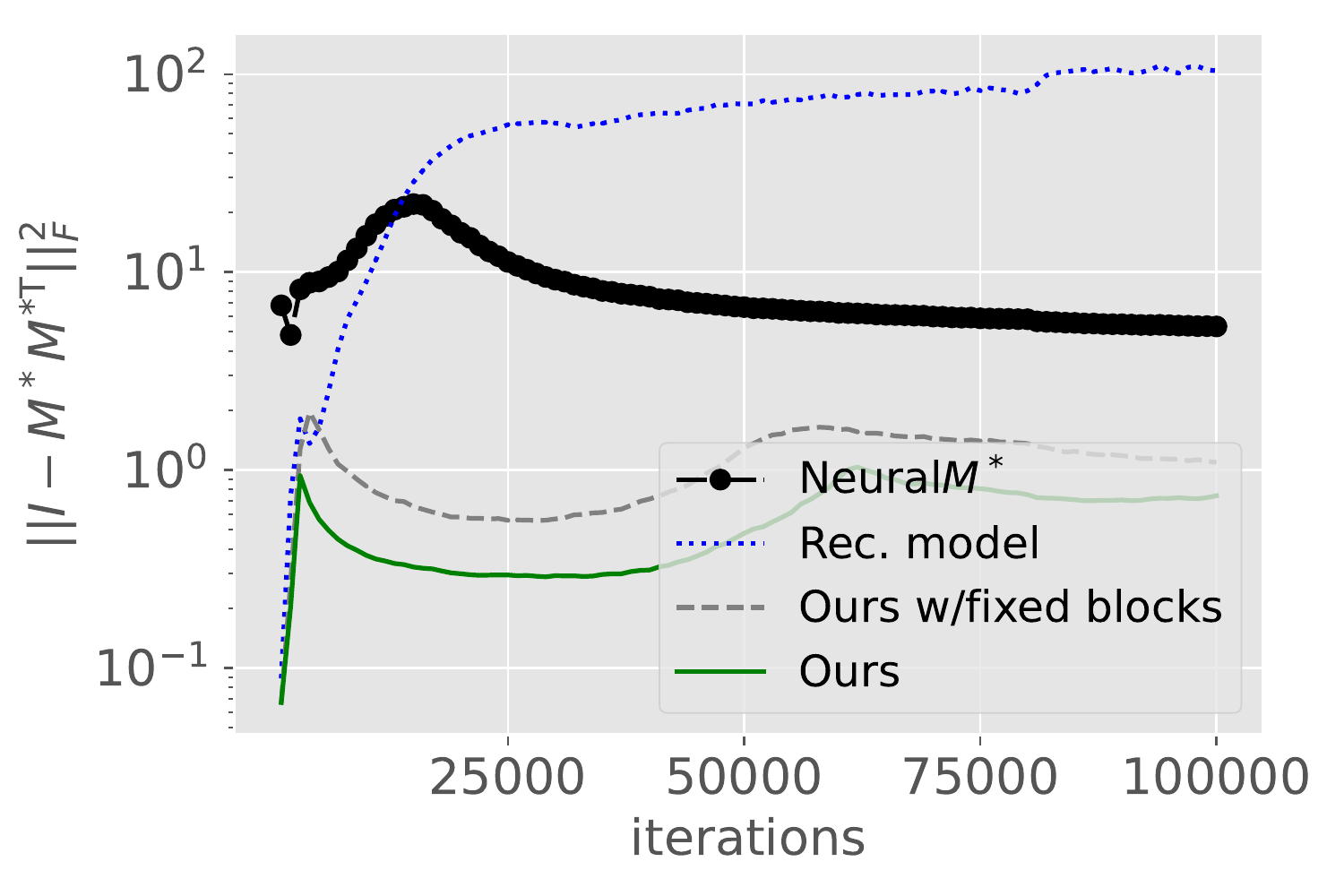}
    }
    
    \subcaptionbox{3DShapes\label{fig:orth_3dshapes}}[.49\linewidth]{
        \includegraphics[scale=0.44]{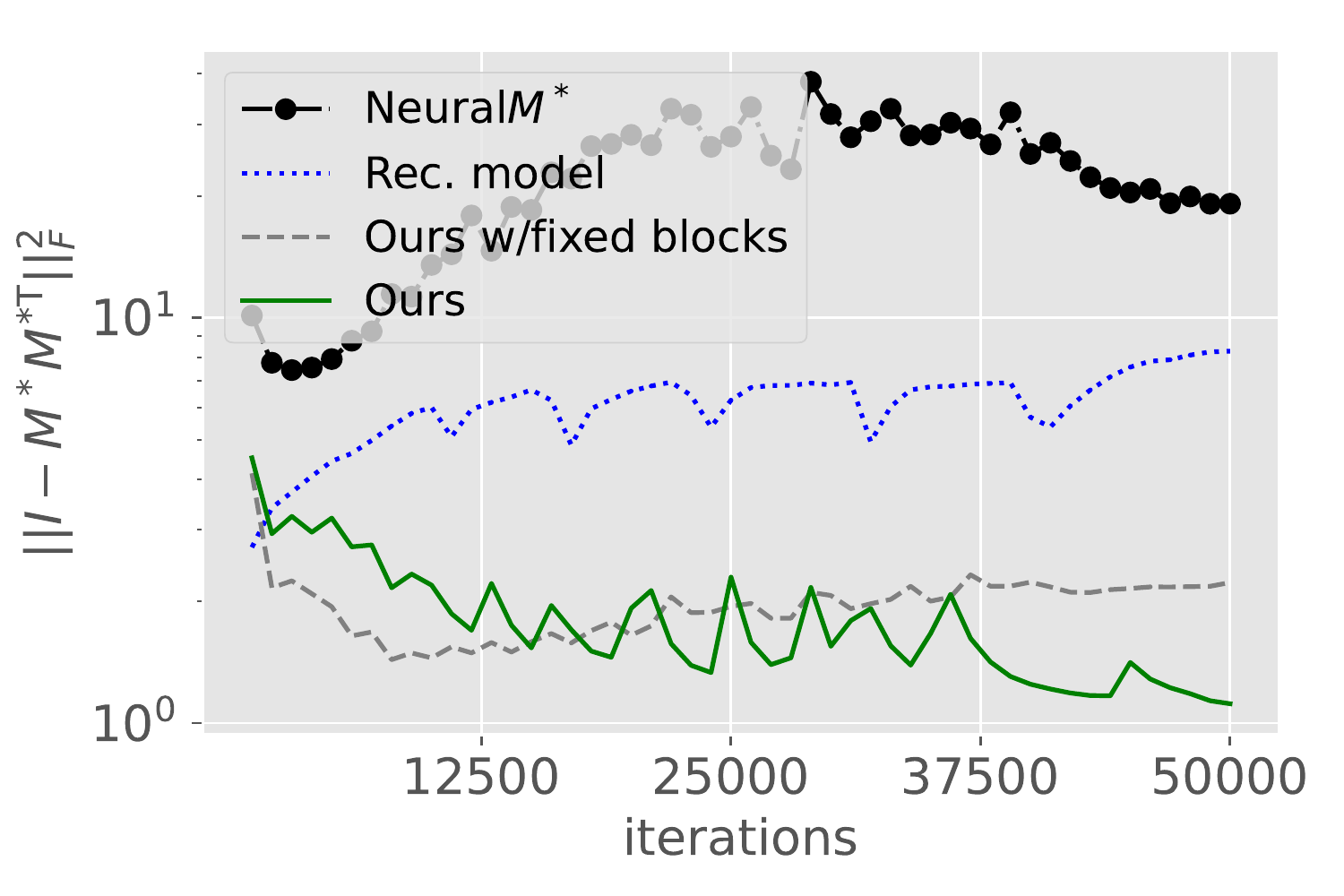}
\    }
    \subcaptionbox{SmallNORB\label{fig:orth_smallNORB}}[.49\linewidth]{
        \includegraphics[scale=0.44]{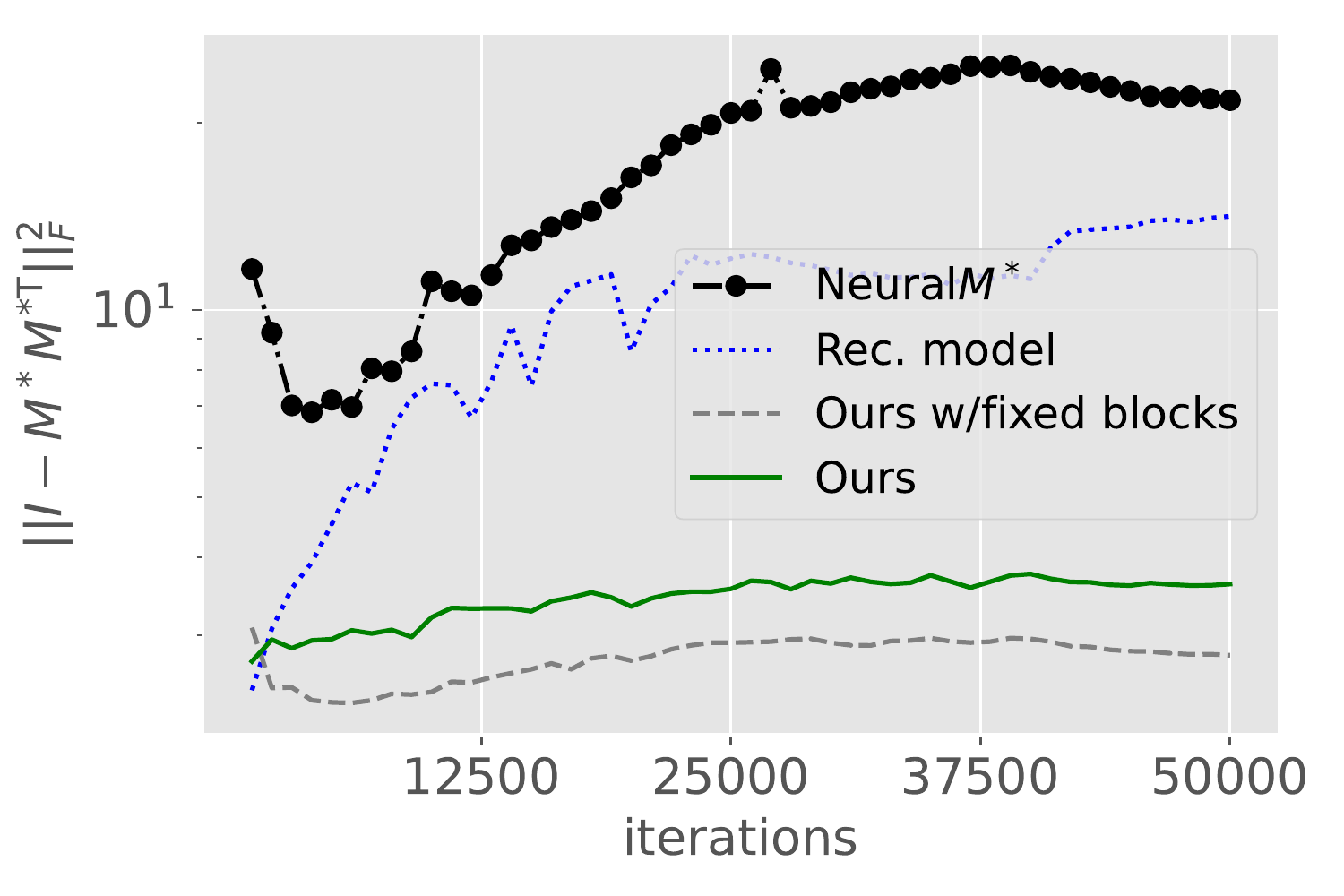}
    }
    \caption{The transition of $\|I-M^*M^{*\rm T}\|_F^2 $ during the training. We can observe that, for our method, the learned representation evolves in such a way that the estimated transition $M^*$ tends to become orthogonal.\label{fig:orgth}}
\end{figure}

\section{Generated examples}\label{sec:gen}
Figures~\ref{fig:cmpr_gen_mnist}-\ref{fig:gen_shapenet} show the seqeuences generated by our method and its variants for each dataset. The visualization follows the same protocol as in Figure~\ref{fig:gen_images_mnist}.

\newpage

\begin{figure}[H]
\begin{minipage}{0.5\textwidth}
    \centering
    \includegraphics[scale=0.175]{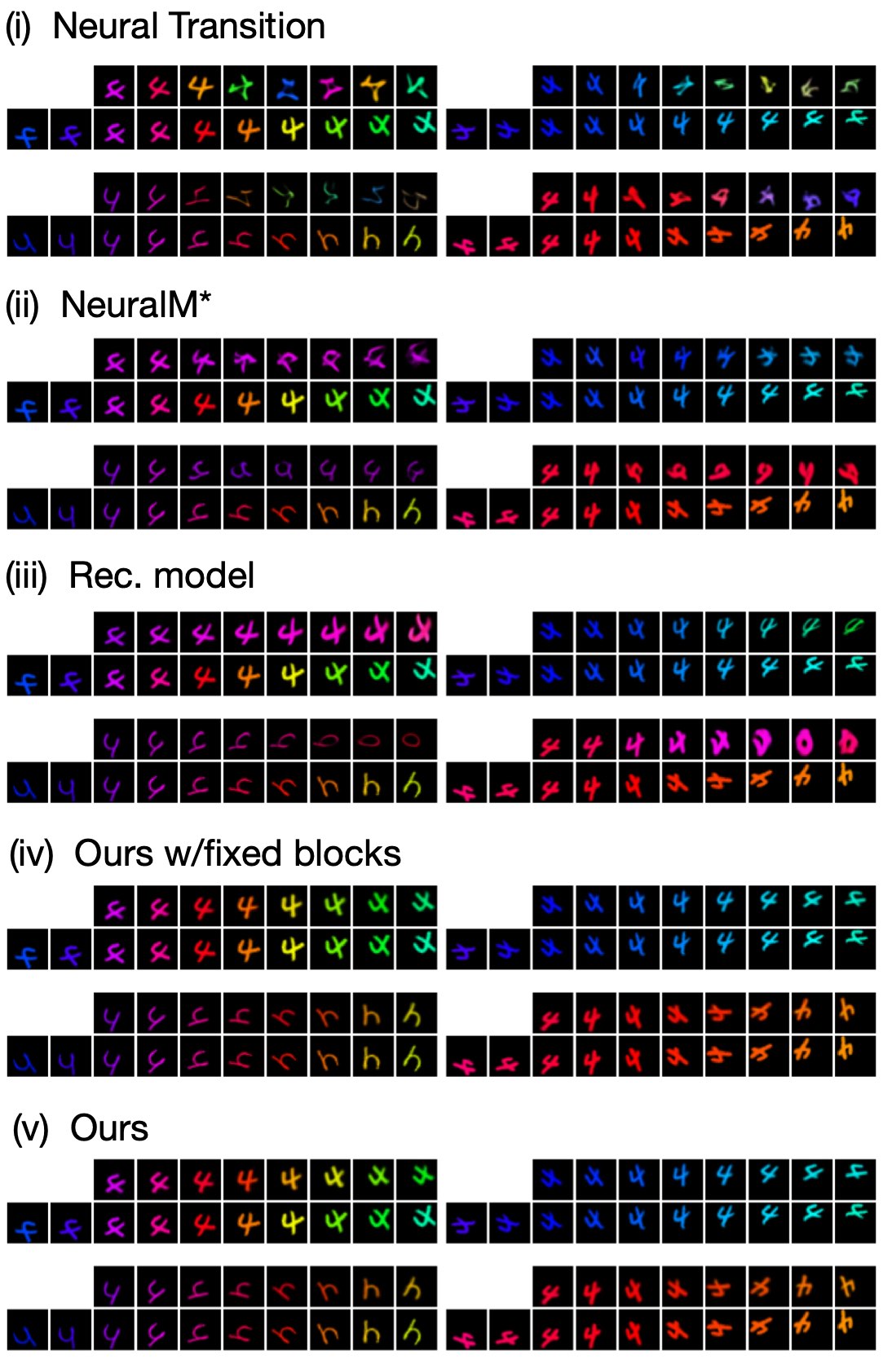}
    \caption{Sequential MNIST\label{fig:cmpr_gen_mnist}}
\end{minipage}
\begin{minipage}{0.5\textwidth}
    \centering
    \includegraphics[scale=0.175]{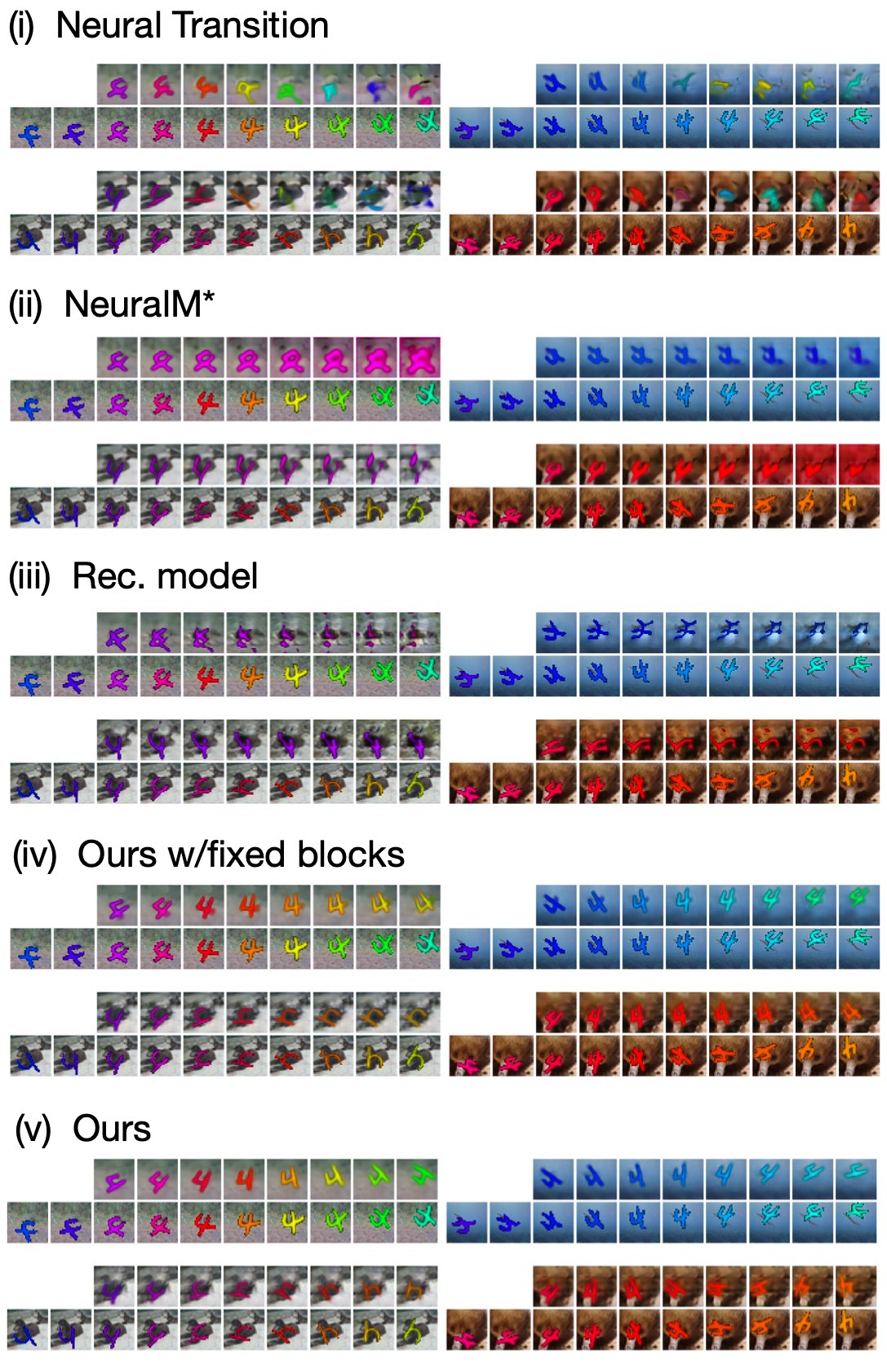}
    \caption{Seq. MNIST-bg~(w/only digit 4)\label{fig:cmpr_gen_mnist_bg_4}}
\end{minipage}  

\begin{minipage}{0.5\textwidth}
    \centering
    \includegraphics[scale=0.175]{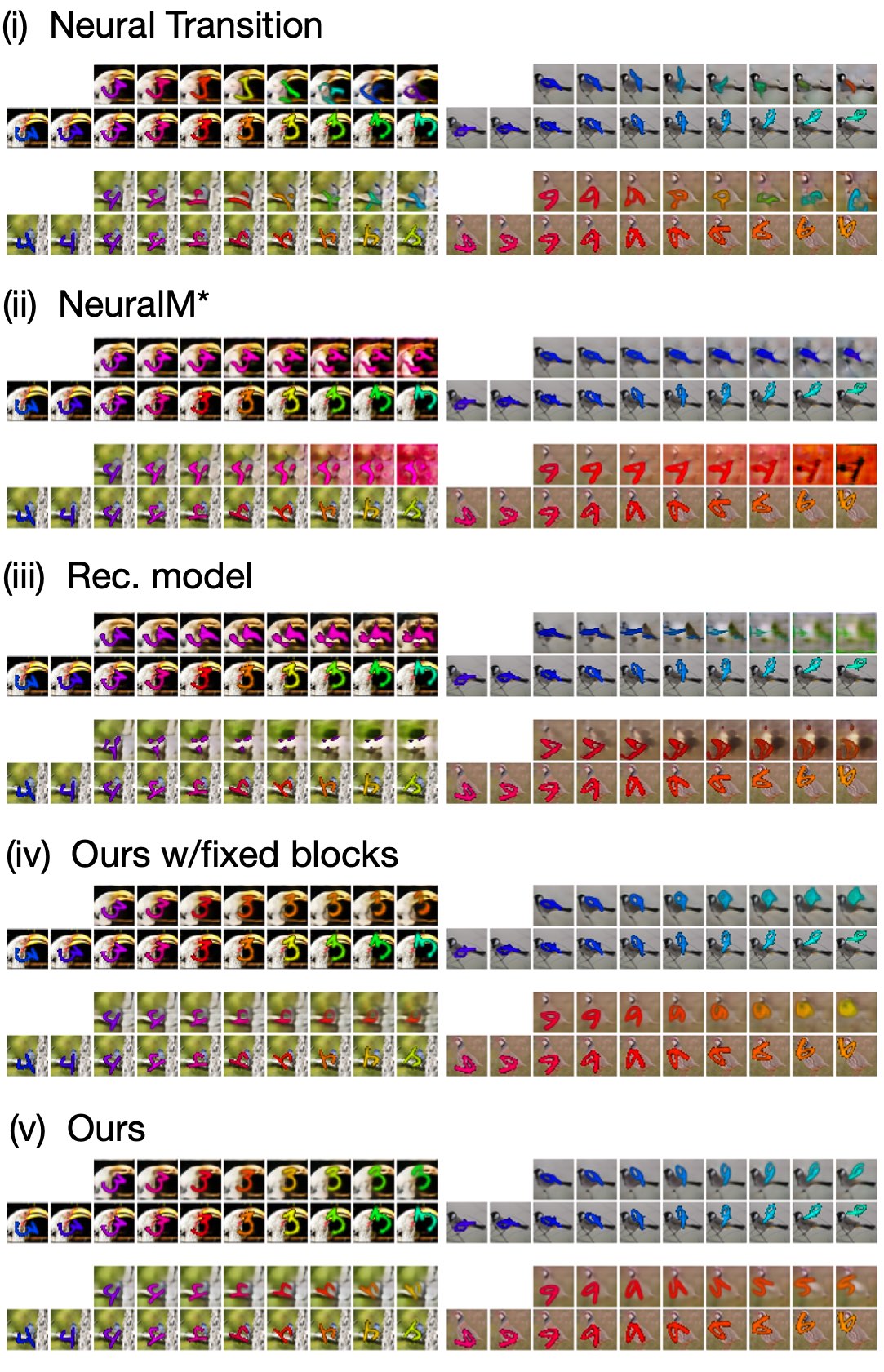}
    \caption{Seq. MNIST-bg~(w/all digits)\label{fig:cmpr_gen_mnist_bg_full}}
\end{minipage}
\begin{minipage}{0.5\textwidth}
    \centering
    \includegraphics[scale=0.175]{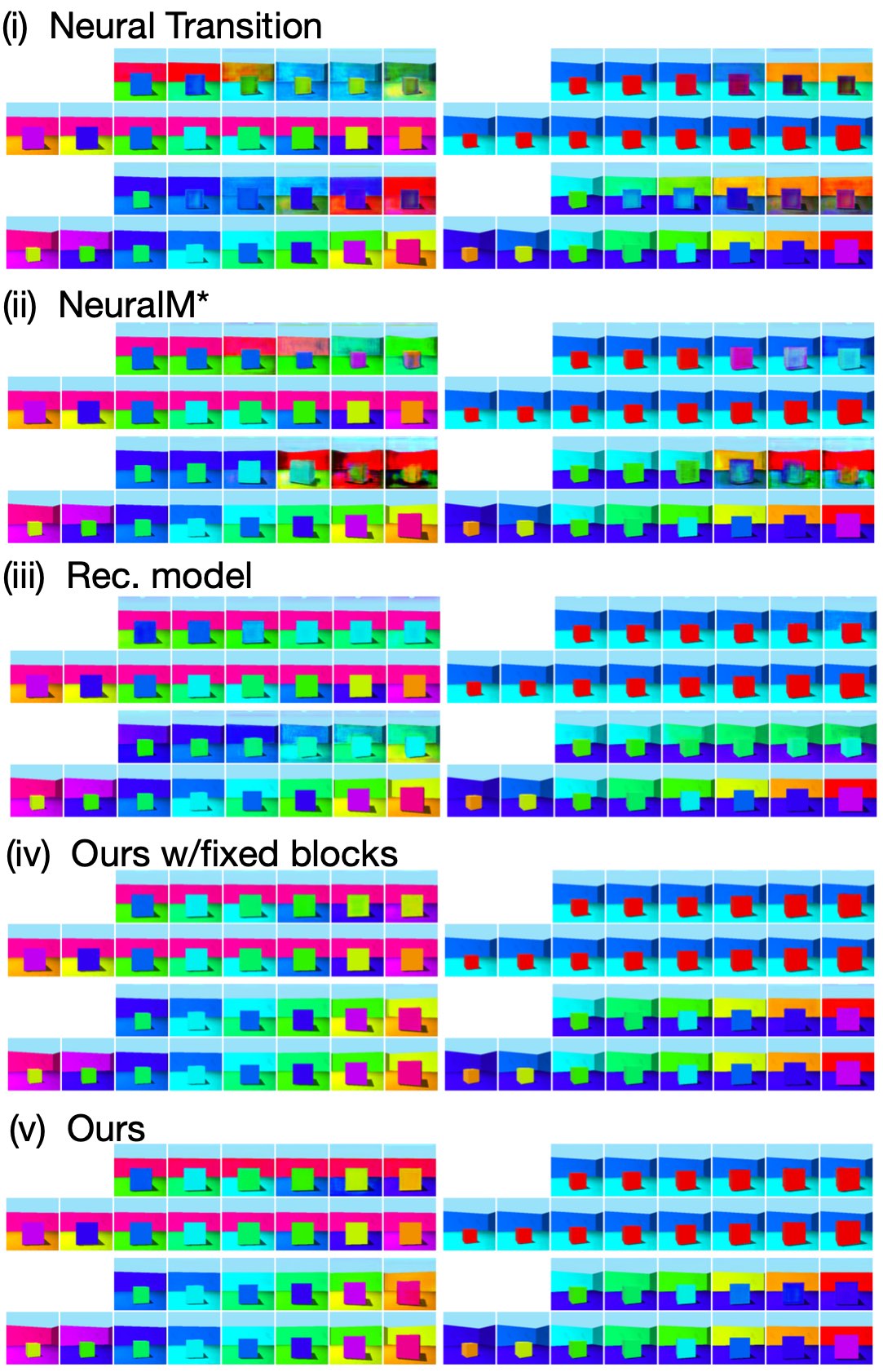}
    \caption{3DShapes\label{fig:cmpr_gen_3dshapes}}
\end{minipage}
\end{figure}

\begin{figure}[H]

\begin{minipage}{0.45\textwidth}
    \includegraphics[scale=0.175]{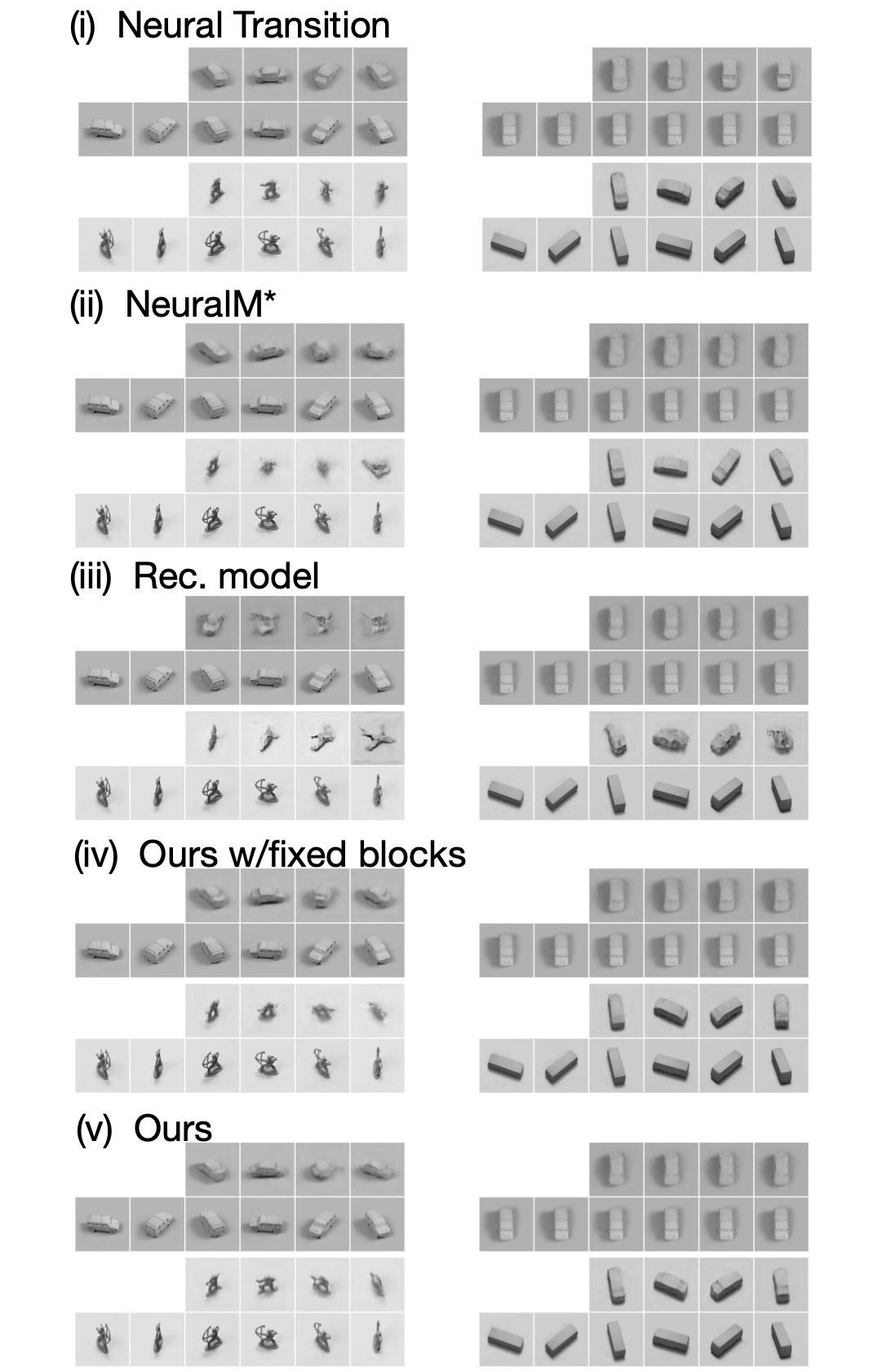}
    \caption{SmallNORB \label{fig:cmpr_gen_small_norb}}
\end{minipage}
\begin{minipage}{0.55\textwidth}
    \centering
    \includegraphics[scale=0.13]{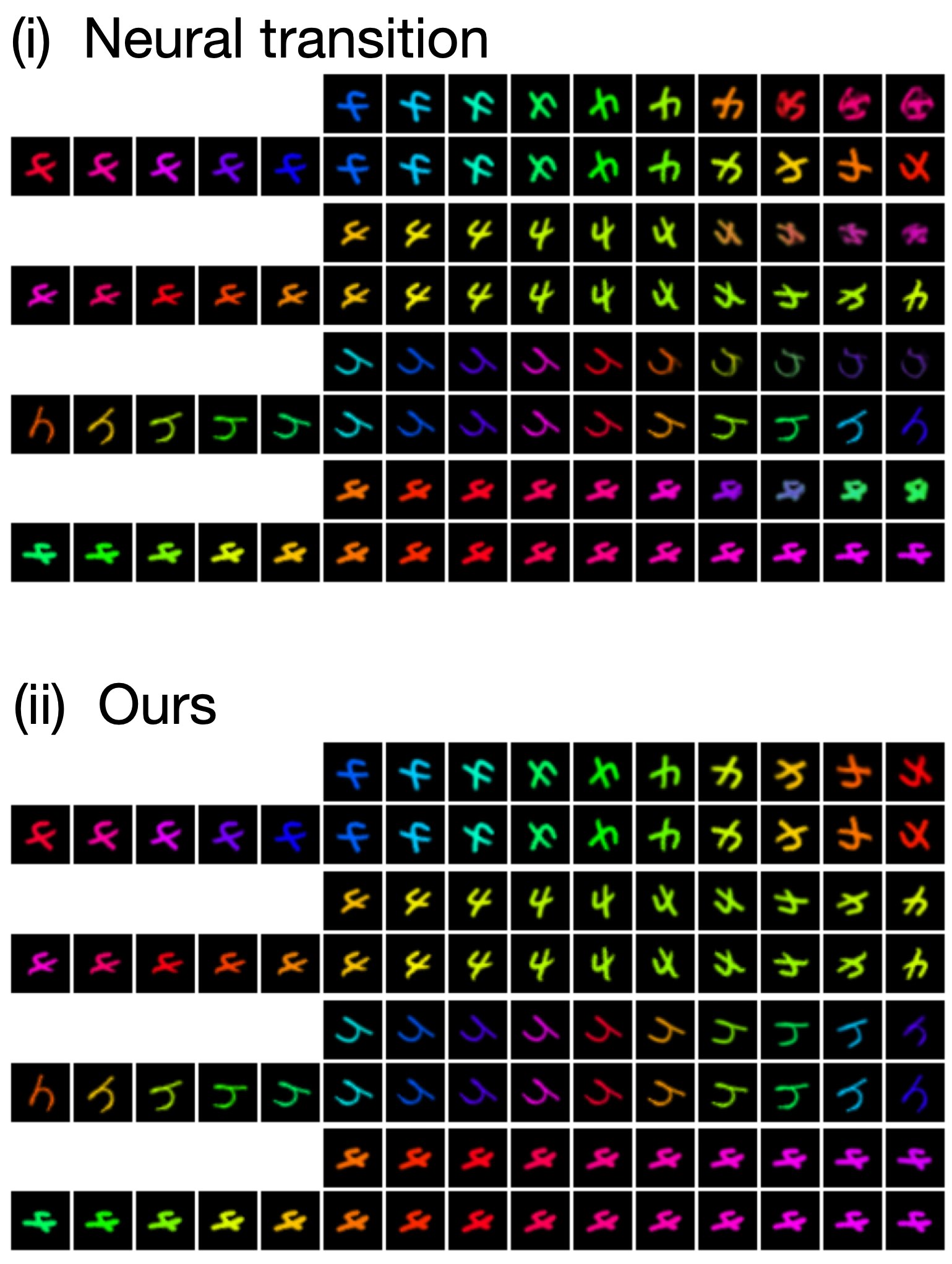}
    \caption{Accelerated Sequential MNIST. Both models were trained with $\Tc=5$ and $\Tp=5$.
    For the training procedure on this experiment, please see Section \ref{sec:accl}.
    \label{fig:cmpr_gen_mnist_accl}
    }
\end{minipage}  
\end{figure}

\begin{figure}[H]
    \centering
    \includegraphics[scale=0.15]{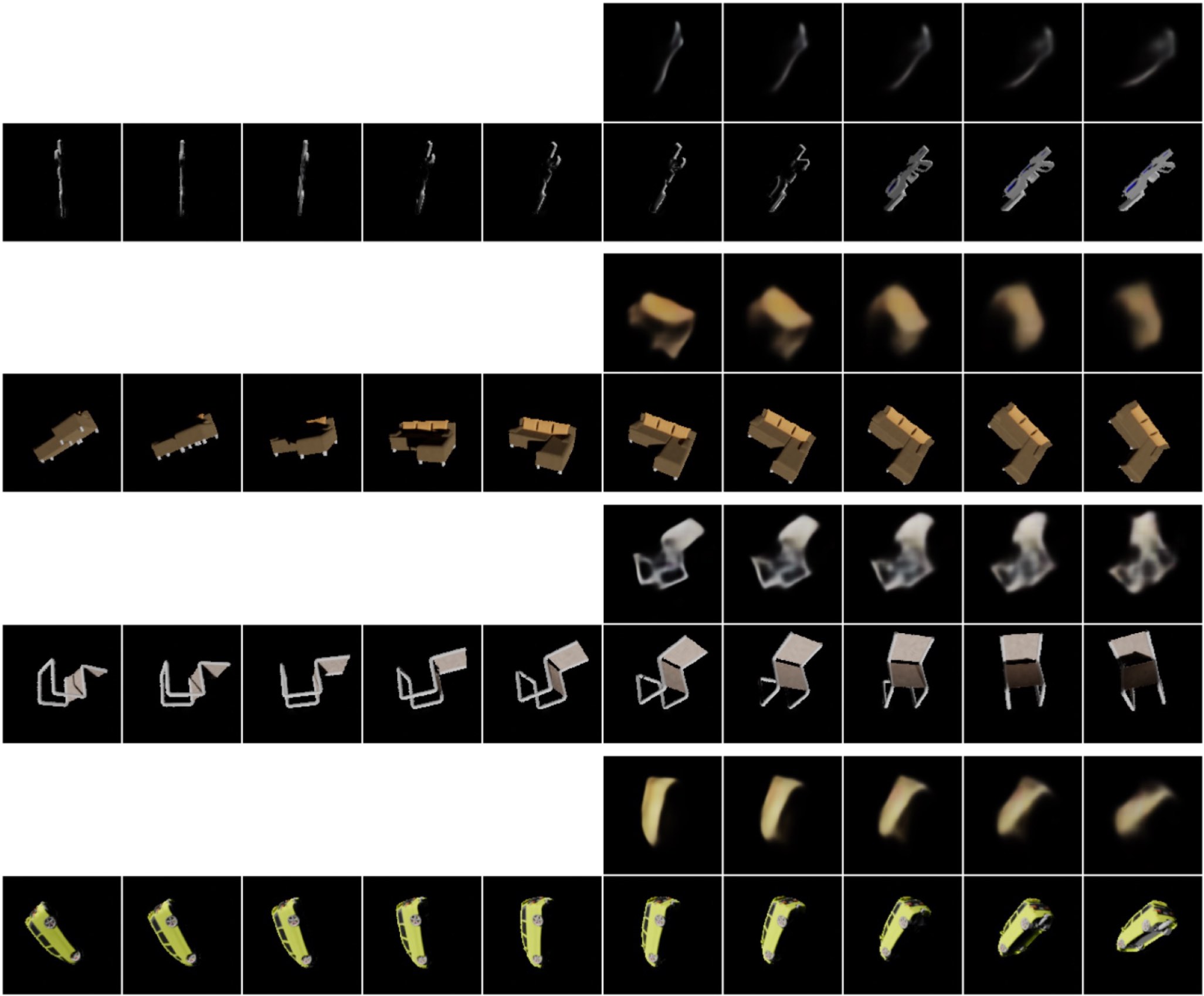}
    \caption{
    Sequential ShapeNet generated by the proposed method. The model was trained with $\Tc=5$ and $\Tp=5$. 
    Our method cannot make good predictions on this dataset.
    Note that, unlike other datasets we studied on this paper, the transition in Sequential ShapeNet is not necessarily invertible,
    because some parts of a $3D$ object are often not visible in the 2D rendering.
    }
    \label{fig:gen_shapenet}
\end{figure}

\begin{figure}[H]
    \centering
    \includegraphics[scale=0.25]{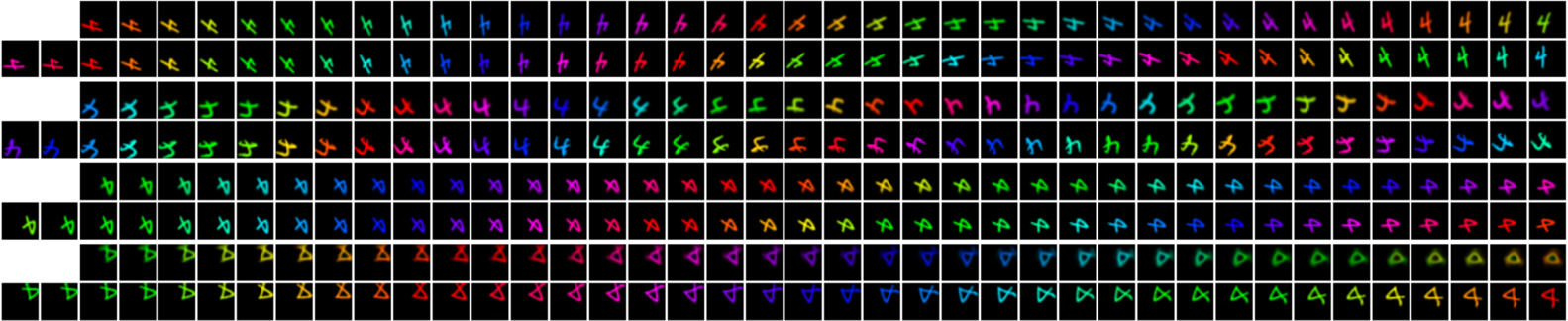}
    \caption{Extrapolation results for longer future horizon~($\tp$ = 1,...,38). The visualization follows the same protocol as in  Figure~\ref{fig:gen_mnist}. }
    \label{fig:gen_mnist_40}
\end{figure}

\section{Algorithm}\label{sec:algo}

We provide the algorithmic description of our method. 
Definitions of all symbols in the table are the same as in the main sections. 
\begin{algorithm}
\caption{Calculate the loss over $\bPhi$ and $\bPsi$. \label{alg:overall}}
Input: Given an encoder $\bPhi$, a decoder $\bPsi$ and a sequence of observations $\bs=[s_1,\dots,s_{\Tc},s_{\Tc+1},\dots, s_{T}]$.
\begin{tight_enumerate}
\item Encode the observations into latent variables $H_t=\Phi(s_t)$ for $s_1,...,s_{T_c}$. 
\item Estimate the transition matrix by solving linear problem: $M^* = \bfH_{+1}\bfH_{+0}^{\dagger}$ where $\bfH_{+0}$ and $\bfH_{+1}$ are the horizontal concatenation $[\bfH_1;\bfH_2;...,;\bfH_{\Tc-1}]$ and $[\bfH_2;\bfH_3;,...,;\bfH_{\Tc}]$, respectively. 
\item Predict the future sequence by : $\tilde{s}_{t} = \Psi((\bfM^*)^{t-\Tc} H_{\Tc})$ for $t = \Tc+1,...,T$.
\item Calculate the loss for the sequence $\bs$: $ \sum_{t=\Tc+1}^{T} \|\tilde{s}_{t} - s_t\|^2_2$
\end{tight_enumerate}
\end{algorithm}

\section{Experimental settings}\label{sec:exp_settings}

\subsection{Ablation studies}\label{sec:detail_ablation}

As ablations, we tested several variants of our method:  \textbf{fixed 2x2 blocks}, \textbf{Neural$M^*$},   \textbf{Reconstruction model} (abbreviated as Rec. Model), and \textbf{Neural transition}. 
We describe each one of them below.
\begin{itemize}
\item {\textit{\bf Fixed 2x2 blocks}}:
For this model, we separated the latent tensor $\bPhi(s)\in \mathbb{R}^{16\times 256}$ into 8 subtensors $\{\bPhi^{(k)}(s)\in \mathbb{R}^{2\times 256}\}_{k=1}^8$ and calculated the pseudo inverse for each $k$ to compute the transition in each $\mathbb{R}^{2\times m}$ dimensional space.
Essentially, this variant of our proposed method computes $M^*$ as a direct sum of eight $2 \times 2$  matrices. 
In the pioneer work of \cite{cohen2014learning} that endeavors to learn the symmetry in a linear system using the representation theory of commutative algebra, the authors hard-code the irreducible representations/block matrices in their model.  Our study is distinctive from many applications of representation theory and symmetry learning in that we uncover the symmetry underlying the dataset not by introducing any explicit structure, but by simply seeking to improve the prediction performance. We therefore wanted to experiment how the introduction of the hard-coded symmetry like the one in \cite{cohen2014learning} would affect the prediction performance.

\item {\textit{\bf Neural$M^*$}}:
Our method is ``meta'' in that we distinguish the internal training of $M^*$ for each sequence from the external training of the encoder $\Phi$. Put in another way, the internal optimization process of $M^*$ itself is the function of the encoder. To measure how important it is to train the encoder with such a “meta” approach, we evaluated the performance of Neural$M^*$ approach.  
To reiterate, Neural$M^*$ uses a neural network $M^*_\theta$ that directly outputs $M^*$ on the conditional sequence, and train the encoder and the decoder via 
$$  \sum_{t= T_c+1} ^{T_c+T_p} \|  \Psi(M^*_\theta(\bsc)^{t-\Tc} \Phi(s_{T_c})) - s_t \|_2^2, $$ thereby testing the training framework that is similar to our method ``minus'' the ``meta'' component.  

\item {\textit{\bf Reconstruction model (Rec. model)}}:
In our default algorithm, we train our encoder and decoder with the prediction loss $\mathcal{L}^p$ in eq.\eqref{eq:pred_loss} over the future horizon of length $\Tp-\Tc$. We therefore wanted to verify what would happen to the learned representation when we train the model with the reconstruction loss $\mathcal{L}^r$ in~eq.\eqref{eq:rec_loss} in which the model predicts the observations contained in the conditional sequence. 
Specifically, we trained $\Phi$ and $\Psi$ based on $\mathcal{L}^r$ in~\eqref{eq:rec_loss} with $T=\Tc=3$. 

\item {\textit{\bf Neural Transition}}:
One important inductive bias that we introduce in our model is that we assume the latent transition to be linear. We therefore wanted to test what happens to the results of our experiment if we drop this inductive bias. 
For Neural transition, we trained a network with 1x1 1D-convolutions that inputs $\Phi(s)$ in the past to produce the latent tensor in the next time step; for instance, $\tilde{H}_{t+1} = {\rm 1DCNN}(\Phi(s_t), \Phi(s_{t-1}))$ when $\Tc=2$. This model can be seen as a simplified version of \cite{oord2016wavenet}. The 1DCNN was applied along the multiplicity dimension ($m$).
\end{itemize}

In testing all of these variants, we used the same pair of encoder and decoder architecture as the proposed method.

\subsection{Training details}
\begin{figure}[H]
    \subcaptionbox{Transformations \label{fig:transformations}}[.3\linewidth]{\includegraphics[scale=0.4]{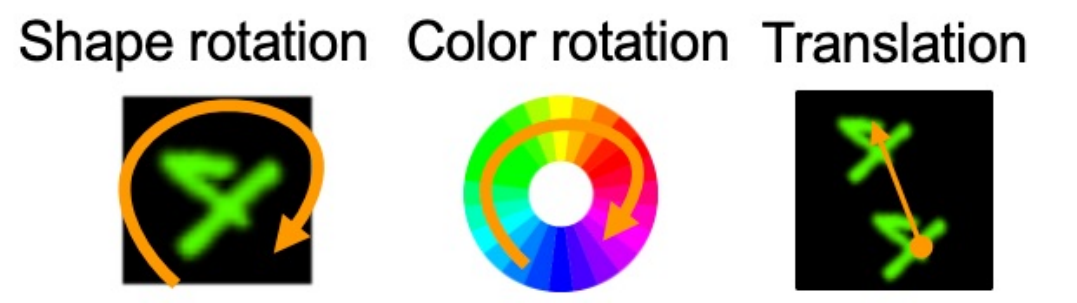}}%
    \subcaptionbox{An example sequence\label{fig:example_seq}}[.7
    \linewidth]{\includegraphics[scale=0.3]{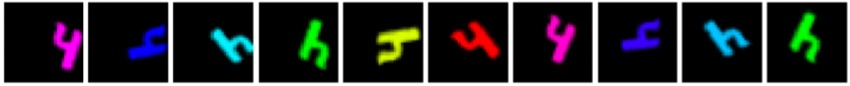}}%
    \caption{\textit{Sequential MNIST} dataset. 
    The transition in each sequence was produced by combining 
    three families of actions:  shape rotation, color rotation and translation.
    }
    \label{fig:seq_mnist}
\end{figure}

For the model optimization in every experiment, we used ADAM~\cite{kingma2014adam}.
The number of iterations for the parameter updates were $50,000$ on Sequential MNIST, 3Dshapes, and SmallNORB.
The number of iterations was $100,000$ on Sequential MNIST-bg (with only digit 4 class) and accelerated Sequential MNIST.
For MNIST-bg (with all digits) and Sequential ShapeNet, the number of iterations was $200,000$.
 We set the initial learning rate of ADAM optimizer to $0.0003$ and decayed it to $0.0001$ after a certain number of iterations.
 For Sequential MNIST, 3Dshapes, and SmallNORB, we began the decay at $40,000$-th iteration.
 For Sequential MNIST-bg (with digit 4 only) and accelerated Sequential MNIST, we began the decay at $80,000$-th iteration.
 For Sequential MNIST-bg (with all digits) and Sequential ShapeNet, we began the decay at $160,000$-th iteration.
 

The batchsize was set to 32 for all experiments.
We conducted all experiments on NVIDIA A100 GPUs. Training our proposed model takes approximately one hour per $50,000$ iterations on a single A100 GPU. Total amount of time to reproduce the full results in our experiments is approximately 12 days on a single A100 GPU.

We found that the choice of the latent dimension ($a \times m$)  does not make significant difference on the results as long as they are not too small (For example, if $G$ is a torus group consisting of $n$ commuting axis,  $a$ must be no less than $2n$ when all the observations are real-valued; otherwise the model will underfit the datasets. 
Also, we chose $m$ to be larger than $a$ so that $\Phi(x)$ becomes full rank almost surely.
This allows us to solve $M^*$ from $M^* \Phi(s_1) = \Phi(s_2)$ (We can compute  $M^*$  with $T_c = 2$.)  Choosing $m>a$ also plays a role in the theory (Section \ref{sec:proof}).

As for the Neural$M^*$ and Neural transition models, we also optimized the invertibility loss: $\sum_{t=1}^{\Tc} \|\Psi(\Phi(s_t)) - s_t\|_2^2$ in addition to $\mathcal{L}^p$. Adding this loss to the original objective yielded better results for these models in terms of prediction error and equivariance error for all experiments. 




\paragraph{SimCLR and CPC settings in the downstream task experiments in Figure~\ref{fig:actpred} and~\ref{fig:cls}}
To evaluate our method as a representation learning method, we compared our method against SimCLR~\cite{chen2020simple} and CPC~\cite{van2018representation, henaff2020data}.
For SimCLR, we treated any pair of observations in the same sequence as a positive pair, and any pair of observations in different sequences as a negative pair. 
We used the same encoder architecture as in our baseline experiment for both SimCLR and CPC.
For the projection head of SimCLR, however, we used the same architecture as in the original paper. 
For the auto regressive network of latent representation in CPC, we used the same architecture as in Neural transition (see Section \ref{sec:experiments}). 
The latent dimension was set to 512 for both models. 
We experimented with larger and smaller dimensions as well, but however large the difference, altering the dimensions did not result in significant improvements in terms of the representation quality evaluated in the experimental sections. 
The temperature parameters for the logit output were searched in the range of [1e-3, 1e-2, 1e-1, 1.0, 10.0]. 
Because SimCLR is not built for the sequential dataset, it is not expected to perform too well in terms of regression performance. 
We however evaluated these models as minimum performance baselines.

\subsection{Additional details of datasets}

Our training-test split was the same as the split in the original dataset. 
Therefore the train-test split of Sequential MNIST/MNIST-bg was the same as that of MNIST, and the split of the SmallNORB dataset we used was the same as that of the original SmallNORB.
Meanwhile, the 3DShapes dataset does not have train-test split, so we conducted the training and the test evaluation on the same dataset for the sequential 3DShapes experiment. 
We also used only cubic shape examples on the 3DShapes experiments. 
For Sequential ShapeNet,  90\% of objects in the original ShapeNetCore assets were used for the training and the rest were used for the evaluation.
The input size of each example in a given sequence was $3\times32\times32$ for Sequential MNIST/MNIST-bg, $3\times64\times 64$ for 3DShapes, $1\times 96\times 96$ for SmallNORB, and $3\times 128\times 128$ for Sequential ShapeNet.

To generate Sequential ShapeNet, we used Kubric~\cite{greff2021kubric} to render the objects in ShapeNet~\cite{shapenet2015} datasets. 
For each sequence, we sampled one object from ShapeNetCore assets, and used 3D rotation to define the transition. 
The angle of 3D rotation in each axis(xyz) was sampled from the uniform distribution over the interval $[0, \pi/4)$.

\subsection{Network architecture}

We used ResNet-based encoder and decoder\cite{he2016deep}. We used ReLU function~\cite{nair2010rectified, glorot2011deep, maas2013rectifier} for each activation function and group normalization~\cite{wu2018group} for the normalization layer. We used weight standarization~\cite{weightstandardization} for all of filters in each convolutional network. 
Also, we used trainable positional embedding in each block of the decoder,  which was initialized to the 2D version of sinusoidal positional embeddings~\cite{wang2021translating}.
We provide the details of the architecture in Table~\ref{tab:arch} and Figure~\ref{fig:resblock}.
 
For the Neural$M^*$ method, we used the same model in the table~\ref{tab:enc_arch} except the input channel of the network was set to $6$ (and $2$ for SmallNORB) because this method uses a pair of images ($s_1$, $s_2$) as an input.

For the Neural transition model, we used a network with 1x1 1D convolutions to map $[H_t, \dots, H_{t+t'}]$ to $H_{t+t'+1}$. The network architecture is the 1x1 1D convolutional version of the table~\ref{tab:enc_arch} without downsampling. The number of ResBlocks was set to two. 
We also replaced all of the group normalization layers with layer normalization~\cite{ba2016layer}.
\newpage

\begin{table}[t]
    \centering
    \subcaptionbox{Encoder architecture \label{tab:enc_arch}}[.70\linewidth]{
    \centering
    \begin{tabular}{cccc}
    \toprule
    & \#channels& \multirow{2}{*}{ Resampling} & Spatial \\
    &or \#dims&  & Resolution \\
    \midrule
         3x3 2DConv & 32*k & - & H$\times$W\\
         ResBlock& 64*k & Down & (H/2)$\times$(W/2) \\
         ResBlock& 128*k & Down &(H/4)$\times$(W/4)\\
         ResBlock& 256*k & Down &(H/8)$\times$(W/8)\\
         GroupNorm &256*k & -&(H/8)$\times$(W/8)\\
         ReLU &256*k & -&(H/8)$\times$(W/8)\\
         Flatten &256*k*(H/8)*(W/8) & -& -\\
         Linear & 16*256 &- &-\\
    \bottomrule
    \end{tabular}
    }
    \subcaptionbox{Decoder architecture \label{tab:dec_arch}}[.70\linewidth]{
    \centering
    \begin{tabular}{cccc}
    \toprule
    & \#channels& \multirow{2}{*}{ Resampling} & Spatial \\
    &or \#dims&  & Resolution \\
    \midrule
         Linear & 256*k*(H/8)*(W/8) & - & -\\
         Reshape & 256*k & - & (H/8)$\times$(W/8) \\
         ResBlock& 128*k & Up & (H/4)$\times$(W/2) \\
         ResBlock& 64*k & Up &(H/2)$\times$(W/4)\\
         ResBlock& 32*k & Up &H$\times$W\\
         GroupNorm &32*k & -&H$\times$W\\
         ReLU &32*k & -&H$\times$W\\
         3x3 2DConv & 3 (1 for SmallNORB)&- &H$\times$W\\
    \bottomrule
    \end{tabular}
    }
 \caption{
  The detail of the encoder and decoder architecture used in our experiments. The columns of `\#channels or \#dims' and `Spatial resolution' respectively represent the channels/dimensions and the spatial resolution at the end of each corresponding module.
  `Resampling' column represents whether the corresponding layer performs upsampling~(Up), downsampling~(Down) or none of them~(-). 
  Please see Figure~\ref{fig:resblock} for the detail of the ResBlock architecture. 
  The value $k$ in the table was set to 1 for 3DShapes, SmallNORB and Sequential ShapeNet. 
  The value $k$ was set to 2 for Sequential MNIST and accelerated Sequential MNIST, and 4 for Sequential MNIST-bg. 
  For SmallNORB, we added one more downsampling ResBlock after the third ResBlock in the encoder and one more upsampling ResBlock before the first ResBlock in the decoder. 
  For Sequential ShapeNet, we added two more downsampling ResBlock in the encoder and two more upsampling Resblock in the decoder. 
 \label{tab:arch}}
\end{table}

\begin{figure}[H]
    \centering
    \includegraphics[scale=0.4]{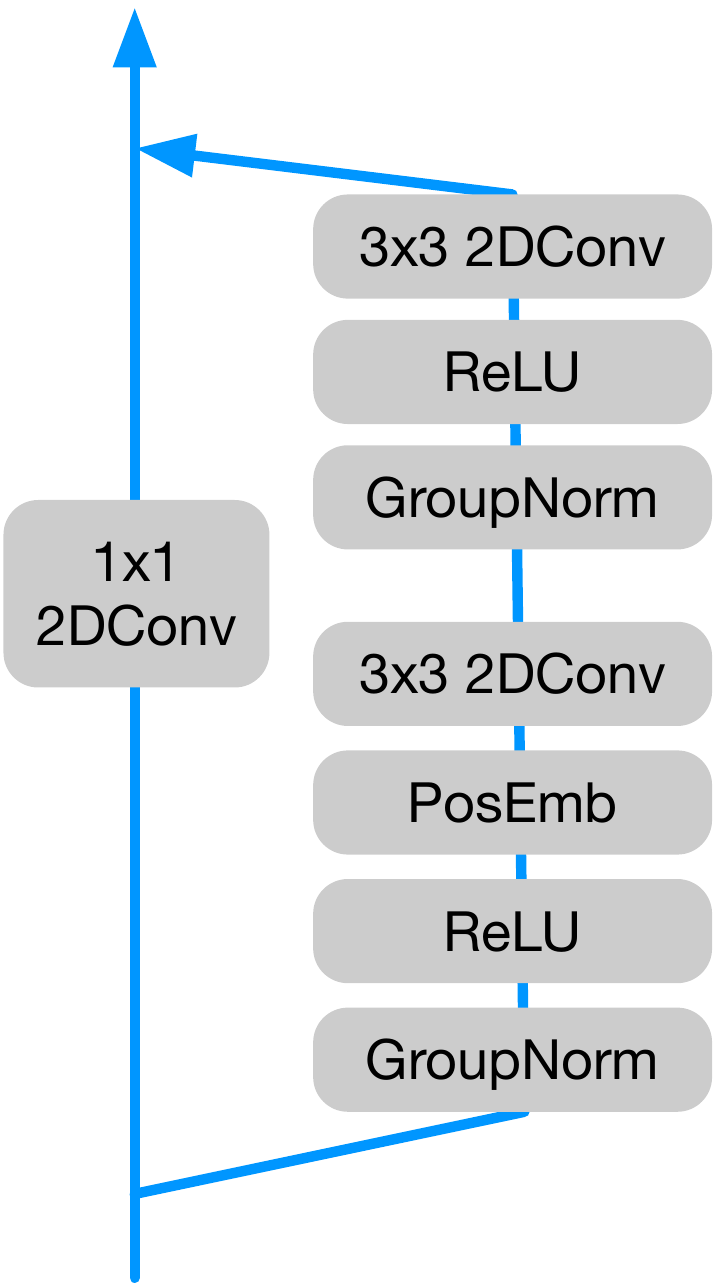}
    \caption{
        ResBlock architecture in the encoder and decoder. `PosEmb' stands for the positional embedding layer which concatenates the learned positional embedding to its input. 
    The embedding dimension was set to $32\times H\times W$. 
    The PosEmb layer was used only in decoder's resblock. 
    For the encoder, we performed downsampling (mean average
pooling) after the second convolution layer. 
For the decoder, we performed upsampling before the first convolution.
Also, we added a downsampling layer (mean avegrage pooling) after the 1x1 convolution.
For upsampling, we added a layer of nearest-neighbor upsampling before the 1x1 convolution.
    The number of groups for the group normalization layer was set to 32.
    \label{fig:resblock}}

\end{figure}

\section{Simultaneous Block-diagonalization}\label{sec:full_blockdiag}
 To find $U$ that simultaneously block-diagonalizes all $\bfM^*(\bs|\Phi)$, we optimized $U$ based on the objective function that measures the block-ness of $\irrepM^{*}(\bs):= \bfU \bfM^{*}(\bs | \Phi) \bfU^{-1}$.   
 Our objective function is based on the fact that, if we are given an adjacency matrix $A$ of a graph, then the number of connected components in the graph can be identified by looking at the rank of graph Laplacian:
 \begin{align}
       {\rm dim}({\rm Ker}(\Delta A)) = \#{\rm of~blocks~in}~A \label{eq:blockness}
 \end{align}
 where $\Delta$ is the graph Laplacian operator on $A$. 
 To relate our $\irrepM^{*}:= \bfU \bfM^{*} \bfU^{-1}$ to a graph, we see it as a bipartite graph and calculate the adjacency matrix by:
 \begin{align}
     A(\irrepM^*(\bs)) := abs(\irrepM^*(\bs)) abs(\irrepM^*(\bs) )^{\rm T} 
 \end{align}
 where $abs(\irrepM^*(\bs))$ represents the element-wise absolute value of $\irrepM^*$:~$abs(\bfM^*(\bs))_{ij} = 
 |\bfM^*_{ij}|$.
 To optimize Eq.\eqref{eq:blockness} with respect to the change of basis $U$ by continuous optimization, we used the lasso version of Eq.\eqref{eq:blockness}:
 \begin{align}
     \mathcal{L}_{\rm bd}\left(\irrepM^*(\bs)\right) := \left\|\Delta \left( A\left(\irrepM^*(\bs)\right)\right)\right\|_{\rm trace} = \sum_{d=1}^{a} \sigma_d \left( A\left(\irrepM^*(\bs)\right)\right) 
 \end{align}
 where $\sigma_i(\Delta A)$ is $i$-th singular value of $\Delta A$.  We used the symmetrically normalized version of the graph Laplacian: $\Delta A = I - D^{-1/2}AD^{-1/2}$ where $D$ is the degree matrix of $A$.
 Summing this over all $\bs^{(i)}$ in the dataset, we obtain:
 $\bar{\mathcal{L}}_{\rm bd}:=\frac{1}{N}\sum_{i=1}^{N} \mathcal{L}_{\rm bd}\left(\irrepM^{*}(\bs^{(i)})\right)
   \label{eq:bd_loss_all}$.
 We search for $U$ that simultaneously block-diagonalizes all 
 $\irrepM^{*}(\bs^{(i)})$ by minimizing $\bar{\mathcal{L}}_{\rm bd}$ w.r.t. $U$.

\section{Formal statements and the proofs of the theory section}\label{sec:proof}

We begin by summarizing the notations to be used in our formal statements. 
We use $\mathcal{X}$ to denote the space of all observations at a single time step, and $\Phi : \mathcal{X} \to \mathcal{H}$ to denote the encoder from $\mathcal{X}$ to the latent space $\mathcal{H}$.
If $\bs = [s_t \in  \mathcal{X};  t =1,...,T ]$ is one instance of video-sequence to be used in our training, we assume that, for each $\bs$, there is an operator $g : \mathcal{X} \to \mathcal{X}$ such that $g s_{t} = s_{t+1}$ for each $t$. As such,  each $\bs$ is characterized by a pair of initial state $s_1 \in \mathcal{X}$ and $g \in \mathcal{G}$, where $\mathcal{G}$ is the set of all operators considered. 
Thus, we would use $\bs(s_1, g)$ to denote a specific sequence.

Now, given a fixed encoder $\Phi : \mathcal{X}  \to \mathbb{R}^{a \times m}$, our training process computes the transition matrix $M$ independently for each instance of $\bs(s_1,g)=[s_t]_{t=1}^T=[g^{t-1} s_1]_{t=1}^T$. 
In particular, we compute 
\begin{align}
     M^* (g,s_1 | \Phi) = \argmin_M \frac{1}{T} \sum_{t=1}^{T-1} \| \Phi(s_{t+1}) - M\Phi(s_{t}) \|_F^2. 
\end{align}
In the theory developed here, we investigate the property of the optimal $\Phi$  when  $T=\infty$ so that $M^* (g, s_1 | \Phi)$ achieves 
\[
 \|\Phi(s_{t+1}) - M^*(g, s_1 | \Phi)\Phi(s_{t})  \|^2 = 0
\]
 or equivalently, 
$$
M^* (g, s_1| \Phi ) \Phi (g^{t-1}\circ s_1) = \Phi(g^{t} \circ s_1)
$$
for all $t\in \mathbb{N}$, $g\in\mathcal{G}$, and $s_1\in\mathcal{X}$. 
We would like to know whether $M^*(g, x| \Phi)$ has no dependency on $x$  so that $M_g = M^*(g|\Phi)$ defines an equivariance relation and hence a group representation .

We begin tackling this problem by first investigating $M^*(g, x| \Phi)$ within an orbit $\mathcal{G}\circ x:=\{h\circ x\in \mathcal{X}\mid h\in \mathcal{G}\}$.
That is, we check if we can say  $M^*(g,x|\Phi) = M^*(g,h\circ x|\Phi)$ for any $g, h\in\mathcal{G}$ and $x\in\mathcal{X}$.   We call this property {\em intra-orbital homogeneity}. 

We assume that $\mathcal{G}$ is a compact commutative Lie group in the following result.
\begin{proposition}[Intra-orbital homogeneity]\label{prop:intra}
Suppose that $\mathcal{G}$ is a compact commutative Lie group, $\Phi(x) \in \mathbb{R}^{a \times m}$ has rank $a$, and $M(g,x) \in \mathbb{R}^{a \times a}$ satisfies
\begin{equation}\label{eq:matrix_equiv}
    M(g, x)\Phi(g^k \circ  x) = \Phi(g^{k+1} \circ x )
\end{equation}
for all $k\in\mathbb{N}\cup\{0\}, x\in \mathcal{X}$ and $g \in \mathcal{G}$.   
If $M(g, x)$ is continuous with respect to $g$ and is uniformly continuous with respect to  $x$, then 
$$M(g, x ) = M(g, h \circ x )$$ 
for all $h \in \mathcal{G}$. \label{thm:intra_appendix}
\end{proposition}

Before going into the proof of this proposition, we show the following lemma about the basic properties of  $M(g,x)$ that satisfies \eqref{eq:matrix_equiv}.
\begin{lemma}\label{lma:M}
Assume that $\Phi(x)\in\mathbb{R}^{a\times m}$ has rank $a$, and that $a\times a$-matrix $M(g,x)$ satisfies \eqref{eq:matrix_equiv}
for all $k\in\mathbb{N}\cup\{0\}, x\in \mathcal{X}$ and $g \in \mathcal{G}$.  Then, 
\begin{itemize}[topsep=0pt, partopsep=0pt, itemsep=0pt, parsep=0pt, leftmargin=20pt]
\item[{\rm (i)}] 
$\;M(gh,x) = M(g,h\circ x)M(h,x)\;$ for any $g,h\in\mathcal{G}$ and $x\in\mathcal{X}$. 
\item[{\rm (ii)}] 
$
\; M(g^\ell, x) = M(g,x)^\ell \;
$ for any $\ell \in \mathbb{Z}$, $g\in \mathcal{G}$, and $x\in\mathcal{X}$.
\item[{\rm (iii)}] $\; M(g, g^\ell\circ x) = M(g,x)\;$ for any $\ell \in \mathbb{Z}$, $g\in \mathcal{G}$, and $x\in\mathcal{X}$. 
\end{itemize}
\end{lemma}
\begin{proof}
First note that, from \eqref{eq:matrix_equiv} with $k=0$, we have 
\begin{equation}\label{eq:matrix_def}
    M(g,x)\Phi(x) = \Phi(g\circ x)
\end{equation}
for any $g\in\mathcal{G}$ and $x\in\mathcal{X}$.

Using \eqref{eq:matrix_def} repeatedly, we have 
$$M(gh,x)\Phi(x) = \Phi(gh\circ x)=\Phi(g\circ(h\circ x)) = M(g,h\circ x)\Phi(h\circ x) = M(g,h\circ x)M(h,x)\Phi(x).$$
The rank assumption of $\Phi$ proves (i). 

Also, \eqref{eq:matrix_def} implies $M(e,x)=id$ for the unit $e\in\mathcal{G}$. We will first prove (ii) and (iii) with $\ell > 0$.  

For (ii), note that the repeated use of \eqref{eq:matrix_equiv} necessiates $$\Phi(g^\ell\circ x) = M(g,x)\Phi(g^{\ell-1}\circ x)=\cdots=M(g,x)^\ell \Phi(x).$$   On the other hand, 
$\Phi(g^\ell\circ x)=M(g^\ell,x)\Phi(x)$.  Equating these two expression of 
$\Phi(g^\ell\circ x)$ proves (ii) with $\ell > 0$.  

Meanwhile,  from \eqref{eq:matrix_equiv} we have $\Phi(g^{\ell+1}\circ x)=M(g,x)\Phi(g^\ell \circ x)$, while 
$$\Phi(g^{\ell+1}\circ x)=\Phi(g\circ (g^\ell\circ x))=M(g,g^\ell\circ x)\Phi(g^\ell\circ x).$$
This proves the assertion (iii) for $\ell > 0$.

Now, substituting  $x\leftarrow g^{-1}\circ x$  for (iii) with $\ell =1$, 
we obtain $M(g, x)=M(g,g^{-1}\circ x)$. 
On the other hand, substituting $h\leftarrow g^{-1}$ for (i), we get $M(g,g^{-1}\circ x)M(g^{-1},x)=M(e,x)=id$. 
Thus, 
\[
M(g^{-1},x) = M(g, x)^{-1}.
\]
Replacing $g$ with $g^{-1}$ in (ii) thus leads to $M(g^{-\ell},x)=M(g^{-1},x)^\ell = M(g,x)^{-\ell}$ for any $\ell\in\mathbb{N}$. 
This shows that (ii) holds for the negative integers as well. 
Also, substituting $g\leftarrow g^{-1}$ in (iii) yields  $M(g^{-1},g^{-\ell}\circ x)=M(g^{-1},x)$.  Taking the inverse of the both sides proves the assertion (iii) for the negative integers. 
\end{proof}

\begin{proof}[Proof of Proposition \ref{prop:intra}] 
Let $h, g \in \mathcal{G}$ be given. 
Since $\mathcal{G}$ is a connected commutative Lie group, the exponential map $\exp:\mathfrak{g}\to \mathcal{G}$ is surjective, where $\mathfrak{g}$ is the Lie algebra of $\mathcal{G}$ \cite{fulton_harris}. 
Therefore, there exists some $\eta\in\mathfrak{g}$ such that $\exp(\eta) = h$. 
%
%
Then, for any $n\in\mathbb{N}$, we can define $h^\frac{1}{n} := \exp(\eta /n )$ and $h^\frac{1}{n} \to e$ as $n\to\infty$.

By the uniform continuity assumption on $M(\cdot, x)$, for any $\epsilon> 0$, 
we can choose $n$ large enough so that 
\begin{equation}
    \label{eq:ineq_M_1}
\| M(g h^\frac{1}{n} ,  g^{-n} \circ x) -  M(g, g^{-n} \circ x) \|_F < \epsilon,
\end{equation}
and 
\begin{equation}    \label{eq:ineq_M_2}
 \| M(g h^\frac{1}{n} ,  h \circ x) -  M(g, h \circ x) \|_F < \epsilon.\end{equation} 

From Lemma \ref{lma:M} (iii), we have 
 \begin{align*}
     M(g h^\frac{1}{n} ,  g^{-n} \circ x) = M(g h^\frac{1}{n} ,  (g h^\frac{1}{n})^n g^{-n} \circ x),
 \end{align*}
and thus it follows from the commutativity assumption that 
\begin{equation}
    \label{eq:eq_M_1}
M(g h^\frac{1}{n} ,  g^{-n} \circ x) =  M(g h^{\frac{1}{n}},  h \circ x).
\end{equation}
At the same time, Lemma \ref{lma:M} (iii) implies  $ M(g, g^{-n}\circ x) = M(g, x)$ so that \eqref{eq:ineq_M_1} and \eqref{eq:eq_M_1} necessiates
\begin{align}\label{eq:ineq_M_3}
    \| M(g h^\frac{1}{n},  h \circ x) -  M(g, x)\|_F < \epsilon.
\end{align}
Finally, the combination of \eqref{eq:ineq_M_2} and \eqref{eq:ineq_M_3} guarantees
\begin{multline*}
    \| M(g,  h \circ x) -  M(g, x)\|_F \\
    \leq  \| M(g, h\circ x) -  M(g h^{\frac{1}{n}},  h \circ x) \|_F + \|M(g h^{\frac{1}{n}},  h \circ x) - M(g, x)\|_F  < 2 \epsilon.
\end{multline*}
Because $\epsilon>0$ is arbitrarily small, $\|M(g,  h \circ x) -  M(g, x)\|_F = 0$ necessarily holds, and the claim follows.
\end{proof}

\begin{proposition}
Suppose that, for a compact connected Lie group $\mathcal{G}$ and connected $\mathcal{X}$,  
 $M: \mathcal{G} \times \mathcal{X} \to \mathbb{R}^{a \times a}$ in \eqref{eq:matrix_equiv} satisfies the intra-orbital homogeneity, 
and that $\Phi(x) \in \mathbb{R}^{a \times m}$ has rank $a$ for all  $x$.   
If $M(g, x)$ is continuous with respect to $x$,  
then $M(g, x)$ is similar to $M(g, x')$ for all $x, x'$; that is, there is some $P\in GL(a,\mathbb{R})$ such that $P M(g,x)P^{-1} = M(g,x')$ for all $g\in\mathcal{G}$. 
\end{proposition}
\begin{proof}
From Lemma \ref{lma:M}, we have 
\[
M(gh,x) = M(g,h\circ x)M(h,x). 
\]
Combining this with intra-homogeneity 
$M(g, h\circ x) = M(g, x)$ provides
$$  M(gh, x) =M(g, x )M(h, x),$$ 
which means that, for each fixed $x$, 
$$ M_x: \mathcal{G} \to GL(a;\mathbb{R})
$$
defined by $M(g, x)= M_x (g)$ is a representation of the Lie group $\mathcal{G}$ \cite{fulton_harris}.
Now, if $\mathcal{G}$ is compact and connected as assumed in the statement, $M_x (g)$ is completely reducible, and $M_x$ is similar to a direct sum of irreducible representations. 
We then use the fact from character theory \cite{fulton_harris} that the multiplicity of any
irreducible representation $D$ in $M_x$ can be computed by 
\begin{align}
    \langle M_x | D \rangle = \int_\mathcal{G} tr(M(g, x))  \overline{tr(D(g))} \mu(dg), 
\end{align} 
where $\mu$ is a Haar measure of $\mathcal{G}$ with volume $1$, and $\overline{tr(D(g))}$ is the complex conjugate of $tr(D(g))$.
Because $\langle M_x | D \rangle$ is a multiplicity, it takes an integer value. At the same time, by its definition and the continuity of $M(\cdot,x)$, this value is continuous with respect to $x$. 
Thus, 
$\langle M_x | D \rangle$ must be constant on $\mathcal{X}$ by the connectedness of  $\mathcal{X}$.
That is, 
$$\langle M_x | D \rangle = \langle M_{x'} | D \rangle$$
for all $x, x' \in \mathcal{X}$. 
This means that, irrespective of $x$, $M(g, x)$ is similar to the direct sum of the same set of irreducible representations, and the claim follows. 
\end{proof}


\end{document}